\documentclass{article}

\PassOptionsToPackage{numbers}{natbib}
\usepackage[final]{neurips_2024}

\usepackage[utf8]{inputenc} % allow utf-8 input
\usepackage[T1]{fontenc}    % use 8-bit T1 fonts
\usepackage{hyperref}       % hyperlinks
\usepackage{url}            % simple URL typesetting
\usepackage{booktabs}       % professional-quality tables
\usepackage{amsfonts}       % blackboard math symbols
\usepackage{nicefrac}       % compact symbols for 1/2, etc.
\usepackage{microtype}      % microtypography
\usepackage[dvipsnames]{xcolor}         % colors

\usepackage{todonotes}
\usepackage{color}
\usepackage{enumitem}
\usepackage{dsfont}
\usepackage{algorithm}
\usepackage{algorithmic}
\usepackage{booktabs}
\usepackage{makecell}
\usepackage{subfigure}
\usepackage{graphicx}
\graphicspath{ {./imgs/} }
\usepackage{onimage}
\usepackage{makecell}
\usepackage{pifont}
\newcommand{\cmark}{\ding{51}}
\newcommand{\xmark}{\ding{55}}
\usepackage{amsmath,amsfonts,bm}
\usepackage{amsthm}

\newtheorem{theorem}{Theorem}
\newtheorem{lemma}{Lemma}

\newtheorem{assumption}{Assumption}

\newtheorem{corollary}{Corollary}

\def\eqref#1{equation~\ref{#1}}
\def\Eqref#1{Equation~\ref{#1}}
\def\floor#1{\lfloor #1 \rfloor}
\def\1{\bm{1}}

\def\vg{{\bm{g}}}

\def\vu{{\bm{u}}}

\def\vx{{\bm{x}}}

\def\vy{{\bm{y}}}

\def\mG{{\bm{G}}}

\DeclareMathAlphabet{\mathsfit}{\encodingdefault}{\sfdefault}{m}{sl}
\SetMathAlphabet{\mathsfit}{bold}{\encodingdefault}{\sfdefault}{bx}{n}

\def\gD{{\mathcal{D}}}

\def\gG{{\mathcal{G}}}

\def\gO{{\mathcal{O}}}

\def\gQ{{\mathcal{Q}}}

\def\gS{{\mathcal{S}}}

\def\sP{{\mathbb{P}}}

\def\bvg{\bar{\vg}}
\def\bvx{\bar{\vx}}

\newcommand{\indicator}[1]{\mathds{1}\left\{#1\right\}}

\newcommand{\Eqmark}[2]{\stackrel{(#1)}{#2}}

\setcitestyle{number}
\title{Federated Learning under Periodic Client Participation and Heterogeneous Data: A New Communication-Efficient Algorithm and Analysis}

\author{%
  Michael Crawshaw\\
  Department of Computer Science\\
  George Mason University\\
  Fairfax, VA 22030 \\
  \texttt{mcrawsha@gmu.edu} \\
  \And
  Mingrui Liu \\
  Department of Computer Science\\
  George Mason University\\
  Fairfax, VA 22030 \\
  \texttt{mingruil@gmu.edu} \\
}

\begin{document}

\maketitle

\begin{abstract}
In federated learning, it is common to assume that clients are always available to
participate in training, which may not be feasible with user devices in practice. Recent
works analyze federated learning under more realistic participation patterns, such as
cyclic client availability or arbitrary participation. However, all such works either
require strong assumptions (e.g., all clients participate almost surely within a bounded
window), do not achieve linear speedup and reduced communication rounds, or are not
applicable in the general non-convex setting. In this work, we focus on nonconvex
optimization and consider participation patterns in which the chance of participation
over a fixed window of rounds is equal among all clients, which includes cyclic client
availability as a special case. Under this setting, we propose a new algorithm, named
Amplified SCAFFOLD, and prove that it achieves linear speedup, reduced communication,
and resilience to data heterogeneity simultaneously. In particular, for cyclic
participation, our algorithm is proved to enjoy $\mathcal{O}(\epsilon^{-2})$
communication rounds to find an $\epsilon$-stationary point in the non-convex stochastic
setting. In contrast, the prior work under the same setting requires
$\mathcal{O}(\kappa^2 \epsilon^{-4})$ communication rounds, where $\kappa$ denotes the
data heterogeneity. Therefore, our algorithm significantly reduces communication rounds
due to better dependency in terms of $\epsilon$ and $\kappa$. Our analysis relies on a
fine-grained treatment of the nested dependence between client participation and errors
in the control variates, which results in tighter guarantees than previous work. We also
provide experimental results with (1) synthetic data and (2) real-world data with a
large number of clients $(N = 250)$, demonstrating the effectiveness of our algorithm
under periodic client participation.
\end{abstract}

\begin{table*}[t]
\caption{Communication and computation complexity of various methods to find an
$\epsilon$-stationary point for $L$-smooth, non-convex objectives. $N$: number of
clients, $\kappa$: data heterogeneity $\sup_\vx \|\nabla f_i(\vx) - \nabla f(\vx)\| \leq
\kappa$. $S$: number of participating clients per round, $\bar{K}$: number of groups for
cyclic participation. See Section \ref{sec:participation_patterns} for a description of
each participation pattern. We say that an algorithm exhibits reduced communication if
its dependence in terms of $\epsilon$ is strictly smaller than
$\mathcal{O}(\epsilon^{-4})$. Derivation of complexities for Amplified FedAvg can be
found in Appendix \ref{app:baseline_complexity}.}
\label{tab:complexity}
\begin{center}
\resizebox{\textwidth}{!}{\begin{tabular}{@{}lcccc@{}}
\toprule
Setting & \makecell{Communication \\ Complexity ($R$)} & \makecell{Iteration \\ Complexity ($RI$)} & \makecell{Reduced \\ Communication} & \makecell{Unaffected by \\ Heterogeneity} \\
\midrule
i.i.d. Participation ($S$) & & & & \\
\quad\quad FedAvg \cite{karimireddy2020scaffold} & $\frac{\Delta \kappa^2}{S \epsilon^4} \left( 1 - \frac{S}{N} \right) + \frac{\sqrt{L} \kappa}{\epsilon^3}$ & $\frac{\Delta L \sigma^2}{S \epsilon^4}$ & \textcolor{red}{\xmark} & \textcolor{red}{\xmark} \\
\quad\quad SCAFFOLD \cite{karimireddy2020scaffold} & $\frac{\Delta L}{\epsilon^2} \left( \frac{N}{S} \right)^{2/3}$ & $\frac{\Delta L \sigma^2}{S \epsilon^4}$ & \textcolor{Green}{\cmark} & \textcolor{Green}{\cmark} \\
\quad\quad Amplified FedAvg \cite{wang2022unified} & $\frac{\Delta L N \kappa^2}{S \epsilon^4} \left( 1 - \frac{S}{N} \right) + \frac{\Delta L + \kappa^2}{\epsilon^2}$ & $\frac{\Delta L \sigma^2}{S \epsilon^4}$ & \textcolor{red}{\xmark} & \textcolor{red}{\xmark} \\
\quad\quad Amplified SCAFFOLD (ours) & $\frac{\Delta L}{\epsilon^2} \frac{N}{S}$ & $\frac{\Delta L \sigma^2}{S \epsilon^4}$ & \textcolor{Green}{\cmark} & \textcolor{Green}{\cmark} \\
\midrule
Regularized Participation ($P, \rho$) & & & & \\
\quad\quad Amplified FedAvg \cite{wang2022unified} & $\frac{\Delta L P + \kappa^2 + \sigma^2}{\epsilon^2}$ & $\frac{\Delta L \rho^2 \sigma^2}{\epsilon^4}$ & \textcolor{Green}{\cmark} & \textcolor{red}{\xmark} \\
\quad\quad Amplified SCAFFOLD (ours) & $\frac{\Delta LP}{\epsilon^2}$ & $\frac{\Delta L \rho^2 \sigma^2}{\epsilon^4}$ & \textcolor{Green}{\cmark} & \textcolor{Green}{\cmark} \\
\midrule
Cyclic Participation ($\bar{K}$, $S$) & & & & \\
\quad\quad Amplified FedAvg \cite{wang2022unified} & $\frac{\Delta LN \kappa^2}{S \epsilon^4} \left( 1 - \frac{S\bar{K}}{N} \right) + \frac{\Delta L \bar{K} + \kappa^2}{\epsilon^2}$ & $\frac{\Delta L \sigma^2}{S \epsilon^4}$ & \textcolor{red}{\xmark} & \textcolor{red}{\xmark} \\
\quad\quad Amplified SCAFFOLD (ours) & $\frac{\Delta L \bar{K}}{\epsilon^2} \frac{N}{S}$ & $\frac{\Delta L \sigma^2}{S \epsilon^4}$ & \textcolor{Green}{\cmark} & \textcolor{Green}{\cmark} \\
\bottomrule
\end{tabular}}
\end{center}
\end{table*}

\section{Introduction}
Federated learning (FL) \cite{mcmahan2017communication, li2020federated1,
kairouz2021advances, zhang2021survey} is a distributed learning paradigm that emphasizes
client privacy \cite{mcmahan2017learning, wei2020federated, mothukuri2021survey},
limited communication \cite{konevcny2016federated, lim2020federated}, and data
heterogeneity across clients \cite{karimireddy2020scaffold, zhu2021federated}. FL has
attracted attention in recent years due to the ability to leverage data and compute from
user devices while respecting privacy \cite{yang2018applied, ding2017collecting}. For
large-scale FL, it is common to limit the number of simultaneously participating
devices, and many works do so by assuming that a random subset of clients can be sampled
independently at each round \cite{mcmahan2017communication, bonawitz2019towards,
yu_linear, li2020federated, wang2021field}. However, this pattern of client
participation is not always practical. If clients are user devices like mobile phones,
they may not have 24/7 availability due to low battery or bad internet connection
\cite{huba2022papaya, paulik2021federated}. In particular, if client availability is
correlated with geographical location (e.g. mobile phones charging at night), then
client availability follows a cyclic pattern \cite{zhu2021diurnal}. Therefore, it
remains an important open question to design federated optimization algorithms with
provable efficiency under non-i.i.d client participation.

Several works have investigated optimization in FL under non-i.i.d. client
participation
\cite{eichner2019semi,avdiukhin2021federated,cho2023convergence,wang2022unified,yan2020distributed}.
However, to the best of our knowledge, no existing algorithm in a non-i.i.d.
participation setting provably exhibits reduced communication cost, linear
speedup with respect to the number of clients, and resilience to client data
heterogeneity for general non-convex optimization.

In this work, we consider FL under an arbitrary participation framework
\cite{wang2022unified}, where client participation during each round is a random
variable with potentially unknown distribution. We focus on client participation
patterns that are periodic, in the sense that all clients are expected to participate
with equal frequency over a window of multiple training rounds.

For this setting, we propose Amplified SCAFFOLD, an optimization algorithm for FL under
periodic client participation. Amplified SCAFFOLD utilizes (a) amplified updates across
participation periods and (b) control variates computed across entire participation
periods, to eliminate the effect of data heterogeneity even under non-i.i.d.
participation. We show that Amplified SCAFFOLD exhibits significantly reduced
communication cost, linear speedup, and is unaffected by client data heterogeneity. To
the best of our knowledge, this is the first result demonstrating reduced communication
or resilience to data heterogeneity without assuming i.i.d. participation. The
complexity of Amplified SCAFFOLD is compared against baselines in Table
\ref{tab:complexity}. For cyclic participation, Amplified SCAFFOLD improves the previous
best communication cost from $\mathcal{O}(\kappa^2 \epsilon^{-4})$ to
$\mathcal{O}(\epsilon^{-2})$.

The main challenges of achieving these properties are (1) simultaneously handling
randomness from stochastic gradients and non-i.i.d. participation; and (2) controlling
the error of control variates under non-i.i.d. participation. Previous work in this
setting \cite{wang2022unified} performs an in-expectation analysis, by taking
expectation only over randomness from stochastic gradients; this avoids (1) but cannot
leverage properties of the participation pattern to reduce communication. We present a
tighter analysis that addresses (1) by taking expectation over both client participation
and the stochastic gradients throughout the analysis, and carefully treating the
trajectory variables which depend on both sources of randomness. We address (2) by
recursively bounding the control variate errors, which involves a non-uniform average of
non-uniform averages of error terms resulting from non-i.i.d. participation. We show
that this nested non-uniform average can be bounded using mild regularity conditions on
the participation pattern.

Our contributions are summarized below.
\begin{itemize}
    \item We introduce Amplified SCAFFOLD, an optimization algorithm for federated
        learning under non-i.i.d. client participation. Our convergence analysis
        demonstrates its computational and communication efficiency: Amplified SCAFFOLD
        exhibits reduced communication, linear speedup, and is unaffected by data
        heterogeneity. These guarantees are achieved with a tighter analysis than used
        in previous work \cite{wang2022unified}, with a fine-grained treatment of the
        two sources of randomness: client participation and stochastic gradients. In the
        case of cyclic participation, we reduce the previous best communication cost of
        $\mathcal{O}(\kappa^2 \epsilon^{-4})$ to $\mathcal{O}(\epsilon^{-2})$.
    \item Experimental results show that Amplified SCAFFOLD converges faster than
        baselines on both synthetic and real-world problems under realistic non-i.i.d.
        client participation patterns. We also include an ablation study which demonstrates
        the robustness of our algorithm to changes in data heterogeneity, the number of
        participating clients per round, and the number of client groups in cyclic
        participation.
\end{itemize}

The paper is outlined as follows. We discuss related work in Section
\ref{sec:related_work}, and Section \ref{sec:setup} provides a formal specification of
the optimization problem. Amplified SCAFFOLD is introduced and theoretically analyzed in
Section \ref{sec:alg}, and we provide experiments in Section \ref{sec:experiments}. We
conclude with Section \ref{sec:conclusion}.

\section{Related Work} \label{sec:related_work}

\textbf{Federated Optimization.} FedAvg~\cite{mcmahan2017communication} characterizes
partial client participation and local updates in each round. FedAvg was analyzed in the
full participation
setting~\cite{stich2018local,wang2018cooperative,yu_linear,yu2019parallel,woodworth2020minibatch,woodworth2020local,khaled2020tighter,glasgow2022sharp}.
Other federated optimization algorithms aim to improve communication
efficiency~\cite{reddi2020adaptive,yuan2020federated} and tackle data
heterogeneity~\cite{li2020federated,karimireddy2020scaffold}. The analysis of FL
optimization algorithms typically either assumes full client participation or partial
client participation where clients are sampled uniformly
randomly~\cite{yang2020achieving,wang2021field,karimireddy2020scaffold,li2019convergence,woodworth2020minibatch,acar2021federated}.
\cite{patel2022towards} provides lower bounds for distributed stochastic, smooth
optimization with intermittent communication and non-convex objectives, both in the full
and partial participation settings. They also include algorithms employing variance
reduction which match (or closely match) lower bounds in the full and partial
participation settings. However, none of the works above are applicable for general
participation patterns such as periodic participation.

\textbf{Client Participation.} Cyclic data sampling was considered for
stochastic convex optimization in~\cite{eichner2019semi}, where they propose
``pluralistic" solutions instead of learning a single model for all clients.
There is a recent line of work considering various participation patterns,
including client
selection~\cite{fraboni2021clustered,chen2020optimal,rizk2022federated} biased
participation~\cite{ruan2021towards,cho2020client,cho2022towards}, independent
participation across
rounds~\cite{karimireddy2020scaffold,li2020federated,li2019convergence},
unbiased participation~\cite{tyurin2022sharper,grudzien2023improving}, bounded
rounds of unavailability~\cite{yan2020distributed,gu2021fast,yang2022anarchic},
asynchronous participation~\cite{lian2018asynchronous,avdiukhin2021federated},
cyclic participation~\cite{cho2023convergence}, and arbitrary
participation~\cite{wang2022unified,wang2023lightweight}. However, none of these
works enjoy linear speedup, reduced communication rounds, and resilience to data
heterogeneity under the general setting of non-convex objectives and periodic
client participation.

See Appendix \ref{app:comparison} for a detailed discussion comparing our results with a small number of closely related baselines.

\section{Problem Setup} \label{sec:setup}
We consider a federated learning problem with $N$ clients, with the overall objective
\begin{equation*}
    \min_{\vx \in \mathbb{R}^d} \left\{ f(x) := \frac{1}{N} \sum_{i=1}^N f_i(x) \right\},
\end{equation*}
where each $f_i: \mathbb{R}^d \rightarrow \mathbb{R}$ is the local objective of one
client. We consider the stochastic optimization problem, so that $f_i(\vx) =
\mathbb{E}_{\xi \sim \gD_i}[F(\vx; \xi)]$, and the optimization algorithm can access
$F_i(\vx; \xi)$ and $\nabla F_i(\vx; \xi)$ for individual values of $\xi$. We make the
following assumptions about the objectives:

\begin{assumption} \label{ass:obj}
\textbf{(a)} $f(\vx_0) - \min_{\vx \in \mathbb{R}^d} f(\vx) \leq \Delta$. \textbf{(b)}
Each $f_i$ is $L$-smooth, i.e., $\|\nabla f_i(\vx) - \nabla f_i(\vy)\| \leq L\|\vx -
\vy\|$ for all $\vx, \vy \in \mathbb{R}^d$. \textbf{(c)} The stochastic gradient has
variance $\sigma^2$, i.e., $\mathbb{E}_{\xi \sim \gD_i} [\|\nabla F_i(\vx; \xi) - \nabla
f_i(\vx)\|^2] \leq \sigma^2$ for all $\vx \in \mathbb{R}^d$.
\end{assumption}

Since each $f_i$ may be non-convex, we consider the problem of finding an
$\epsilon$-stationary point of $f$, that is, a point $\vx \in \mathbb{R}^d$ such that
$\|\nabla f(\vx)\| \leq \epsilon$.

\subsection{Participation Framework} \label{sec:participation_framework}
We consider a federated learning framework consisting of $R$ rounds. For any round $r
\in \{0, \ldots, R-1\}$ and client $i \in [N]$, the availability of client $i$ at round
$r$ is a random variable $q_r^i$, following the arbitrary participation framework of
\cite{wang2022unified}. If $q_r^i = 0$, then client $i$ may not participate during round
$i$. For example, under the conventional i.i.d. sampling of clients, at each round $r$ a
subset of clients $\gS_r \subset [N]$ is sampled uniformly without replacement, and the
weights are set as $q_r^i = \frac{\indicator{i \in \gS_r}}{S}$.

For some $P \in \mathbb{N}$, let $\gQ_{r_0}$ be the filtration generated by $\{q_r^i:
r_0 \leq r < r_0 + P, i \in [N]\}$, let $\gQ$ be the filtration generated by $\gQ_0,
\ldots, \gQ_{R-P}$, and let $\gG$ be the filtration generated by $\{\xi_{r,k}^i: 0 \leq
r < R, 0 \leq k < I, i \in [N]\}$, where $\xi_{r,k}^i$ is the random sampling of the
stochastic gradient of round $r$, step $k$, client $i$. We make the following
assumptions about the participation distribution.

\begin{assumption} \label{ass:sampling}
For all $r \in \{0, \ldots, R-1\}$: \textbf{(a)} $\sum_{i=1}^N q_r^i = 1$ and
$\sum_{i=1}^N(q_r^i)^2 \leq \rho^2$. \textbf{(b)} The distribution of $\{q_r^i\}$ is
unbiased across clients over every window of $P$ rounds, i.e.,
$\mathbb{E}_{\gQ_{r_0}}[\frac{1}{P} \sum_{r=mP}^{(m+1)P - 1} q_r^i] = 1/N$ for every $m
< R/P$ and $i \in [N]$. \textbf{(c)} Each client has a non-zero probability of being
sampled over every window of $P$ rounds, i.e., $\sP_{\gQ_{r_0}}(\frac{1}{P}
\sum_{r=mP}^{(m+1)P - 1} q_r^i > 0) > p_{\text{sample}}$ for every $m < R/P$ and $i \in
[N]$. \textbf{(d)} $\gQ$ and $\gG$ are independent.
\end{assumption}

For each round, Assumption \ref{ass:sampling}(a) enforces that the participation weights
$q_r^i$ are normalized to sum to 1, and characterizes the spread of participation
weights across clients with the constant $\rho^2$. Assumption \ref{ass:sampling}(b)
requires that the set of rounds can be partitioned into windows of length $P$ within
which clients are expected to participate with equal frequency. Lastly, Assumption
\ref{ass:sampling}(c) enforces that within each window, for each client the probability
of being sampled is nonzero. Conventional i.i.d client sampling satisfies Assumption
\ref{ass:sampling} with $\rho = S^{-1/2}, P=1$, and $p_{\text{sample}} = S/N$.

An important difference from conventional i.i.d. participation is that here, client
participation is not necessarily independent across rounds. Accordingly, we emphasize
that the expectation and probability in Assumptions \ref{ass:sampling}(b)-(c) are taken
only over $\gQ_{r_0}$. Therefore, the mean participation weight in Assumption
\ref{ass:sampling}(b) may itself be a random variable if client participation at some
rounds is dependent on the outcome of participation in previous rounds. Similarly, the
sampling probability in Assumption \ref{ass:sampling}(c) may be a random variable. For
the participation patterns considered in the next section, Assumption \ref{ass:sampling}
is satisfied even when client sampling is not independent across rounds.

\subsection{Specific Participation Patterns} \label{sec:participation_patterns}

\textbf{Regularized Participation} We say that client participation is
\textit{regularized} \cite{wang2022unified} if $\bar{q}_{r_0}^i = \frac{1}{N}$ almost
surely for all $r_0$ and $i$, where $\bar{q}_{r_0}^i$ is defined on Line 18 of Algorithm
1 as the participation of client $i$ averaged over rounds $r_0, \ldots, r_0 + P - 1$. In
this case, Assumption \ref{ass:sampling} is satisfied with $p_{\text{sample}} = 1$,
while $P$ and $\rho^2$ are parameters of the participation pattern. Regularized
participation is a relatively strong constraint, since every client must participate
within each window, which may not be practical. However, it is flexible in that there is
no constraint on how clients participate within each window. Regularized participation
was also considered for strongly convex objectives \cite{malinovsky2023federated}.

\textbf{Cyclic Participation} Following the CyCP framework \cite{cho2023convergence},
$N$ clients are partitioned into $\bar{K}$ equally sized subsets, and at round $r$ only
clients in group $(r \text{ mod } \bar{K})$ may participate. $S$ clients are sampled
without replacement from group $(r \text{ mod } \bar{K})$, for whom the participation
weight is $q_r^i = 1/S$. All other clients are assigned $q_r^i = 0$. Cyclic
participation satisfies Assumption \ref{ass:sampling} with $P = \bar{K}$, $\rho =
S^{-1/2}$, and $p_{\text{sample}} = S\bar{K}/N$. Notice that i.i.d. client sampling is
the special case where $\bar{K} = 1$.

Cyclic participation can model a situation where each client group is available at a
different time of day. For example, if client devices are mobile phones, then clients
are available for participation at night, when phones are charging, likely to have
internet connection, and otherwise idle. If devices are spread across the globe, then
client groups are naturally formed by time zones. Cyclic participation is less stringent
than regularized participation since not all clients are required to participate within
each window. FedAvg was analyzed under cyclic participation for PL objectives
\cite{cho2023convergence}, although this analysis is not applicable in our setting,
which uses general non-convex objectives.

\section{Algorithm and Analysis} \label{sec:alg}
In this section, we present Amplified SCAFFOLD, our algorithm to solve the FL problem
described in Section \ref{sec:setup}. Pseudocode for Amplified SCAFFOLD is shown in
Algorithm 1. The main components of the Amplified SCAFFOLD algorithm are (1) amplified
updates and (2) long-range control variates.

\begin{figure}
\begin{center}
\begin{tikzonimage}[width=\textwidth]{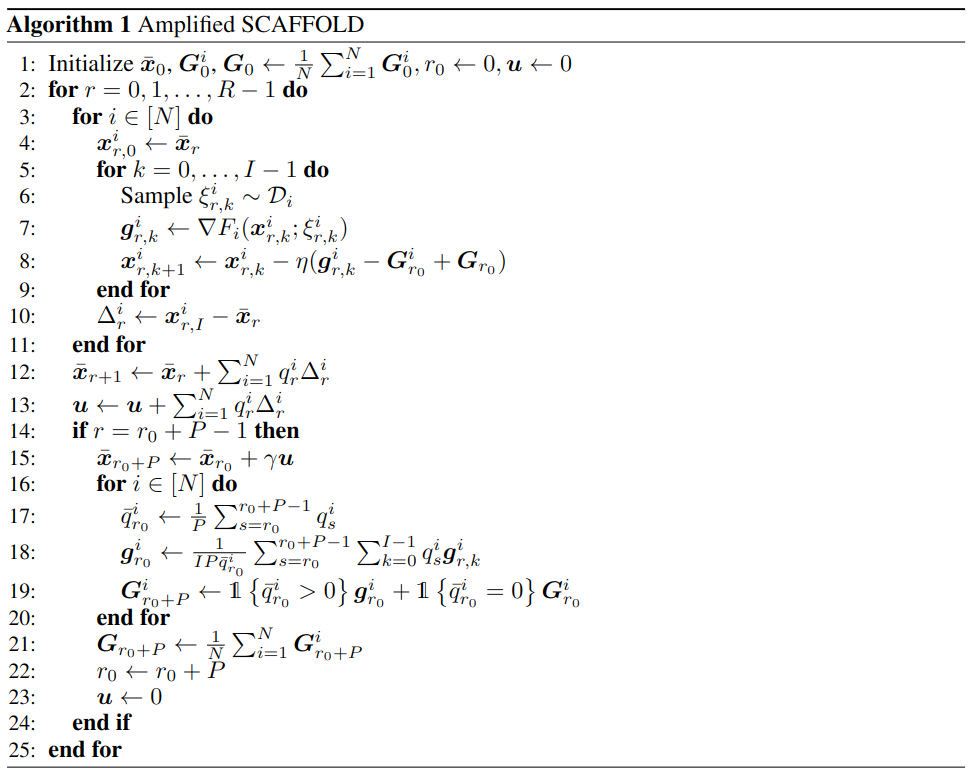}
\draw [blue, fill=blue, opacity=0.2, line width=1pt] (0.005,0.835) rectangle (0.625,0.57);
\node[draw] at (0.575,0.795) {client};
\draw [orange, fill=orange, opacity=0.2, line width=1pt] (0.005,0.54) rectangle (0.625,0.085);
\node[draw] at (0.575,0.50) {server};
\end{tikzonimage}
\end{center}
\end{figure}

\subsection{Algorithm Overview}
To deal with the non-stationarity of client availability, Amplified SCAFFOLD performs
amplified updates based on information accumulated over a window of $P$ rounds. In
Algorithm 1, the variable $\vu$ holds a weighted average of local updates to client
models, weighted by client participation. Every $P$ rounds, the global model is updated
in the direction $\vu$ scaled by the amplification factor $\gamma$. Informally, the
direction $\vu$ includes information from all clients with equal representation,
according to Assumption \ref{ass:sampling}(b). Similar amplified updates are used in
Amplified FedAvg \cite{wang2022unified}.

Control variates for heterogeneous federated learning were first introduced by SCAFFOLD
\cite{karimireddy2020scaffold}. However, SCAFFOLD-style control variates are updated
every time a client participates, which may not be appropriate under periodic
availability. For example, under non-i.i.d. participation control variates for different
clients would be updated with different frequencies, so that some clients may have
consistently less accurate control variates than others. Informally, this may lead to a
bias in which some clients' objective is underweighted relative to others. To deal with
this issue, Amplified SCAFFOLD updates control variates based on information accumulated
over a window of $P$ rounds, which enforces equal representation of all clients in
expectation, according to Assumption \ref{ass:sampling}(b).

\textbf{Comparison with \cite{karimireddy2020scaffold, wang2022unified}} Although the
two algorithmic components of Amplified SCAFFOLD individually appear in previous work
\cite{karimireddy2020scaffold, wang2022unified}, we emphasize that our complexity
results cannot be achieved by simply combining the analyses of these two works. The
analysis of \cite{wang2022unified} requires $\epsilon^{-4}$ communication cost due to
their treatment of the randomness in client participation. Here, we present a tighter
analysis with $\epsilon^{-2}$ cost from a more fine-grained treatment of the two sources
of randomness (stochastic gradients and client sampling). See Section
\ref{sec:proof_sketch} for more details on our approach.

\subsection{Main Results} \label{sec:main_results}
Let $\hat{\vx} = \bvx_{mP}$, where $m$ is sampled uniformly from $\{0, \ldots, R/P-1\}$,
and let $r_0 \in \{0, P, \ldots, R/P\}$.  Denote ${w_{r_0}^i = \frac{1}{N} \sum_{j=1}^N
\frac{\indicator{\bar{q}_{r_0}^j > 0}}{P \bar{q}_{r_0}^j} \sum_{s=r_0}^{r_0+P-1} q_s^i
q_s^j}$ and ${v_{r_0}^i = \bar{q}_{r_0}^i - \frac{1}{N}}$.  Informally, $w_{r_0}^i$
represents the "non-uniformity" of the client sampling distribution. We also consider a
variable $\Lambda_{r_0}^i$ that depends only on the client sampling distribution and
characterizes the sample size from which $\mG_{r_0}^i$ is computed. See Appendix
\ref{app:definitions} for further discussion of these quantities. We consider
convergence under the following conditions, which are satisfied by several participation
patterns of interest.
\begin{align}
    \mathbb{E}[w_{r_0}^i] \leq \frac{P^2}{N} \quad \text{for all } r_0 \text{ and } i, \label{eq:w_cond} \\
    \mathbb{E} \left[ \sum_{i=1}^N \left( v_{r_0}^i \right)^2 \Lambda_{r_0}^i \right] \leq \rho^2 \quad \text{for all } r_0 \text{ and } i, \label{eq:v_cond}
\end{align}

\begin{theorem} \label{thm:main_convergence}
Suppose that Assumptions \ref{ass:obj} and \ref{ass:sampling} hold, and that
\Eqref{eq:w_cond} and \Eqref{eq:v_cond} hold. If ${\gamma \eta \leq
\frac{p_{\text{sample}}}{60 LIP}}$ and ${\eta \leq \frac{\sqrt{p_{\text{sample}}}}{60
LIP}}$, then Algorithm 1 satisfies
\begin{align*}
    \mathbb{E}[\|\nabla f(\hat{\vx})\|^2] \leq \mathcal{O} \left( \frac{\Delta}{\gamma \eta I R} + \left( \gamma \eta L \rho^2 + \eta^2 L^2 IP \right) \sigma^2 \right).
\end{align*}
\end{theorem}

\begin{corollary} \label{cor:main_complexity}
For any $\epsilon > 0$ and $I \geq 1$, there exist choices of $\gamma$ and $\eta$ such
that $\mathbb{E}[\|\nabla f(\hat{\vx})\|^2] \leq \mathcal{O}(\epsilon^2)$ as long as ${R
\geq \mathcal{O} \left( \frac{\Delta L \rho^2 \sigma^2}{I \epsilon^4} + \frac{\Delta
LP}{p_{\text{sample}} \epsilon^2} \right)}$.
\end{corollary}

The complexity of Amplified SCAFFOLD has several important properties:

\textbf{Reduced Communication} By choosing $I = \Theta(\Delta \rho^2 \sigma^2
p_{\text{sample}} P^{-1} \epsilon^{-2})$, Amplified SCAFFOLD has communication
complexity $R = \mathcal{O}(LP p_{\text{sample}}^{-1} \epsilon^{-2})$, which improves
upon the $\epsilon^{-4}$ complexity of parallel SGD. We are not aware of any existing
work that achieves this communication reduction for non-convex federated optimization
with periodic participation.

\textbf{Unaffected by Heterogeneity} The iterations $RI$ and the number of
communications $R$ are unaffected by heterogeneity, which is not achieved for periodic
participation by any existing work \cite{wang2022unified, cho2023convergence}.

\textbf{Linear Speedup} The number of iterations $RI = \mathcal{O}(\Delta L \rho^2
\sigma^2 \epsilon^{-4})$ will exhibit linear speedup in the number of clients through
the term $\rho^2$, depending on the client participation pattern.

\subsection{Application to Participation Patterns} \label{sec:pattern_results}
The results from Section \ref{sec:main_results} apply under any participation pattern
that satisfies Assumption \ref{ass:sampling}, \Eqref{eq:w_cond}, and \Eqref{eq:v_cond},
and below we discuss the participation patterns discussed in Section
\ref{sec:participation_patterns}. The complexity of Amplified SCAFFOLD for each
participation pattern are shown in Table \ref{tab:complexity}, and these results can be
obtained by plugging $\rho^2$, $P$, and $p_{\text{sample}}$ into Corollary
\ref{cor:main_complexity}, together with a choice of $I$ as described in Section
\ref{sec:main_results}. The derivations of each result below are given in Appendix
\ref{app:pattern_proofs}.

\textbf{Regularized Participation} Recall that regularized participation satisfies
Assumption \ref{ass:sampling} with $p_{\text{sample}} = 1$, and $P$, $\rho^2$ are
parameters of the participation pattern. Also, under regularized participation,
$w_{r_0}^i = \bar{q}_{r_0}^i = 1/N$ almost surely, so that $\mathbb{E}[w_{r_0}^i] = 1/N
\leq P^2/N$ and $v_{r_0}^i = 0$. Therefore \Eqref{eq:w_cond} and \Eqref{eq:v_cond} are
satisfied. Plugging $p_{\text{sample}} = 1$ into Corollary \ref{cor:main_complexity}
yields
\begin{equation*}
    R = \mathcal{O} \left( LP \epsilon^{-2} \right), \quad RI = \mathcal{O} \left( \Delta L \rho^2 \sigma^2 \epsilon^{-4} \right).
\end{equation*}
In this setting, our algorithm exhibits reduced communication and resilience to
heterogeneity. To our knowledge, the only existing algorithm with theoretical guarantees
for non-convex problems under regularized participation is Amplified FedAvg
\cite{wang2022unified}. However, as seen in Table \ref{tab:complexity}, the
communication complexity of Amplified FedAvg has order $\epsilon^{-4}$ in terms of
$\epsilon$ and suffers from a $\kappa^2$ dependence.

\textbf{Cyclic Participation} Recall that cyclic participation satisfies Assumption
\ref{ass:sampling} with $P = \bar{K}$, $\rho = S^{-1/2}$, and $p_{\text{sample}} = S/N$.
Also, $\mathbb{E}[w_{r_0}^i] = S/N^2 \leq P^2/N$ and $\mathbb{E}[(v_{r_0}^i)^2
\Lambda_{r_0}^i] = \rho^2$, so that \Eqref{eq:w_cond} and \Eqref{eq:v_cond} are
satisfied. Based on the above parameter values, the resulting complexities are
\begin{equation*}
    R = \mathcal{O} \left( \frac{L \bar{K}}{\epsilon^2} \left( \frac{N}{S} \right) \right), \quad RI = \mathcal{O} \left( \frac{\Delta L \sigma^2}{S \epsilon^4} \right).
\end{equation*}
Again, Amplified SCAFFOLD achieves reduced communication, linear speedup, and resilience
to heterogeneity. Amplified FedAvg \cite{wang2022unified} is the only existing algorithm
with theoretical guarantees in this setting, but it fails to achieve resilience to
heterogeneity or reduce communication cost outside of the trivial case of full
participation ($\bar{K} = 1, S=N$). Also, even for the setting of PL-functions, the
convergence rate of FedAvg under cyclic participation from \cite{cho2023convergence}
does not demonstrate an improvement with respect to the number of local steps. See
Appendix \ref{app:comparison} for further discussion of their results.

Recall that i.i.d. participation is a special case of cyclic participation with $\bar{K}
= 1$. In this case, Amplified FedAvg fails to recover the reduced communication usually
achieved under i.i.d. participation, such as by SCAFFOLD \cite{karimireddy2020scaffold}.
In fact, Amplified FedAvg fails to recover the communication cost of FedAvg under i.i.d.
participation, requiring an additional factor of $LN$. The larger communication cost of
Amplified FedAvg is a result of its convergence analysis, which does not leverage the
property of unbiased participation (Assumption \ref{ass:sampling}(b)) during the
analysis, and requires $P = \mathcal{O}(\epsilon^{-2})$ in order to converge (see
Appendix \ref{app:baseline_complexity} for more details). In contrast, Amplified
SCAFFOLD succeeds in recovering the results of SCAFFOLD under i.i.d. participation, with
only a slightly worse dependence of $R$ on $N/S$. This difference in the order of
$\frac{N}{S}$ is due to a potential small issue in the analysis of SCAFFOLD, which we
intentionally avoided by accepting a slightly worse dependence on $\frac{N}{S}$. We
provide a detailed discussion of the $N/S$ dependence in Appendix \ref{app:comparison}.

\subsection{Proof Sketch} \label{sec:proof_sketch}
The main challenges for demonstrating convergence are (1) simultaneously handling
randomness from stochastic gradients and non-i.i.d. client sampling, and (2) controlling
error of control variates under non-i.i.d. client sampling. Previous work
\cite{wang2022unified} subverts (1) by conditioning on $\gQ$ throughout the entire
analysis. However, this eliminates the possibility of utilizing the condition
$\mathbb{E}[\bar{q}_{r_0}^i] = 1/N$, and ultimately incurs a dependence on the data
heterogeneity (see the term $\tilde{\delta}^2(P)$ in Theorem 3.1 of
\cite{wang2022unified}). Instead, we take expectation over both sources of randomness
throughout the analysis, which requires a careful treatment of each iterate's
dependence, and enables communication reduction. For (2), previous analysis of federated
algorithms with control variates \cite{karimireddy2020scaffold} recursively bounds the
error of control variates between consecutive rounds. However, this recursion crucially
depends on i.i.d. client participation. We extend this analysis to our setting,
establishing a recursion over the control variate error between consecutive windows of
$P$ rounds. Establishing this recursion under non-i.i.d.  participation involves a
non-uniform average of non-uniform averages of error terms, which we handle by invoking
the regularity conditions stated in \Eqref{eq:w_cond} and \Eqref{eq:v_cond}.

Using smoothness of $f$, the objective function decrease $f(\bvx_{r_0+P}) -
f(\bvx_{r_0})$ is upper bounded by $\langle \nabla f(\bvx_{r_0}), \bvx_{r_0+P} -
\bvx_{r_0} \rangle + \frac{L}{2} \|\bvx_{r_0+P} - \bvx_{r_0}\|^2$. Letting $\bvx_{r,k} =
\sum_{i=1}^N q_r^i \vx_{r,k}^i$ be a weighted average of local models, the sum of the
previous inner product and quadratic terms can be bounded by $-\gamma \eta IP \|\nabla
f(\bvx_{r_0})\|^2$, plus standard noise terms, the additional ``drift" terms
\begin{align*}
    \tilde{D}_{r,k} = \sum_{i=1}^N q_r^i \left\| \vx_{r,k}^i - \bvx_{r,k} \right\|^2 \quad \tilde{M}_{r,k} = \left\| \bvx_{r,k} - \bvx_{r_0} \right\|^2,
\end{align*}
and control variate errors $C_{r_0}^i = \| \nabla f_i(\bvx_{r_0}) - \mG_{r_0}^i \|^2$.
$\tilde{D}_{r,k}$ captures the distance between local client models, while
$\tilde{M}_{r,k}$ captures the distance from local models to the previous global model
$\bvx_{r_0}$.

Taking conditional expectation $D_{r,k} = \mathbb{E}[\tilde{D}_{r,k} | \gQ]$ and
$M_{r,k} = \mathbb{E}[\tilde{M}_{r,k} | \gQ]$, Lemma \ref{lem:dm_ub} bounds the drift
terms by establishing and unrolling a mutually recurrent relation between $D_{r,k}$ and
$M_{r,k}$. The resulting bound involves a non-uniform average over the control variate
errors: $\sum_{i=1}^N q_r^i \mathbb{E}[C_{r_0}^i | \gQ]$.

Denoting the average control variate error $C_{r_0} = \frac{1}{N} \sum_{i=1}^N
C_{r_0}^i$, we want to bound $\mathbb{E}[C_{r_0+P}]$ in terms of $\mathbb{E}[C_{r_0}]$.
$C_{r_0+P}$ can be decomposed into drift terms, but the result is a non-uniform average:
\[
    \sum_{s=r_0}^{r_0+P-1} \sum_{k=0}^{I-1} q_s^i (D_{s,k} + M_{s,k}).
\]
Since the bound for each $M_{s,k} + D_{s,k}$ from Lemma 1 involves a non-uniform average
over $C_{r_0}^i$, the resulting bound of $C_{r_0+P}$ involves a non-uniform average of
non-uniform averages of $C_{r_0}^i$, instead of the uniform average $C_{r_0}$. The
regularity conditions in \Eqref{eq:w_cond} and \Eqref{eq:v_cond} allow us to bound this
nested non-uniform average by a uniform average, which finishes the recursion.

Putting everything together, we obtain the descent inequality
\begin{align*}
    \mathbb{E}[\tilde{f}_{r_0+P}] &\leq \mathbb{E}[\tilde{f}_{r_0}] - \gamma \eta IP \mathbb{E} \left[ \|\nabla f(\bvx_{r_0})\|^2 \right] + \gamma \eta IP (\gamma \eta L \rho^2 + \eta^2 L^2 IP) \sigma^2,
\end{align*}
where $\tilde{f}_{r_0} := f(\bvx_{r_0+P}) + \Phi(r_0+P)$ and $\Phi$ is a potential
function that depends on the control variate errors. Theorem \ref{thm:main_convergence}
is then obtained by averaging over $r_0$ and isolating the gradient.

\section{Experiments} \label{sec:experiments}
We experimentally validate our algorithm for non-i.i.d client participation under three
settings: minimizing a synthetic function, logistic regression for Fashion-MNIST
\footnote{Fashion-MNIST is licensed under the MIT License.} \cite{xiao2017/online}, and
training a CNN for CIFAR-10 \cite{krizhevsky2009learning}. We also include an ablation
study on Fashion-MNIST, to investigate how each algorithm is affected by changes in data
heterogeneity, the number of participating clients, and the number of client groups in
cyclic participation.

\subsection{Setup} \label{sec:exp_setup}
All of our experiments utilize a non-i.i.d. client participation pattern similar to
cyclic participation (discussed in Section \ref{sec:participation_patterns}). We
partition the total set of $N$ clients into $\bar{K}$ equally sized subsets, and at each
training round only a single client group is available for participation. In our
experiments, the available group does not change every round; instead, each group is
available for $g$ rounds at a time. Under this pattern, Assumption \ref{ass:sampling} is
satisfied with $P = g\bar{K}$. We refer to $g$ as the availability time.

We evaluate five algorithms: FedAvg \cite{mcmahan2017communication}, FedProx
\cite{li2020federated}, SCAFFOLD \cite{karimireddy2020scaffold}, Amplified FedAvg
\cite{wang2022unified}, and Amplified SCAFFOLD (ours). We tune each algorithm's
parameters by grid search, including learning rate $\eta$, amplification rate $\gamma$,
and FedProx's $\mu$. The search ranges and tuned values can be found in Appendix
\ref{app:experiment_details}. All experiments were run on a single node with eight
NVIDIA A6000 GPUs. Code is available at the following repository:
\url{https://github.com/MingruiLiu-ML-Lab/FL-under-Periodic-Participation}

\textbf{Synthetic}
We evaluate each algorithm's convergence on a difficult objective based on a lower bound
for FedAvg \cite{woodworth2020minibatch}. The objective maps $\mathbb{R}^4$ to
$\mathbb{R}$, is convex, and is parameterized by a smoothness $L$, stochastic gradient
variance $\sigma^2$, and heterogeneity $\kappa$, so that it satisfies Assumption
\ref{ass:obj} by construction. The complete definition of the objective can be found in
Appendix \ref{app:experiment_details}. Since there are only two distinct local
objectives, we set the number of clients $N=2$ and the number of sampled clients $S=1$,
and the number of groups $\bar{K} = 2$. All other settings can be found in Appendix
\ref{app:experiment_details}.

\textbf{Fashion-MNIST and CIFAR-10}
We evaluate each algorithm for training an image classifier, using logistic regression
for Fashion-MNIST and a two-layer CNN for CIFAR-10. To simulate heterogeneous data in
federated learning, we use a common protocol \cite{karimireddy2020scaffold,
wang2022unified}, to partition each dataset into client datasets according to a data
similarity parameter $s$. This protocol is detailed in Appendix
\ref{app:experiment_details}. Following \cite{wang2022unified}, we set the number of
clients $N=250$, data similarity $s=5\%$, and the number of sampled clients per round
$S=10$. For client participation, we set the number of groups $\bar{K}=5$, so that each
group contains clients that have majority label from two different classes. We run all
baselines with 5 different random seeds and report the mean results with error bars in
Section \ref{sec:results} (the radius of each error bar is 1 standard deviation). All
other settings can be found in Appendix \ref{app:experiment_details}.

Additional experimental results are provided in Appendix \ref{app:extra_experiments},
where we compare against extra baselines (FedAdam \cite{reddi2020adaptive}, FedYogi
\cite{reddi2020adaptive}, FedAvg-M \cite{cheng2023momentum}, and Amplified FedAvg with
FedProx regularization), and evaluate training under another non-i.i.d. client
participation pattern.

\subsection{Main Results} \label{sec:results}
Results for the synthetic experiment and CIFAR-10 are shown in Figure
\ref{fig:synthetic_cifar}, and results for Fashion-MNIST are shown in Figure
\ref{fig:fashion_ablation}. We make the following observations:

\begin{figure*}
\begin{center}
\includegraphics[width=0.99\textwidth]{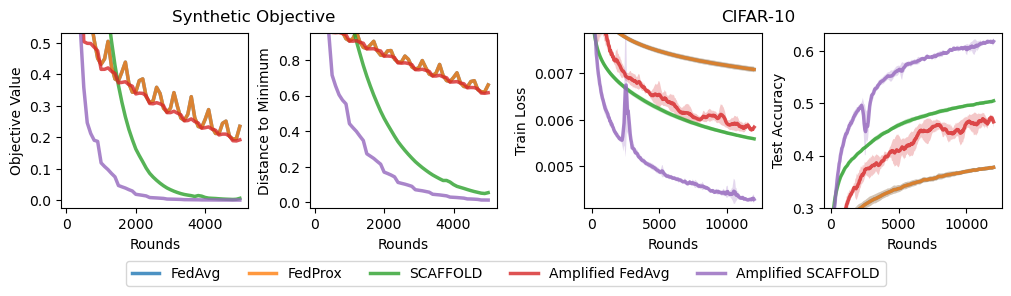}
\end{center}
\caption{Results for synthetic objective and CIFAR-10. Left: Amplified SCAFFOLD and
SCAFFOLD both converge to the global minimum, but Amplified SCAFFOLD converges
significantly faster. Right: Amplified SCAFFOLD converges to the best solution by a
significant margin. Note that in both cases, the curves for FedAvg and FedProx are
nearly overlapping.}
\label{fig:synthetic_cifar}
\end{figure*}

\textbf{Amplified SCAFFOLD converges the fastest.} In all three settings, Amplified
SCAFFOLD reaches the best overall solution among all algorithms (by all metrics) and
requires the fewest communication rounds. In the synthetic experiment, Amplified
SCAFFOLD requires 800 communication rounds to reach an objective value of $0.2$, while
SCAFFOLD requires 1900 rounds, and both FedAvg and Amplified FedAvg require 4800 rounds
to reach the same objective value.

\textbf{Amplified FedAvg is comparable to FedAvg.} Amplified FedAvg shows slight
improvement over FedAvg for the synthetic experiment and for Fashion-MNIST. Only for
CIFAR-10 is Amplified FedAvg significantly faster than FedAvg, but there it also
exhibits a reduction in stability. The underwhelming experimental performance of
Amplified FedAvg corroborates our discussion from Section \ref{sec:pattern_results};
Amplified FedAvg requires many communication rounds and suffers from data heterogeneity.

Contrary to our findings, the original evaluation of Amplified FedAvg
\cite{wang2022unified} showed a significant improvement over FedAvg. One explanation is
that the original evaluation employed pretraining using FedAvg, so that each algorithm
was evaluated only for fine-tuning. Our experiments suggest that Amplified FedAvg may
have limited improvement over FedAvg when training from scratch.

\textbf{SCAFFOLD beats Amplified FedAvg.} Despite a lack of theoretical guarantees under
non-i.i.d. participation, SCAFFOLD outperforms Amplified FedAvg in all settings. This
suggests that SCAFFOLD may have reasonable performance under some non-i.i.d.
participation patterns. For the synthetic objective and CIFAR-10, SCAFFOLD is
still significantly slower than Amplified SCAFFOLD.

\subsection{Ablation Study} \label{sec:ablation}
To understand how each algorithm's performance is affected by data heterogeneity, the
number of participating clients, and the number of client groups, we perform an ablation
study on Fashion-MNIST. First, we fix the data similarity $s = 5\%$ and number of groups
$\bar{K} = 5$ while varying the number of participating clients ($S$) over $\{5, 15, 20,
25\}$. Next, we fix $S = 10, \bar{K} = 5$ while varying the similarity $s$ over
$\{2.5\%, 10\%, 33\%, 100\%\}$. Lastly, we fix $s = 5\%, S = 10$ while varying the
number of client groups $\bar{K} \in \{2, 4, 6, 8\}$. In each of these 12 scenarios, we
evaluate all five algorithms using the same settings as detailed in Section
\ref{sec:setup}. We train with three random seeds for each algorithm, and report the
average results in Figure \ref{fig:fashion_ablation} (right).

\begin{figure*}
\begin{center}
\includegraphics[width=0.99\textwidth]{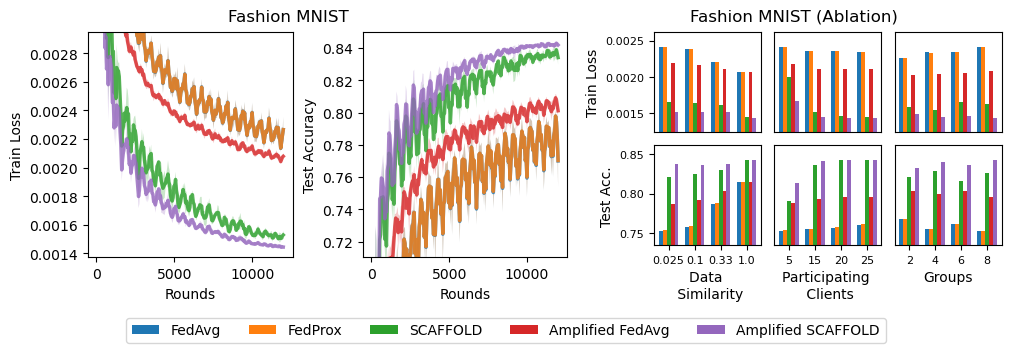}
\end{center}
\caption{Results for Fashion MNIST and ablation study. Left: Amplified SCAFFOLD reaches
the best solution, but SCAFFOLD is competitive. Other baselines are much slower. Right:
Amplified SCAFFOLD is robust to changes in data heterogeneity, number of
participating clients, and number of client groups.}
\label{fig:fashion_ablation}
\end{figure*}

\textbf{Amplified SCAFFOLD reaches the best solution in all settings.} Similarly to
Section \ref{sec:results}, Amplified SCAFFOLD consistently reaches the best solution in
terms of both training loss and testing accuracy. While our theoretical results provide
guarantees for optimization, these experiments show that Amplified SCAFFOLD also
exhibits superior generalization in a variety of settings.

\textbf{Robustness to data heterogeneity.} When changing from completely homogeneous
data ($s=100\%$) to extremely heterogeneous data ($s=2.5\%$), the test accuracy of
Amplified SCAFFOLD exhibits a very small decrease from $84.6\%$ to $84.45\%$, so that
our algorithm behaves nearly identically with homogeneous data as with extremely
heterogeneous data. All baselines suffer a larger decrease when transitioning from
homogeneous data to heterogeneous data.

\textbf{Robustness to number of participating clients.} The number of participating
clients has a smaller effect on performance than data heterogeneity, but some
degradation happens in the extreme case $S = 5$. In particular, SCAFFOLD has competitive
performance with large $S \geq 15$, but its test accuracy drops off significantly
compared to Amplified SCAFFOLD in the case $S = 5$.

\textbf{Robustness to number of client groups.} As $\bar{K}$ increases, FedAvg and
Amplified FedAvg get worse, while SCAFFOLD and Amplified SCAFFOLD maintain performance.
It makes intuitive sense for an algorithm to degrade as $\bar{K}$ increases, since a
larger $\bar{K}$ means that the participation is in some sense ``further" from i.i.d.
participation. Still, Amplified SCAFFOLD (and SCAFFOLD) are able to maintain performance
even as $\bar{K}$ increases. While the worst-case communication complexity of Amplified
SCAFFOLD (listed in Table \ref{tab:complexity}) actually increases with $\bar{K}$, these
experiments demonstrate that in practice, Amplified SCAFFOLD can maintain performance as
$\bar{K}$ increases.

\section{Conclusion} \label{sec:conclusion}
We propose Amplified SCAFFOLD, an optimization algorithm for federated learning under
periodic client participation, and prove that it exhibits reduced communication cost,
linear speedup, and is unaffected by data heterogeneity. We also show that Amplified
SCAFFOLD experimentally outperforms baselines on standard benchmarks under non-i.i.d.
client participation, and that the performance of our algorithm is robust to changes in
data heterogeneity and the number of participating clients.

\textbf{Limitations} While our analysis covers a general class of participation
patterns, it may not cover some participation patterns that appear in practice. Our
framework requires that all clients have an equal chance of participation across
well-defined windows of time that are known to the algorithm implementer, which may not
always hold. One such practical situation is where clients may freely join or leave the
federated learning process during training. Extending our algorithm and guarantees for
this situation would require a reformulation of the optimization problem, and possibly
additional assumptions about the participation structure. We leave such analysis for
future work.

\section*{Acknowledgements}
We would like to thank the anonymous reviewers for their helpful comments. This work is supported by the Institute for Digital Innovation fellowship from George Mason University,
a ORIEI seed funding, an IDIA P3 fellowship from George Mason University,
and a Cisco Faculty Research Award.
Experiments were partially run on Hopper, a research computing cluster provided by the Office of Research Computing at George Mason University (URL: https://orc.gmu.edu).

\bibliography{references}

\begin{thebibliography}{54}
\providecommand{\natexlab}[1]{#1}
\providecommand{\url}[1]{\texttt{#1}}
\expandafter\ifx\csname urlstyle\endcsname\relax
  \providecommand{\doi}[1]{doi: #1}\else
  \providecommand{\doi}{doi: \begingroup \urlstyle{rm}\Url}\fi

\bibitem[Acar et~al.(2021)Acar, Zhao, Matas, Mattina, Whatmough, and
  Saligrama]{acar2021federated}
Durmus Alp~Emre Acar, Yue Zhao, Ramon Matas, Matthew Mattina, Paul Whatmough,
  and Venkatesh Saligrama.
\newblock Federated learning based on dynamic regularization.
\newblock In \emph{International Conference on Learning Representations}, 2021.

\bibitem[Avdiukhin and Kasiviswanathan(2021)]{avdiukhin2021federated}
Dmitrii Avdiukhin and Shiva Kasiviswanathan.
\newblock Federated learning under arbitrary communication patterns.
\newblock In \emph{International Conference on Machine Learning}, pages
  425--435. PMLR, 2021.

\bibitem[Bonawitz et~al.(2019)Bonawitz, Eichner, Grieskamp, Huba, Ingerman,
  Ivanov, Kiddon, Kone{\v{c}}n{\`y}, Mazzocchi, McMahan,
  et~al.]{bonawitz2019towards}
Keith Bonawitz, Hubert Eichner, Wolfgang Grieskamp, Dzmitry Huba, Alex
  Ingerman, Vladimir Ivanov, Chloe Kiddon, Jakub Kone{\v{c}}n{\`y}, Stefano
  Mazzocchi, Brendan McMahan, et~al.
\newblock Towards federated learning at scale: System design.
\newblock \emph{Proceedings of machine learning and systems}, 1:\penalty0
  374--388, 2019.

\bibitem[Chen et~al.(2020)Chen, Horvath, and Richtarik]{chen2020optimal}
Wenlin Chen, Samuel Horvath, and Peter Richtarik.
\newblock Optimal client sampling for federated learning.
\newblock \emph{arXiv preprint arXiv:2010.13723}, 2020.

\bibitem[Cheng et~al.(2023)Cheng, Huang, Wu, and Yuan]{cheng2023momentum}
Ziheng Cheng, Xinmeng Huang, Pengfei Wu, and Kun Yuan.
\newblock Momentum benefits non-iid federated learning simply and provably.
\newblock In \emph{The Twelfth International Conference on Learning
  Representations}, 2023.

\bibitem[Cho et~al.(2020)Cho, Wang, and Joshi]{cho2020client}
Yae~Jee Cho, Jianyu Wang, and Gauri Joshi.
\newblock Client selection in federated learning: Convergence analysis and
  power-of-choice selection strategies.
\newblock \emph{arXiv preprint arXiv:2010.01243}, 2020.

\bibitem[Cho et~al.(2022)Cho, Wang, and Joshi]{cho2022towards}
Yae~Jee Cho, Jianyu Wang, and Gauri Joshi.
\newblock Towards understanding biased client selection in federated learning.
\newblock In \emph{International Conference on Artificial Intelligence and
  Statistics}, pages 10351--10375. PMLR, 2022.

\bibitem[Cho et~al.(2023)Cho, Sharma, Joshi, Xu, Kale, and
  Zhang]{cho2023convergence}
Yae~Jee Cho, Pranay Sharma, Gauri Joshi, Zheng Xu, Satyen Kale, and Tong Zhang.
\newblock On the convergence of federated averaging with cyclic client
  participation.
\newblock \emph{arXiv preprint arXiv:2302.03109}, 2023.

\bibitem[Ding et~al.(2017)Ding, Kulkarni, and Yekhanin]{ding2017collecting}
Bolin Ding, Janardhan Kulkarni, and Sergey Yekhanin.
\newblock Collecting telemetry data privately.
\newblock \emph{Advances in Neural Information Processing Systems}, 30, 2017.

\bibitem[Eichner et~al.(2019)Eichner, Koren, McMahan, Srebro, and
  Talwar]{eichner2019semi}
Hubert Eichner, Tomer Koren, Brendan McMahan, Nathan Srebro, and Kunal Talwar.
\newblock Semi-cyclic stochastic gradient descent.
\newblock In \emph{International Conference on Machine Learning}, pages
  1764--1773. PMLR, 2019.

\bibitem[Fraboni et~al.(2021)Fraboni, Vidal, Kameni, and
  Lorenzi]{fraboni2021clustered}
Yann Fraboni, Richard Vidal, Laetitia Kameni, and Marco Lorenzi.
\newblock Clustered sampling: Low-variance and improved representativity for
  clients selection in federated learning.
\newblock In \emph{International Conference on Machine Learning}, pages
  3407--3416. PMLR, 2021.

\bibitem[Glasgow et~al.(2022)Glasgow, Yuan, and Ma]{glasgow2022sharp}
Margalit~R Glasgow, Honglin Yuan, and Tengyu Ma.
\newblock Sharp bounds for federated averaging (local sgd) and continuous
  perspective.
\newblock In \emph{International Conference on Artificial Intelligence and
  Statistics}, pages 9050--9090. PMLR, 2022.

\bibitem[Grudzie{\'n} et~al.(2023)Grudzie{\'n}, Malinovsky, and
  Richt{\'a}rik]{grudzien2023improving}
Micha{\l} Grudzie{\'n}, Grigory Malinovsky, and Peter Richt{\'a}rik.
\newblock Improving accelerated federated learning with compression and
  importance sampling.
\newblock In \emph{Federated Learning and Analytics in Practice: Algorithms,
  Systems, Applications, and Opportunities}, 2023.

\bibitem[Gu et~al.(2021)Gu, Huang, Zhang, and Huang]{gu2021fast}
Xinran Gu, Kaixuan Huang, Jingzhao Zhang, and Longbo Huang.
\newblock Fast federated learning in the presence of arbitrary device
  unavailability.
\newblock \emph{Advances in Neural Information Processing Systems},
  34:\penalty0 12052--12064, 2021.

\bibitem[Huba et~al.(2022)Huba, Nguyen, Malik, Zhu, Rabbat, Yousefpour, Wu,
  Zhan, Ustinov, Srinivas, et~al.]{huba2022papaya}
Dzmitry Huba, John Nguyen, Kshitiz Malik, Ruiyu Zhu, Mike Rabbat, Ashkan
  Yousefpour, Carole-Jean Wu, Hongyuan Zhan, Pavel Ustinov, Harish Srinivas,
  et~al.
\newblock Papaya: Practical, private, and scalable federated learning.
\newblock \emph{Proceedings of Machine Learning and Systems}, 4:\penalty0
  814--832, 2022.

\bibitem[Kairouz et~al.(2021)Kairouz, McMahan, Avent, Bellet, Bennis, Bhagoji,
  Bonawitz, Charles, Cormode, Cummings, et~al.]{kairouz2021advances}
Peter Kairouz, H~Brendan McMahan, Brendan Avent, Aur{\'e}lien Bellet, Mehdi
  Bennis, Arjun~Nitin Bhagoji, Kallista Bonawitz, Zachary Charles, Graham
  Cormode, Rachel Cummings, et~al.
\newblock Advances and open problems in federated learning.
\newblock \emph{Foundations and trends{\textregistered} in machine learning},
  14\penalty0 (1--2):\penalty0 1--210, 2021.

\bibitem[Karimireddy et~al.(2020)Karimireddy, Kale, Mohri, Reddi, Stich, and
  Suresh]{karimireddy2020scaffold}
Sai~Praneeth Karimireddy, Satyen Kale, Mehryar Mohri, Sashank Reddi, Sebastian
  Stich, and Ananda~Theertha Suresh.
\newblock Scaffold: Stochastic controlled averaging for federated learning.
\newblock In \emph{International conference on machine learning}, pages
  5132--5143. PMLR, 2020.

\bibitem[Khaled et~al.(2020)Khaled, Mishchenko, and
  Richt{\'a}rik]{khaled2020tighter}
Ahmed Khaled, Konstantin Mishchenko, and Peter Richt{\'a}rik.
\newblock Tighter theory for local sgd on identical and heterogeneous data.
\newblock In \emph{International Conference on Artificial Intelligence and
  Statistics}, pages 4519--4529. PMLR, 2020.

\bibitem[Kone{\v{c}}n{\`y} et~al.(2016)Kone{\v{c}}n{\`y}, McMahan, Yu,
  Richt{\'a}rik, Suresh, and Bacon]{konevcny2016federated}
Jakub Kone{\v{c}}n{\`y}, H~Brendan McMahan, Felix~X Yu, Peter Richt{\'a}rik,
  Ananda~Theertha Suresh, and Dave Bacon.
\newblock Federated learning: Strategies for improving communication
  efficiency.
\newblock \emph{arXiv preprint arXiv:1610.05492}, 2016.

\bibitem[Krizhevsky et~al.(2009)Krizhevsky, Hinton,
  et~al.]{krizhevsky2009learning}
Alex Krizhevsky, Geoffrey Hinton, et~al.
\newblock Learning multiple layers of features from tiny images.
\newblock 2009.

\bibitem[Li et~al.(2020{\natexlab{a}})Li, Sahu, Talwalkar, and
  Smith]{li2020federated1}
Tian Li, Anit~Kumar Sahu, Ameet Talwalkar, and Virginia Smith.
\newblock Federated learning: Challenges, methods, and future directions.
\newblock \emph{IEEE Signal Processing Magazine}, 37\penalty0 (3):\penalty0
  50--60, 2020{\natexlab{a}}.

\bibitem[Li et~al.(2020{\natexlab{b}})Li, Sahu, Zaheer, Sanjabi, Talwalkar, and
  Smith]{li2020federated}
Tian Li, Anit~Kumar Sahu, Manzil Zaheer, Maziar Sanjabi, Ameet Talwalkar, and
  Virginia Smith.
\newblock Federated optimization in heterogeneous networks.
\newblock \emph{Proceedings of Machine learning and systems}, 2:\penalty0
  429--450, 2020{\natexlab{b}}.

\bibitem[Li et~al.(2019)Li, Huang, Yang, Wang, and Zhang]{li2019convergence}
Xiang Li, Kaixuan Huang, Wenhao Yang, Shusen Wang, and Zhihua Zhang.
\newblock On the convergence of fedavg on non-iid data.
\newblock \emph{arXiv preprint arXiv:1907.02189}, 2019.

\bibitem[Lian et~al.(2018)Lian, Zhang, Zhang, and Liu]{lian2018asynchronous}
Xiangru Lian, Wei Zhang, Ce~Zhang, and Ji~Liu.
\newblock Asynchronous decentralized parallel stochastic gradient descent.
\newblock In \emph{International Conference on Machine Learning}, pages
  3043--3052. PMLR, 2018.

\bibitem[Lim et~al.(2020)Lim, Luong, Hoang, Jiao, Liang, Yang, Niyato, and
  Miao]{lim2020federated}
Wei Yang~Bryan Lim, Nguyen~Cong Luong, Dinh~Thai Hoang, Yutao Jiao, Ying-Chang
  Liang, Qiang Yang, Dusit Niyato, and Chunyan Miao.
\newblock Federated learning in mobile edge networks: A comprehensive survey.
\newblock \emph{IEEE Communications Surveys \& Tutorials}, 22\penalty0
  (3):\penalty0 2031--2063, 2020.

\bibitem[Malinovsky et~al.(2023)Malinovsky, Horv{\'a}th, Burlachenko, and
  Richt{\'a}rik]{malinovsky2023federated}
Grigory Malinovsky, Samuel Horv{\'a}th, Konstantin Burlachenko, and Peter
  Richt{\'a}rik.
\newblock Federated learning with regularized client participation.
\newblock \emph{arXiv preprint arXiv:2302.03662}, 2023.

\bibitem[McMahan et~al.(2017{\natexlab{a}})McMahan, Moore, Ramage, Hampson, and
  y~Arcas]{mcmahan2017communication}
Brendan McMahan, Eider Moore, Daniel Ramage, Seth Hampson, and Blaise~Aguera
  y~Arcas.
\newblock Communication-efficient learning of deep networks from decentralized
  data.
\newblock In \emph{Artificial intelligence and statistics}, pages 1273--1282.
  PMLR, 2017{\natexlab{a}}.

\bibitem[McMahan et~al.(2017{\natexlab{b}})McMahan, Ramage, Talwar, and
  Zhang]{mcmahan2017learning}
H~Brendan McMahan, Daniel Ramage, Kunal Talwar, and Li~Zhang.
\newblock Learning differentially private recurrent language models.
\newblock \emph{arXiv preprint arXiv:1710.06963}, 2017{\natexlab{b}}.

\bibitem[Mothukuri et~al.(2021)Mothukuri, Parizi, Pouriyeh, Huang,
  Dehghantanha, and Srivastava]{mothukuri2021survey}
Viraaji Mothukuri, Reza~M Parizi, Seyedamin Pouriyeh, Yan Huang, Ali
  Dehghantanha, and Gautam Srivastava.
\newblock A survey on security and privacy of federated learning.
\newblock \emph{Future Generation Computer Systems}, 115:\penalty0 619--640,
  2021.

\bibitem[Patel et~al.(2022)Patel, Wang, Woodworth, Bullins, and
  Srebro]{patel2022towards}
Kumar~Kshitij Patel, Lingxiao Wang, Blake~E Woodworth, Brian Bullins, and Nati
  Srebro.
\newblock Towards optimal communication complexity in distributed non-convex
  optimization.
\newblock In S.~Koyejo, S.~Mohamed, A.~Agarwal, D.~Belgrave, K.~Cho, and A.~Oh,
  editors, \emph{Advances in Neural Information Processing Systems}, volume~35,
  pages 13316--13328. Curran Associates, Inc., 2022.
\newblock URL
  \url{https://proceedings.neurips.cc/paper_files/paper/2022/file/56bd21259e28ebdc4d7e1503733bf421-Paper-Conference.pdf}.

\bibitem[Paulik et~al.(2021)Paulik, Seigel, Mason, Telaar, Kluivers, van Dalen,
  Lau, Carlson, Granqvist, Vandevelde, et~al.]{paulik2021federated}
Matthias Paulik, Matt Seigel, Henry Mason, Dominic Telaar, Joris Kluivers,
  Rogier van Dalen, Chi~Wai Lau, Luke Carlson, Filip Granqvist, Chris
  Vandevelde, et~al.
\newblock Federated evaluation and tuning for on-device personalization: System
  design \& applications.
\newblock \emph{arXiv preprint arXiv:2102.08503}, 2021.

\bibitem[Reddi et~al.(2021)Reddi, Charles, Zaheer, Garrett, Rush, Konecny,
  Kumar, and McMahan]{reddi2020adaptive}
Sashank Reddi, Zachary Charles, Manzil Zaheer, Zachary Garrett, Keith Rush,
  Jakub Konecny, Sanjiv Kumar, and H~Brendan McMahan.
\newblock Adaptive federated optimization.
\newblock \emph{ICLR}, 2021.

\bibitem[Rizk et~al.(2022)Rizk, Vlaski, and Sayed]{rizk2022federated}
Elsa Rizk, Stefan Vlaski, and Ali~H Sayed.
\newblock Federated learning under importance sampling.
\newblock \emph{IEEE Transactions on Signal Processing}, 70:\penalty0
  5381--5396, 2022.

\bibitem[Ruan et~al.(2021)Ruan, Zhang, Liang, and Joe-Wong]{ruan2021towards}
Yichen Ruan, Xiaoxi Zhang, Shu-Che Liang, and Carlee Joe-Wong.
\newblock Towards flexible device participation in federated learning.
\newblock In \emph{International Conference on Artificial Intelligence and
  Statistics}, pages 3403--3411. PMLR, 2021.

\bibitem[Stich(2018)]{stich2018local}
Sebastian~U Stich.
\newblock Local sgd converges fast and communicates little.
\newblock \emph{arXiv preprint arXiv:1805.09767}, 2018.

\bibitem[Tyurin et~al.(2022)Tyurin, Sun, Burlachenko, and
  Richt{\'a}rik]{tyurin2022sharper}
Alexander Tyurin, Lukang Sun, Konstantin~Pavlovich Burlachenko, and Peter
  Richt{\'a}rik.
\newblock Sharper rates and flexible framework for nonconvex sgd with client
  and data sampling.
\newblock \emph{Transactions on Machine Learning Research}, 2022.

\bibitem[Wang and Joshi(2018)]{wang2018cooperative}
Jianyu Wang and Gauri Joshi.
\newblock Cooperative sgd: A unified framework for the design and analysis of
  communication-efficient sgd algorithms.
\newblock \emph{arXiv preprint arXiv:1808.07576}, 2018.

\bibitem[Wang et~al.(2021)Wang, Charles, Xu, Joshi, McMahan, Al-Shedivat,
  Andrew, Avestimehr, Daly, Data, et~al.]{wang2021field}
Jianyu Wang, Zachary Charles, Zheng Xu, Gauri Joshi, H~Brendan McMahan, Maruan
  Al-Shedivat, Galen Andrew, Salman Avestimehr, Katharine Daly, Deepesh Data,
  et~al.
\newblock A field guide to federated optimization.
\newblock \emph{arXiv preprint arXiv:2107.06917}, 2021.

\bibitem[Wang and Ji(2022)]{wang2022unified}
Shiqiang Wang and Mingyue Ji.
\newblock A unified analysis of federated learning with arbitrary client
  participation.
\newblock \emph{Advances in Neural Information Processing Systems},
  35:\penalty0 19124--19137, 2022.

\bibitem[Wang and Ji(2023)]{wang2023lightweight}
Shiqiang Wang and Mingyue Ji.
\newblock A lightweight method for tackling unknown participation probabilities
  in federated averaging.
\newblock \emph{arXiv preprint arXiv:2306.03401}, 2023.

\bibitem[Wei et~al.(2020)Wei, Li, Ding, Ma, Yang, Farokhi, Jin, Quek, and
  Vincent~Poor]{wei2020federated}
Kang Wei, Jun Li, Ming Ding, Chuan Ma, Howard~H. Yang, Farhad Farokhi, Shi Jin,
  Tony Q.~S. Quek, and H.~Vincent~Poor.
\newblock Federated learning with differential privacy: Algorithms and
  performance analysis.
\newblock \emph{IEEE Transactions on Information Forensics and Security},
  15:\penalty0 3454--3469, 2020.
\newblock \doi{10.1109/TIFS.2020.2988575}.

\bibitem[Woodworth et~al.(2020{\natexlab{a}})Woodworth, Patel, Stich, Dai,
  Bullins, Mcmahan, Shamir, and Srebro]{woodworth2020local}
Blake Woodworth, Kumar~Kshitij Patel, Sebastian Stich, Zhen Dai, Brian Bullins,
  Brendan Mcmahan, Ohad Shamir, and Nathan Srebro.
\newblock Is local sgd better than minibatch sgd?
\newblock In \emph{International Conference on Machine Learning}, pages
  10334--10343. PMLR, 2020{\natexlab{a}}.

\bibitem[Woodworth et~al.(2020{\natexlab{b}})Woodworth, Patel, and
  Srebro]{woodworth2020minibatch}
Blake~E Woodworth, Kumar~Kshitij Patel, and Nati Srebro.
\newblock Minibatch vs local sgd for heterogeneous distributed learning.
\newblock \emph{Advances in Neural Information Processing Systems},
  33:\penalty0 6281--6292, 2020{\natexlab{b}}.

\bibitem[Xiao et~al.(2017)Xiao, Rasul, and Vollgraf]{xiao2017/online}
Han Xiao, Kashif Rasul, and Roland Vollgraf.
\newblock Fashion-mnist: a novel image dataset for benchmarking machine
  learning algorithms.
\newblock 2017.

\bibitem[Yan et~al.(2020)Yan, Niu, Ding, Zheng, Wu, Chen, Tang, and
  Wu]{yan2020distributed}
Yikai Yan, Chaoyue Niu, Yucheng Ding, Zhenzhe Zheng, Fan Wu, Guihai Chen,
  Shaojie Tang, and Zhihua Wu.
\newblock Distributed non-convex optimization with sublinear speedup under
  intermittent client availability.
\newblock \emph{arXiv preprint arXiv:2002.07399}, 2020.

\bibitem[Yang et~al.(2020)Yang, Fang, and Liu]{yang2020achieving}
Haibo Yang, Minghong Fang, and Jia Liu.
\newblock Achieving linear speedup with partial worker participation in non-iid
  federated learning.
\newblock In \emph{International Conference on Learning Representations}, 2020.

\bibitem[Yang et~al.(2022)Yang, Zhang, Khanduri, and Liu]{yang2022anarchic}
Haibo Yang, Xin Zhang, Prashant Khanduri, and Jia Liu.
\newblock Anarchic federated learning.
\newblock In \emph{International Conference on Machine Learning}, pages
  25331--25363. PMLR, 2022.

\bibitem[Yang et~al.(2018)Yang, Andrew, Eichner, Sun, Li, Kong, Ramage, and
  Beaufays]{yang2018applied}
Timothy Yang, Galen Andrew, Hubert Eichner, Haicheng Sun, Wei Li, Nicholas
  Kong, Daniel Ramage, and Fran{\c{c}}oise Beaufays.
\newblock Applied federated learning: Improving google keyboard query
  suggestions.
\newblock \emph{arXiv preprint arXiv:1812.02903}, 2018.

\bibitem[Yu et~al.(2019{\natexlab{a}})Yu, Jin, and Yang]{yu_linear}
Hao Yu, Rong Jin, and Sen Yang.
\newblock On the linear speedup analysis of communication efficient momentum
  {SGD} for distributed non-convex optimization.
\newblock In \emph{Proceedings of the 36th International Conference on Machine
  Learning, {ICML} 2019, 9-15 June 2019, Long Beach, California, {USA}}, pages
  7184--7193, 2019{\natexlab{a}}.

\bibitem[Yu et~al.(2019{\natexlab{b}})Yu, Yang, and Zhu]{yu2019parallel}
Hao Yu, Sen Yang, and Shenghuo Zhu.
\newblock Parallel restarted sgd with faster convergence and less
  communication: Demystifying why model averaging works for deep learning.
\newblock In \emph{Proceedings of the AAAI Conference on Artificial
  Intelligence}, volume~33, pages 5693--5700, 2019{\natexlab{b}}.

\bibitem[Yuan and Ma(2020)]{yuan2020federated}
Honglin Yuan and Tengyu Ma.
\newblock Federated accelerated stochastic gradient descent.
\newblock \emph{Advances in Neural Information Processing Systems},
  33:\penalty0 5332--5344, 2020.

\bibitem[Zhang et~al.(2021)Zhang, Xie, Bai, Yu, Li, and Gao]{zhang2021survey}
Chen Zhang, Yu~Xie, Hang Bai, Bin Yu, Weihong Li, and Yuan Gao.
\newblock A survey on federated learning.
\newblock \emph{Knowledge-Based Systems}, 216:\penalty0 106775, 2021.

\bibitem[Zhu et~al.(2021{\natexlab{a}})Zhu, Xu, Chen, Kone{\v{c}}n{\`y}, Hard,
  and Goldstein]{zhu2021diurnal}
Chen Zhu, Zheng Xu, Mingqing Chen, Jakub Kone{\v{c}}n{\`y}, Andrew Hard, and
  Tom Goldstein.
\newblock Diurnal or nocturnal? federated learning of multi-branch networks
  from periodically shifting distributions.
\newblock In \emph{International Conference on Learning Representations},
  2021{\natexlab{a}}.

\bibitem[Zhu et~al.(2021{\natexlab{b}})Zhu, Xu, Liu, and Jin]{zhu2021federated}
Hangyu Zhu, Jinjin Xu, Shiqing Liu, and Yaochu Jin.
\newblock Federated learning on non-iid data: A survey.
\newblock \emph{Neurocomputing}, 465:\penalty0 371--390, 2021{\natexlab{b}}.

\end{thebibliography}

%%%%%%%%%%%%%%%%%%%%%%%%%%%%%%%%%%%%%%%%%%%%%%%%%%%%%%%%%%%%

\newpage
\appendix
\tableofcontents

\section{Proof of Theorem \ref{thm:main_convergence}} \label{app:main_proof}

\subsection{Preliminary Definitions} \label{app:definitions}
Let $\gG_r$ denote the filtration generated by $\{\xi_{r,k}^i : 0 \leq k \leq I - 1, i
\in [N]\}$, that is, by the randomness in the stochastic gradients during round $r$.
Also, let $\gQ_{r_0}$ denote the filtration generated by $\{q_r^i : r_0 \leq r <
r_0 + P, i \in [N]\}$, that is, the randomness in client sampling between rounds $r_0$
and $r_0 + P - 1$ (inclusive). Also, let $\gG$ denote the filtration generated by $\gG_0 \cup \gG_1 \cup \ldots \cup \gG_{R-1}$ and $\gQ$ denote the filtration generated by
$\gQ_0 \cup \gQ_{P} \cup \ldots \cup \gQ_{R - P}$. Similarly, let $\gG_{:r}$ denote the filtration generated by $\gG_0 \cup \gG_1 \cup \ldots \cup \gG_{r-1}$ and $\gQ_{:r_0}$ denote the filtration
generated by $\gQ_0 \cup \gQ_P \cup \ldots \cup \gQ_{r_0-P}$. Lastly, denote $\mathbb{E}_{r_0}
[\cdot] = \mathbb{E} [\cdot | \gQ_{:r_0}, \gG_{:r_0}]$.

In order to analyze the control variate errors (i.e., $\|\nabla f_i(\bvx_{r_0}) - \mG_{r_0}^i\|$), we introduce notation to refer to the iterates whose stochastic gradients were used to construct $\mG_{r_0}^i$. These iterates are exactly the iterates during the most recent window before $r_0$ in which client $i$ was sampled, i.e., where $\frac{1}{P} \sum_{s=mP}^{(m+1)P - 1} q_s^i > 0$.
For each $r \in \{r_0, r_0 + 1, \ldots, r_0 + P - 1\}$, $k \in \{0, \ldots, I-1\}$, and
$i \in [N]$, let
\begin{align*}
    t_i(r) = \begin{cases}
        r & \text{if } \bar{q}_{r_0}^i > 0 \\
        t_i(r-P) & \text{if } \bar{q}_{r_0}^i = 0
    \end{cases}
\end{align*}
with the initialization $t_i(r) = 0$ for all $r \in \{-P, \ldots, -1\}$. Then denote
\begin{align*}
    \vy_{r,k}^i &= \vx_{t_i(r),k}^i \\
    z_r^i &= q_{t_i(r)}^i \\
    \zeta_{r,k}^i &= \xi_{t_i(r),k}^i.
\end{align*}
Also, denote $\bar{z}_{r_0}^i = \frac{1}{P} \sum_{r=r_0}^{r_0+P-1} z_r^i$. Then we can
rewrite
\begin{align*}
    \mG_{r_0}^i = \frac{1}{P \bar{z}_{r_0-P}^i I} \sum_{s=r_0-P}^{r_0-1} z_s^i \sum_{k=0}^{I-1} \nabla F_i(\vy_{s,k}^i; \zeta_{s,k}^i).
\end{align*}

Define
\begin{align*}
    \bvx_{r,k} &= \sum_{i=1}^N q_r^i \vx_{r,k}^i \\
    D_{r,k} &= \sum_{i=1}^N q_r^i \mathbb{E} \left[ \left\| \vx_{r,k}^i - \bvx_{r,k} \right\|^2 \middle| \gQ \right] \\
    M_{r,k} &= \mathbb{E} \left[ \left\| \bvx_{r,k} - \bvx_{r_0} \right\|^2 \middle| \gQ \right] \\
    S_{r_0}^i &= \frac{1}{IP} \sum_{s=r_0-P}^{r_0-1} \sum_{k=0}^{I-1} \frac{z_s^i}{\bar{z}_{r_0-P}^i} \mathbb{E} \left[ \left\| \vy_{s,k}^i - \bvx_{r_0} \right\|^2 \middle| \gQ \right] \\
    w_{r_0}^i &= \frac{1}{N} \sum_{j=1}^N \frac{\indicator{\bar{q}_{r_0}^j > 0}}{P \bar{q}_{r_0}^j} \sum_{s=r_0}^{r_0+P-1} q_s^i q_s^j \\
    v_{r_0}^i &= \bar{q}_{r_0}^i - \frac{1}{N} \\
    \Lambda_{r_0}^i &= \frac{ \frac{1}{P} \sum_{s=r_0-P}^{r_0-1} \left( z_s^i \right)^2 }{\left( \frac{1}{P} \sum_{s=r_0-P}^{r_0-1} z_s^i \right)^2}.
\end{align*}
As discussed in the main body, $\bvx_{r,k}$ is a weighted average over local client models, weighted according to client participation. $D_{r,k}$ and $M_{r,k}$ are the drift terms described in the proof sketch of Section \ref{sec:proof_sketch}.
In the proof sketch, we also informally discuss the control variate error $C_{r_0}^i = \|\nabla f_i(\bvx_{r_0}) - \mG_{r_0}^i\|^2$. This error is closely related to the term $S_{r_0}^i$ defined above, since $\mathbb{E} \left[ \left\| \nabla f_i(\bvx_{r_0}) - \mathbb{E}[\mG_{r_0}^i | \gQ] \right\|^2 \middle| \gQ \right] = L^2 S_{r_0}^i$.

The quantities $w_{r_0}^i, v_{r_0}^i$ and $\Lambda_{r_0}^i$ arise in our proof from the
use of control variates under non-i.i.d. client participation. As discussed in the proof
sketch (Section \ref{sec:proof_sketch}), bounding the errors introduced by control
variates involves a non-uniform average of non-uniform averages of error terms. This
nested non-uniform average can be bounded by a uniform average as long as the term
$w_{r_0}^i$ is not too large, which is exactly the requirement stated in Equation
\ref{eq:w_cond}. On the other hand, the term $\Lambda_{r_0}^i$ arises while bounding the
variance of the update $\|\bvx_{r_0+P} - \bvx_{r_0}\|$, and it represents the sample
size for which the correction $\mG_{r_0}^i$ is computed. When the client sampling
distribution is closer to uniform, then each client may participate frequently during
each window of $P$ rounds, and the effective sample size used to compute each control
variate is large. In this case, the variance of the update $\| \bvx_{r_0+P} - \bvx_{r_0}
\|$ will be smaller, since larger sample size implies smaller variance. For general
client sampling distributions, $\Lambda_{r_0}^i$ describes the reduction of variance of
one component of the update (i.e. $\mG_{r_0}^i$) due to the sampling of stochastic
gradients during the previous window of $P$ rounds. In our analysis, the variance of the
update $\|\bvx_{r_0+P} - \bvx_{r_0}\|$ can be bounded in terms of the variance of
$\mG_{r_0}^i$, which in turn depends on $\Lambda_{r_0}^i$. The condition in
\Eqref{eq:v_cond} essentially enforces that clients participate sufficiently uniformly,
so that the variance of $\mG_{r_0}^i$ can be bounded. We note that both
\Eqref{eq:w_cond} and \Eqref{eq:v_cond} are satisfied by i.i.d. participation,
regularized participation, and cyclic participation.

Finally, for ease of exposition, one clarification must be made about the conditional
expectation $\mathbb{E}[ \cdot | \gQ ]$ as it appears in, for example, $M_{r,k}$. The
filtration $\gQ$ contains randomness from every round from $0$ to $R-1$, but the
expression of which we are taking expectation may only depend on randomness for a
smaller subset rounds.  For example, in $M_{r,k}$, the expression $\|\bvx_{r,k} -
\bvx_{r_0}\|^2$ only depends on the randomness up to round $r \leq r_0 + P$, so that
\begin{equation*}
    \mathbb{E} \left[ \|\bvx_{r,k} - \bvx_{r_0}\|^2 \middle| \gQ \right] = \mathbb{E} \left[ \|\bvx_{r,k} - \bvx_{r_0}\|^2 \middle| \gQ_{:r_0+P} \right].
\end{equation*}
In several places throughout the proof, we will replace $\mathbb{E}[ \cdot | \gQ]$ by
$\mathbb{E}[ \cdot | \gQ_{:r_0+P}]$ in similar situations.

\subsection{Proofs}
\begin{lemma} \label{lem:dm_ub}
If $\eta \leq \frac{1}{60 LIP}$ and $\mathbb{E}_{r_0}[\bar{q}_{r_0}^i] = \frac{1}{N}$, then
\begin{align}
    D_{r,k} + M_{r,k} &\leq 108 \eta^2 I^2 P^2 \mathbb{E} \left[ \left\| \nabla f(\bvx_{r_0}) \right\|^2 \middle| \gQ \right] + \left( 75 \eta^2 I + 65 \eta^2 IP \rho^2 \right) \sigma^2 \nonumber \\
    &\quad + 109 \eta^2 L^2 I^2 P^2 \sum_{i=1}^N \left( \bar{q}_{r_0}^i + \frac{1}{N} \right) S_{r_0}^i + 36 \eta^2 L^2 I^2 \sum_{i=1}^N q_r^i S_{r_0}^i , \label{eq:dm_ub_single}
\end{align}
and
\begin{align}
    \mathbb{E} \left[ \sum_{r=r_0}^{r_0+P-1} \sum_{k=0}^{I-1} (D_{r,k} + M_{r,k}) \right] &\leq 108 \eta^2 I^3 P^3 \mathbb{E} \left[ \left\| \nabla f(\bvx_{r_0}) \right\|^2 \right] + \left( 75 \eta^2 I^2 P + 65 \eta^2 I^2 P^2 \rho^2 \right) \sigma^2 \nonumber \\
    &\quad + 254 \eta^2 L^2 I^3 P^3 \frac{1}{N} \sum_{i=1}^N \mathbb{E} \left[ S_{r_0}^i \right]. \label{eq:dm_ub_sum}
\end{align}
\end{lemma}

\begin{proof}
Denote
\begin{align*}
    \bar{\mG}_{r_0}^i &= \frac{1}{P \bar{z}_{r_0-P}^i I} \sum_{s=r_0-P}^{r_0-1} z_s^i \sum_{k=0}^{I-1} \nabla f_i(\vy_{s,k}^i) \\
    \bar{\mG}_{r_0} &= \frac{1}{N} \sum_{i=1}^N \bar{\mG}_{r_0}^i.
\end{align*}
Also, denote $\vg_{r,k}^i = \nabla F_i(\vx_{r,k}^i; \xi_{r,k}^i) - \mG_{r_0}^i +
\mG_{r_0}$ and $\bvg_{r,k}^i = \nabla f_i(\vx_{r,k}^i) - \bar{\mG}_{r_0}^i +
\bar{\mG}_{r_0}$. $\bar{\mG}_{r_0}^i$ is the analogue of $\bar{\mG}_{r_0}$ that depends
on deterministic gradients $\nabla f_i(\vy_{s,k}^i)$ instead of stochastic gradients.

\paragraph{Variance of updates}
We first compute the errors $\mathbb{E} \left[ \left\| \vg_{r,k}^i - \bar{\vg}_{r,k}^i
\right\|^2 \middle| \gQ \right]$ and $\mathbb{E} \left[ \left\| \sum_{i=1}^N q_r^i
\left( \vg_{r,k}^i - \bar{\vg}_{r,k}^i \right) \right\|^2 \middle| \gQ \right]$, which
will be needed in several places.
\begin{align}
    \mathbb{E} \left[ \left\| \mG_{r_0}^i - \bar{\mG}_{r_0}^i \right\|^2 \middle| \gQ \right] &= \mathbb{E} \left[ \left\| \frac{1}{P \bar{z}_{r_0-P}^i I} \sum_{s=r_0-P}^{r_0-1} \sum_{k=0}^{I-1} z_s^i \left( \nabla F_i(\vy_{s,k}^i; \zeta_{s,k}^i) - \nabla f_i(\vy_{s,k}^i) \right) \right\|^2 \middle| \gQ \right] \nonumber \\
    &\Eqmark{i}{\leq} \frac{1}{P \bar{z}_{r_0-P}^i I} \sum_{s=r_0-P}^{r_0-1} \sum_{k=0}^{I-1} z_s^i \mathbb{E} \left[ \left\| \nabla F_i(\vy_{s,k}^i; \zeta_{s,k}^i) - \nabla f_i(\vy_{s,k}^i) \right\|^2 \middle| \gQ \right] \nonumber \\
    &\leq \frac{1}{P \bar{z}_{r_0-P}^i I} \sum_{s=r_0-P}^{r_0-1} \sum_{k=0}^{I-1} z_s^i \sigma^2 \nonumber \\
    &= \sigma^2, \label{eq:update_var_inter}
\end{align}
where $(i)$ uses Jensen's inequality. Therefore
\begin{align}
    &\mathbb{E} \left[ \left\| \vg_{r,k}^i - \bar{\vg}_{r,k}^i \right\|^2 \middle| \gQ \right] \nonumber \\
    &\quad = \mathbb{E} \left[ \left\| \left( \nabla F_i(\vx_{r,k}^i; \xi_{r,k}^i) - \nabla f_i(\vx_{r,k}^i) \right) - \left( \mG_{r_0}^i - \bar{\mG}_{r_0}^i \right) + \left( \mG_{r_0} - \bar{\mG}_{r_0} \right) \right\|^2 \middle| \gQ \right] \nonumber \\
    &\quad \leq 3 \mathbb{E} \left[ \left\| \nabla F_i(\vx_{r,k}^i; \xi_{r,k}^i) - \nabla f_i(\vx_{r,k}^i) \right\|^2 \middle| \gQ \right] + 3 \mathbb{E} \left[ \left\| \mG_{r_0}^i - \bar{\mG}_{r_0}^i \right\|^2 \middle| \gQ \right] \nonumber \\
    &\quad\quad + 3 \mathbb{E} \left[ \left\| \mG_{r_0} - \bar{\mG}_{r_0} \right\|^2 \middle| \gQ \right] \nonumber \\
    &\quad = 3 \mathbb{E} \left[ \left\| \nabla F_i(\vx_{r,k}^i; \xi_{r,k}^i) - \nabla f_i(\vx_{r,k}^i) \right\|^2 \middle| \gQ \right] + 3 \mathbb{E} \left[ \left\| \mG_{r_0}^i - \bar{\mG}_{r_0}^i \right\|^2 \middle| \gQ \right] \nonumber \\
    &\quad\quad + 3 \mathbb{E} \left[ \left\| \frac{1}{N} \sum_{j=1}^N \left( \mG_{r_0}^j - \bar{\mG}_{r_0}^j \right) \right\|^2 \middle| \gQ \right] \nonumber \\
    &\quad = 3 \mathbb{E} \left[ \left\| \nabla F_i(\vx_{r,k}^i; \xi_{r,k}^i) - \nabla f_i(\vx_{r,k}^i) \right\|^2 \middle| \gQ \right] + 3 \mathbb{E} \left[ \left\| \mG_{r_0}^i - \bar{\mG}_{r_0}^i \right\|^2 \middle| \gQ \right] \nonumber \\
    &\quad\quad + 3 \frac{1}{N} \sum_{j=1}^N \mathbb{E} \left[ \left\| \mG_{r_0}^j - \bar{\mG}_{r_0}^j \right\|^2 \middle| \gQ \right] \nonumber \\
    &\quad \leq 9 \sigma^2, \label{eq:update_var}
\end{align}
where the last line uses \Eqref{eq:update_var_inter}. Also
\begin{align}
    &\mathbb{E} \left[ \left\| \sum_{i=1}^N q_r^i \left( \vg_{r,k}^i - \bar{\vg}_{r,k}^i \right) \right\|^2 \middle| \gQ \right] \nonumber \\
    &\quad = \mathbb{E} \left[ \left\| \sum_{i=1}^N q_r^i \left( \left( \nabla F_i(\vx_{r,k}^i; \xi_{r,k}^i) - \nabla f_i(\vx_{r,k}^i) \right) - \left( \mG_{r_0}^i - \bar{\mG}_{r_0}^i \right) + \left( \mG_{r_0} - \bar{\mG}_{r_0} \right) \right) \right\|^2 \middle| \gQ \right] \nonumber \\
    &\quad \leq 3 \mathbb{E} \left[ \left\| \sum_{i=1}^N q_r^i \left( \nabla F_i(\vx_{r,k}^i; \xi_{r,k}^i) - \nabla f_i(\vx_{r,k}^i) \right) \right\|^2 \middle| \gQ \right] + 3 \mathbb{E} \left[ \left\| \sum_{i=1}^N q_r^i \left( \mG_{r_0}^i - \bar{\mG}_{r_0}^i \right) \right\|^2 \middle| \gQ \right] \nonumber \\
    &\quad\quad + 3 \mathbb{E} \left[ \left\| \sum_{i=1}^N q_r^i \left( \mG_{r_0} - \bar{\mG}_{r_0} \right) \right\|^2 \middle| \gQ \right] \nonumber \\
    &\quad = 3 \mathbb{E} \left[ \left\| \sum_{i=1}^N q_r^i \left( \nabla F_i(\vx_{r,k}^i; \xi_{r,k}^i) - \nabla f_i(\vx_{r,k}^i) \right) \right\|^2 \middle| \gQ \right] + 3 \mathbb{E} \left[ \left\| \sum_{i=1}^N q_r^i \left( \mG_{r_0}^i - \bar{\mG}_{r_0}^i \right) \right\|^2 \middle| \gQ \right] \nonumber \\
    &\quad\quad + 3 \mathbb{E} \left[ \left\| \sum_{i=1}^N \frac{1}{N} \left( \mG_{r_0}^i - \bar{\mG}_{r_0}^i \right) \right\|^2 \middle| \gQ \right] \nonumber \\
    &\quad \Eqmark{i}{=} 3 \sum_{i=1}^N \mathbb{E} \left[ \left\| q_r^i \left( \nabla F_i(\vx_{r,k}^i; \xi_{r,k}^i) - \nabla f_i(\vx_{r,k}^i) \right) \right\|^2 \middle| \gQ \right] + 3 \sum_{i=1}^N \mathbb{E} \left[ \left\| q_r^i \left( \mG_{r_0}^i - \bar{\mG}_{r_0}^i \right) \right\|^2 \middle| \gQ \right] \nonumber \\
    &\quad\quad + 3 \sum_{i=1}^N \mathbb{E} \left[ \left\| \frac{1}{N} \left( \mG_{r_0}^i - \bar{\mG}_{r_0}^i \right) \right\|^2 \middle| \gQ \right] \nonumber \\
    &\quad = 3 \sum_{i=1}^N \left( q_r^i \right)^2 \mathbb{E} \left[ \left\| \nabla F_i(\vx_{r,k}^i; \xi_{r,k}^i) - \nabla f_i(\vx_{r,k}^i) \right\|^2 \middle| \gQ \right] + 3 \sum_{i=1}^N \left( q_r^i \right)^2 \mathbb{E} \left[ \left\| \mG_{r_0}^i - \bar{\mG}_{r_0}^i \right\|^2 \middle| \gQ \right] \nonumber \\
    &\quad\quad + 3 \frac{1}{N^2} \sum_{i=1}^N \mathbb{E} \left[ \left\| \mG_{r_0}^i - \bar{\mG}_{r_0}^i \right\|^2 \middle| \gQ \right] \nonumber \\
    &\quad \leq 6 \sigma^2 \sum_{i=1}^N \left( q_r^i \right)^2 + 3 \frac{\sigma^2}{N} \nonumber \\
    &\quad \leq \sigma^2 \left( 6 \rho^2 + 3 \frac{1}{N} \right) \nonumber \\
    &\quad \Eqmark{ii}{\leq} 9 \rho^2 \sigma^2, \label{eq:avg_update_var}
\end{align}
where $(i)$ uses the fact that stochastic gradient noise is independent across each
client, and $(ii)$ uses $\rho^2 \geq \frac{1}{N}$.

\paragraph{One-step recursive bound for $D_{r,k}$} For any $k \geq 0$,
\begin{align}
    D_{r,k+1} &= \sum_{i=1}^N q_r^i \mathbb{E} \left[ \left\| \vx_{r,k+1}^i - \sum_{j=1}^N q_r^j \vx_{r,k+1}^j \right\|^2 \middle| \gQ \right] \nonumber \\
    &= \sum_{i=1}^N q_r^i \mathbb{E} \left[ \left\| \vx_{r,k}^i - \eta \vg_{r,k}^i - \sum_{j=1}^N q_r^j (\vx_{r,k}^j - \eta \vg_{r,k}^j) \right\|^2 \middle| \gQ \right] \nonumber \\
    &= \sum_{i=1}^N q_r^i \mathbb{E} \left[ \left\| \vx_{r,k}^i - \sum_{j=1}^N q_r^j \vx_{r,k}^j - \eta \left( \vg_{r,k}^i - \sum_{j=1}^N q_r^j \vg_{r,k}^j \right) \right\|^2 \middle| \gQ \right] \nonumber \\
    &= \sum_{i=1}^N q_r^i \mathbb{E} \left[ \left\| \vx_{r,k}^i - \sum_{j=1}^N q_r^j \vx_{r,k}^j - \eta \left( \bvg_{r,k}^i - \sum_{j=1}^N q_r^j \bvg_{r,k}^j \right) \right\|^2 \middle| \gQ \right] \nonumber \\
    &\quad + \eta^2 \sum_{i=1}^N q_r^i \mathbb{E} \left[ \left\| (\vg_{r,k}^i - \bar{\vg}_{r,k}^i) + \sum_{j=1}^N q_r^j (\vg_{r,k}^j - \bar{\vg}_{r,k}^j) \right\|^2 \middle| \gQ \right] \nonumber \\
    &\Eqmark{i}{=} \sum_{i=1}^N q_r^i \mathbb{E} \left[ \left\| \vx_{r,k}^i - \sum_{j=1}^N q_r^j \vx_{r,k}^j - \eta \left( \bvg_{r,k}^i - \sum_{j=1}^N q_r^j \bvg_{r,k}^j \right) \right\|^2 \middle| \gQ \right] + 36 \eta^2 \sigma^2 \nonumber \\
    &\Eqmark{ii}{\leq} \left( 1 + \frac{1}{\lambda_1} \right) \sum_{i=1}^N q_r^i \mathbb{E} \left[ \left\| \vx_{r,k}^i - \sum_{j=1}^N q_r^j \vx_{r,k}^j \right\|^2 \middle| \gQ \right] \nonumber \\
    &\quad + \eta^2 (1 + \lambda_1) \sum_{i=1}^N q_r^i \mathbb{E} \left[ \left\| \bvg_{r,k}^i - \sum_{j=1}^N q_r^j \bvg_{r,k}^j \right\|^2 \middle| \gQ \right] + 36 \eta^2 \sigma^2 \nonumber \\
    &= \left( 1 + \frac{1}{\lambda_1} \right) D_{r,k} + \eta^2 (1 + \lambda_1) \sum_{i=1}^N q_r^i \mathbb{E} \left[ \left\| \bvg_{r,k}^i - \sum_{j=1}^N q_r^j \bvg_{r,k}^j \right\|^2 \middle| \gQ \right] + 36 \eta^2 \sigma^2 , \label{eq:D_ub_1}
\end{align}
where $(i)$ uses
\begin{align*}
    &\mathbb{E} \left[ \left\| (\vg_{r,k}^i - \bar{\vg}_{r,k}^i) + \sum_{j=1}^N q_r^j (\vg_{r,k}^j - \bar{\vg}_{r,k}^j) \right\|^2 \middle| \gQ \right] \\
    &\quad \leq 2 \mathbb{E} \left[ \left\| \vg_{r,k}^i - \bar{\vg}_{r,k}^i \right\|^2 \middle| \gQ \right] + 2 \mathbb{E} \left[ \left\| \sum_{j=1}^N q_r^j (\vg_{r,k}^j - \bar{\vg}_{r,k}^j) \right\|^2 \middle| \gQ \right] \\
    &\quad \leq 2 \mathbb{E} \left[ \left\| \vg_{r,k}^i - \bar{\vg}_{r,k}^i \right\|^2 \middle| \gQ \right] + 2 \sum_{j=1}^N q_r^j \mathbb{E} \left[ \left\| \vg_{r,k}^j - \bar{\vg}_{r,k}^j \right\|^2 \middle| \gQ \right] \\
    &\quad \Eqmark{a}{\leq} 36 \sigma^2,
\end{align*}
$(a)$ uses \Eqref{eq:update_var}, and $(ii)$ uses Young's inequality. Focusing on the
second term of \Eqref{eq:D_ub_1}:
\begin{align*}
    &\quad \sum_{i=1}^N q_r^i \mathbb{E} \left[ \left\| \bvg_{r,k}^i - \sum_{j=1}^N q_r^j \bvg_{r,k}^j \right\|^2 \middle| \gQ \right] = \sum_{i=1}^N q_r^i \mathbb{E} \left[ \left\| \sum_{j=1}^N q_r^j (\bvg_{r,k}^i - \bvg_{r,k}^j) \right\|^2 \middle| \gQ \right] \\
    &\Eqmark{i}{\leq} \sum_{i=1}^N q_r^i \sum_{j=1}^N q_r^j \mathbb{E} \left[ \left\| \bvg_{r,k}^i - \bvg_{r,k}^j \right\|^2 \middle| \gQ \right] \\
    &\leq \sum_{i=1}^N q_r^i \sum_{j=1}^N q_r^j \mathbb{E} \left[ \left\| \nabla f_i(\vx_{r,k}^i) - \bar{\mG}_{r_0}^i - \nabla f_j(\vx_{r,k}^j) + \bar{\mG}_{r_0}^j \right\|^2 \middle| \gQ \right] \\
    &\leq \sum_{i=1}^N q_r^i \sum_{j=1}^N q_r^j \bigg( 2 \mathbb{E} \left[ \left\| \nabla f_i(\vx_{r,k}^i) - \bar{\mG}_{r_0}^i \right\|^2 \middle| \gQ \right] + 2 \mathbb{E} \left[ \left\| \nabla f_j(\vx_{r,k}^j) - \bar{\mG}_{r_0}^j \right\|^2 \middle| \gQ \right] \bigg) \\
    &= 4 \sum_{i=1}^N q_r^i \mathbb{E} \left[ \left\| \nabla f_i(\vx_{r,k}^i) - \bar{\mG}_{r_0}^i \right\|^2 \middle| \gQ \right],
\end{align*}
where $(i)$ uses Jensen's inequality. Using the decomposition
\begin{equation*}
    \nabla f_i(\vx_{r,k}^i) - \bar{\mG}_{r_0}^i = (\nabla f_i(\vx_{r,k}^i) - \nabla f_i(\bvx_{r,k})) + (\nabla f_i(\bvx_{r,k}) - \nabla f_i(\bvx_{r_0})) + (\nabla f_i(\bvx_{r_0}) - \bar{\mG}_{r_0}^i),
\end{equation*}
we have
\begin{align*}
    &\sum_{i=1}^N q_r^i \mathbb{E} \left[ \left\| \bvg_{r,k}^i - \sum_{j=1}^N q_r^j \bvg_{r,k}^j \right\|^2 \middle| \gQ \right] \\
    &\leq 4 \sum_{i=1}^N q_r^i \bigg( 3 \mathbb{E} \left[ \left\| \nabla f_i(\vx_{r,k}^i) - \nabla f_i(\bvx_{r,k}) \right\|^2 \middle| \gQ \right] + 3 \mathbb{E} \left[ \left\| \nabla f_i(\bvx_{r,k}) - \nabla f_i(\bvx_{r_0}) \right\|^2 \middle| \gQ \right] \\
    &\quad + 3 \mathbb{E} \left[ \left\| \nabla f_i(\bvx_{r_0}) - \bar{\mG}_{r_0}^i \right\|^2 \middle| \gQ \right] \bigg) \\
    &\leq 12 L^2 \sum_{i=1}^N q_r^i \mathbb{E} \left[ \left\| \vx_{r,k}^i - \bvx_{r,k} \right\|^2 \middle| \gQ \right] + 12 L^2 \sum_{i=1}^N q_r^i \mathbb{E} \left[ \left\| \bvx_{r,k} - \bvx_{r_0} \right\|^2 \middle| \gQ \right] \\
    &\quad + 12 \sum_{i=1}^N q_r^i \mathbb{E} \left[ \left\| \nabla f_i(\bvx_{r_0}) - \bar{\mG}_{r_0}^i \right\|^2 \middle| \gQ \right] \\
    &= 12 L^2 D_{r,k} + 12 L^2 M_{r,k} + 12 \sum_{i=1}^N q_r^i \mathbb{E} \left[ \left\| \nabla f_i(\bvx_{r_0}) - \frac{1}{P \bar{z}_{r_0-P}^i I} \sum_{s=r_0-P}^{r_0-1} \sum_{k=0}^{I-1} z_s^i \nabla f_i(\vy_{s,k}^i) \right\|^2 \middle| \gQ \right] \\
    &= 12 L^2 D_{r,k} + 12 L^2 M_{r,k} + 12 \sum_{i=1}^N q_r^i \mathbb{E} \left[ \left\| \frac{1}{P \bar{z}_{r_0-P}^i I} \sum_{s=r_0-P}^{r_0-1} \sum_{k=0}^{I-1} z_s^i (\nabla f_i(\bvx_{r_0}) - \nabla f_i(\vy_{s,k}^i)) \right\|^2 \middle| \gQ \right] \\
    &\leq 12 L^2 D_{r,k} + 12 L^2 M_{r,k} + 12 \sum_{i=1}^N q_r^i \left( \frac{1}{P \bar{z}_{r_0-P}^i I} \sum_{s=r_0-P}^{r_0-1} \sum_{k=0}^{I-1} z_s^i \mathbb{E} \left[ \left\| \nabla f_i(\bvx_{r_0}) - \nabla f_i(\vy_{s,k}^i) \right\|^2 \middle| \gQ \right] \right) \\
    &\leq 12 L^2 D_{r,k} + 12 L^2 M_{r,k} + 12 L^2 \sum_{i=1}^N q_r^i \left( \frac{1}{P \bar{z}_{r_0-P}^i I} \sum_{s=r_0-P}^{r_0-1} \sum_{k=0}^{I-1} z_s^i \mathbb{E} \left[ \left\| \bvx_{r_0} - \vy_{s,k}^i \right\|^2 \middle| \gQ \right] \right) \\
    &\leq 12 L^2 D_{r,k} + 12 L^2 M_{r,k} + 12 L^2 \sum_{i=1}^N q_r^i S_{r_0}^i.
\end{align*}
Plugging back into \Eqref{eq:D_ub_1},
\begin{align*}
    D_{r,k+1} &\leq \left( 1 + \frac{1}{\lambda_1} \right) D_{r,k} + 12 \eta^2 L^2 (1 + \lambda_1) \left( D_{r,k} + M_{r,k} + \sum_{i=1}^N q_r^i S_{r_0}^i \right) + 36 \eta^2 \sigma^2 \\
    &\leq \left( 1 + \frac{1}{\lambda_1} + 12 \eta^2 L^2 (1 + \lambda_1) \right) D_{r,k} + 12 \eta^2 L^2 (1 + \lambda_1) \left( M_{r,k} + \sum_{i=1}^N q_r^i S_{r_0}^i \right) + 36 \eta^2 \sigma^2.
\end{align*}
Now choose $\lambda_1 = 2I$, so that $1 \leq \frac{\lambda_1}{2}$ and
\begin{equation*}
    12 \eta^2 L^2 (1 + \lambda_1) \leq 18 \eta^2 L^2 \lambda_1 \leq 36 \eta^2 L^2 I \leq \frac{1}{2I},
\end{equation*}
where we used the condition $\eta \leq \frac{1}{60 LIP}$. This yields the
following recursive bound on $D_{r,k+1}$:
\begin{equation} \label{eq:D_rec_ub}
    D_{r,k+1} \leq \left( 1 + \frac{1}{I} \right) D_{r,k} + 18 \eta^2 L^2 I M_{r,k} + 18 \eta^2 L^2 I \sum_{i=1}^N q_r^i S_{r_0}^i + 36 \eta^2 \sigma^2.
\end{equation}

\paragraph{One-step recursive bound for $M_{r,k}$} For any $k \geq 0$,
\begin{align}
    M_{r,k+1} &= \mathbb{E} \left[ \left\| \sum_{i=1}^N q_r^i \vx_{r,k+1}^i - \bvx_{r_0} \right\|^2 \middle| \gQ \right] \nonumber \\
    &= \mathbb{E} \left[ \left\| \sum_{i=1}^N q_r^i \vx_{r,k}^i - \bvx_{r_0} - \eta \sum_{i=1}^N q_r^i \vg_{r,k}^i \right\|^2 \middle| \gQ \right] \nonumber \\
    &\leq \mathbb{E} \left[ \left\| \sum_{i=1}^N q_r^i \vx_{r,k}^i - \bvx_{r_0} - \eta \sum_{i=1}^N q_r^i \bvg_{r,k}^i \right\|^2 \middle| \gQ \right] + \eta^2 \mathbb{E} \left[ \left\| \sum_{i=1}^N q_r^i (\vg_{r,k}^i - \bar{\vg}_{r,k}^i) \right\|^2 \middle| \gQ \right] \nonumber \\
    &\Eqmark{i}{\leq} \mathbb{E} \left[ \left\| \sum_{i=1}^N q_r^i \vx_{r,k}^i - \bvx_{r_0} - \eta \sum_{i=1}^N q_r^i \bvg_{r,k}^i \right\|^2 \middle| \gQ \right] + 9 \eta^2 \rho^2 \sigma^2 \nonumber \\
    &\Eqmark{ii}{\leq} \left( 1 + \frac{1}{\lambda_2} \right) \mathbb{E} \left[ \left\| \sum_{i=1}^N q_r^i \vx_{r,k}^i - \bvx_{r_0} \right\|^2 \middle| \gQ \right] + \eta^2 (1 + \lambda_2) \mathbb{E} \left[ \left\| \sum_{i=1}^N q_r^i \bvg_{r,k}^i \right\|^2 \middle| \gQ \right] + 9 \eta^2 \rho^2 \sigma^2 \nonumber \\
    &\leq \left( 1 + \frac{1}{\lambda_2} \right) M_{r,k} + \eta^2 (1 + \lambda_2) \mathbb{E} \left[ \left\| \sum_{i=1}^N q_r^i \bvg_{r,k}^i \right\|^2 \middle| \gQ \right] + 9 \eta^2 \rho^2 \sigma^2 \nonumber \\
    &\leq \left( 1 + \frac{1}{\lambda_2} \right) M_{r,k} + \eta^2 (1 + \lambda_2) \sum_{i=1}^N q_r^i \mathbb{E} \left[ \left\| \bvg_{r,k}^i \right\|^2 \middle| \gQ \right] + 9 \eta^2 \rho^2 \sigma^2, \label{eq:M_ub_1}
\end{align}
where $(i)$ uses \Eqref{eq:avg_update_var} and $(ii)$ uses Young's inequality. Focusing on the second term in \Eqref{eq:M_ub_1}:
\begin{align*}
    &\sum_{i=1}^N q_r^i \mathbb{E} \left[ \left\| \bvg_{r,k}^i \right\|^2 \middle| \gQ \right] \\
    &= \sum_{i=1}^N q_r^i \mathbb{E} \left[ \left\| \nabla f_i(\vx_{r,k}^i) - \bar{\mG}_{r_0}^i + \bar{\mG}_{r_0} \right\|^2 \middle| \gQ \right] \\
    &= \sum_{i=1}^N q_r^i \mathbb{E} \bigg[ \big\| (\nabla f_i(\vx_{r,k}^i) - \nabla f_i(\bvx_{r,k})) + (\nabla f_i(\bvx_{r,k}) - \nabla f_i(\bvx_{r_0})) \\
    &\quad\quad + (\nabla f_i(\bvx_{r_0}) - \bar{\mG}_{r_0}^i) + (\bar{\mG}_{r_0} - \nabla f(\bvx_{r_0})) + \nabla f(\bvx_{r_0}) \big\|^2 \bigg| \gQ \bigg] \\
    &\leq 5 \sum_{i=1}^N q_r^i \mathbb{E} \left[ \left\| \nabla f_i(\vx_{r,k}^i) - \nabla f_i(\bvx_{r,k}) \right\|^2 \middle| \gQ \right] + 5 \sum_{i=1}^N q_r^i \mathbb{E} \left[ \left\| \nabla f_i(\bvx_{r,k}) - \nabla f_i(\bvx_{r_0}) \right\|^2 \middle| \gQ \right] \\
    &\quad\quad + 5 \sum_{i=1}^N q_r^i \mathbb{E} \left[ \left\| \nabla f_i(\bvx_{r_0}) - \bar{\mG}_{r_0}^i \right\|^2 \middle| \gQ \right] + 5 \mathbb{E} \left[ \left\| \bar{\mG}_{r_0} - \nabla f(\bvx_{r_0}) \right\|^2 \middle| \gQ \right] + 5 \mathbb{E} \left[ \left\| \nabla f(\bvx_{r_0}) \right\|^2 \middle| \gQ \right] \\
    &\leq 5L^2 \sum_{i=1}^N q_r^i \mathbb{E} \left[ \left\| \vx_{r,k}^i - \bvx_{r,k} \right\|^2 \middle| \gQ \right] + 5L^2 \sum_{i=1}^N q_r^i \mathbb{E} \left[ \left\| \bvx_{r,k} - \bvx_{r_0} \right\|^2 \middle| \gQ \right] \\
    &\quad\quad + 5 \sum_{i=1}^N \left( q_r^i + \frac{1}{N} \right) \mathbb{E} \left[ \left\| \nabla f_i(\bvx_{r_0}) - \frac{1}{P \bar{z}_{r_0-P}^i I} \sum_{s=r_0-P}^{r_0-1} \sum_{k=0}^{I-1} z_s^i \nabla f_i(\vy_{s,k}^i) \right\|^2 \middle| \gQ \right] + 5 \mathbb{E} \left[ \left\| \nabla f(\bvx_{r_0}) \right\|^2 \middle| \gQ \right] \\
    &\leq 5L^2 D_{r,k} + 5L^2 M_{r,k} \\
    &\quad\quad + 5 \sum_{i=1}^N \left( q_r^i + \frac{1}{N} \right) \mathbb{E} \left[ \left\| \frac{1}{P \bar{z}_{r_0-P}^i I} \sum_{s=r_0-P}^{r_0-1} \sum_{k=0}^{I-1} z_s^i (\nabla f_i(\bvx_{r_0}) - \nabla f_i(\vy_{s,k}^i)) \right\|^2 \middle| \gQ \right] + 5 \mathbb{E} \left[ \left\| \nabla f(\bvx_{r_0}) \right\|^2 \middle| \gQ \right] \\
    &\leq 5L^2 D_{r,k} + 5L^2 M_{r,k} \\
    &\quad\quad + 5 \sum_{i=1}^N \left( q_r^i + \frac{1}{N} \right) \frac{1}{P \bar{z}_{r_0-P}^i I} \sum_{s=r_0-P}^{r_0-1} \sum_{k=0}^{I-1} z_s^i \mathbb{E} \left[ \left\| \nabla f_i(\bvx_{r_0}) - \nabla f_i(\vy_{s,k}^i) \right\|^2 \middle| \gQ \right] + 5 \mathbb{E} \left[ \left\| \nabla f(\bvx_{r_0}) \right\|^2 \middle| \gQ \right] \\
    &\leq 5L^2 D_{r,k} + 5L^2 M_{r,k} + 5L^2 \sum_{i=1}^N \left( q_r^i + \frac{1}{N} \right) S_{r_0}^i + 5 \mathbb{E} \left[ \left\| \nabla f(\bvx_{r_0}) \right\|^2 \middle| \gQ \right].
\end{align*}
Plugging back into \Eqref{eq:M_ub_1}:
\begin{align*}
    M_{r,k+1} &\leq \left( 1 + \frac{1}{\lambda_2} \right) M_{r,k} + 5 \eta^2 L^2 (1 + \lambda_2) \left( D_{r,k} + M_{r,k} + \sum_{i=1}^N \left( q_r^i + \frac{1}{N} \right) S_{r_0}^i \right) \\
    &\quad\quad + 5 \eta^2 (1 + \lambda_2) \mathbb{E} \left[ \left\| \nabla f(\bvx_{r_0}) \right\|^2 \middle| \gQ \right] + 9 \eta^2 \rho^2 \sigma^2 \\
    &\leq \left( 1 + \frac{1}{\lambda_2} + 5 \eta^2 L^2 (1 + \lambda_2) \right) M_{r,k} + 5 \eta^2 L^2 (1 + \lambda_2) \left( D_{r,k} + \sum_{i=1}^N \left( q_r^i + \frac{1}{N} \right) S_{r_0}^i \right) \\
    &\quad\quad + 5 \eta^2 (1 + \lambda_2) \mathbb{E} \left[ \left\| \nabla f(\bvx_{r_0}) \right\|^2 \middle| \gQ \right] + 9 \eta^2 \rho^2 \sigma^2.
\end{align*}
Now choose $\lambda_2 = 2IP$, so that $1 \leq \frac{\lambda_2}{2}$ and
\begin{equation*}
    5 \eta^2 L^2 (1 + \lambda_2) \leq \frac{15}{2} \eta^2 L^2 \lambda_2 \leq 15 \eta^2 L^2 IP \leq 15 L^2 IP \frac{1}{72 L^2 I^2 P^2} \leq \frac{1}{2 IP},
\end{equation*}
where we used the condition $\eta \leq \frac{1}{60 LIP}$. This yields the
following recursive bound on $M_{r,k+1}$:
\begin{align}
    M_{r,k+1} &\leq \left( 1 + \frac{1}{IP} \right) M_{r,k} + 15 \eta^2 L^2 IP D_{r,k} + 15 \eta^2 L^2 IP\sum_{i=1}^N \left( q_r^i + \frac{1}{N} \right) S_{r_0}^i \nonumber \\
    &\quad\quad + 15 \eta^2 IP \mathbb{E} \left[ \left\| \nabla f(\bvx_{r_0}) \right\|^2 \middle| \gQ \right] + 9 \eta^2 \rho^2 \sigma^2. \label{eq:M_rec_ub}
\end{align}
Note that the same argument can be used to bound $M_{r,0}$ in terms of $M_{r-1,I-1}$
with the same recurrence relation.

\paragraph{Unrolling the recursion} For $0 \leq t \leq IP$, let
$r_t = r_0 + \floor{t/I}$ and $k_t = t - \floor{t/I}*I$. Then let $d_t = D_{r_t,k_t}$ and
$m_t = M_{r_t,k_t}$. Also, let
\begin{align*}
    q_1 &= 18 \eta^2 L^2 I \\
    q_2 &= 15 \eta^2 L^2 IP \\
    a_r &= 18 \eta^2 L^2 I \sum_{i=1}^N q_r^i S_{r_0}^i + 36 \eta^2 \sigma^2 \\
    b_r &= 15 \eta^2 L^2 IP\sum_{i=1}^N \left( q_r^i + \frac{1}{N} \right) S_{r_0}^i + 15 \eta^2 IP \mathbb{E} \left[ \left\| \nabla f(\bvx_{r_0}) \right\|^2 \middle| \gQ \right] + 9 \eta^2 \rho^2 \sigma^2.
\end{align*}
Then \Eqref{eq:D_rec_ub} and \Eqref{eq:M_rec_ub} imply the following mutually recursive
relation over $d_t$ and $m_t$:
\begin{align}
    d_t &= \begin{cases}
        0 & k_t = 0 \\
        \left( 1 + \frac{1}{I} \right) d_{t-1} + q_1 m_{t-1} + a_{r_{t-1}} & \text{otherwise}
    \end{cases} \nonumber \\
    m_t &= \left( 1 + \frac{1}{IP} \right) m_{t-1} + q_2 d_{t-1} + b_{r_{t-1}}, \label{eq:dual_recursion}
\end{align}
where $d_0 = m_0 = 0$. We will unroll this recurrence, showing a bound for $m_t$ and $d_t$ in terms of the following quantities. For any $0 \leq t \leq IP - 1$ and $0 \leq s \leq r_t$, let
$j(s,t) = \min\{I, t-sI\}$ and $\ell(s,t) = \max\{0, t - (s+1)I\}$. For any $s, t$, if $k_{t+1} = 0$, define \[
    \alpha_{s,t+1} = 0.
\] Otherwise, define
\begin{align*}
    \alpha_{s,t+1} &= \alpha_{s,t} \left( 1 + \frac{1}{I} \right) + q_1 P \left( \left( 1 + \frac{1}{IP} \right)^{j(s,t)} - 1 \right) \left( 1 + \frac{1}{IP} \right)^{\ell(s,t)} \\
    &\quad + q_1 q_2 \sum_{i=0}^{t-sI-2} \left( 1 + \frac{1}{IP} \right)^i \alpha_{s,t-1-i}.
\end{align*}
Also, define
\begin{align*}
    \beta_{s,t+1} &= \beta_{s,t} \left( 1 + \frac{1}{IP} \right) + \indicator{s=r_t} \frac{q_2}{P} \left( \left( 1 + \frac{1}{I} \right)^{k_t} - 1 \right) + q_1 q_2 \sum_{i=0}^{k_t-2} \left( 1 + \frac{1}{I} \right)^i \beta_{s,t-1-i} \\
    \psi_{s,t+1} &= \frac{1}{P} \sum_{i=0}^{t-sI-1} \left( 1 + \frac{1}{IP} \right)^i \alpha_{s,t-i} \\
    \phi_{s,t+1} &= P \sum_{i=0}^{k_t-1} \left( 1 + \frac{1}{I} \right)^i \beta_{s,t-i},
\end{align*}
under the initial conditions
\begin{align*}
    \phi_{s,kI} &= 0 \text{ for all } s \leq P \text{ and } s \leq k \leq P \\
    \alpha_{s,kI} &= 0 \text{ for all } s \leq P \text{ and } s \leq k \leq P \\
    \psi_{s,sI} &= 0 \text{ for all } s \leq P \\
    \beta_{s,sI} &= 0 \text{ for all } s \leq P.
\end{align*}
Now we can unroll the recurrence in \Eqref{eq:dual_recursion}, by proving the following statements by induction on $t$:
\begin{align}
    d_t &\leq \left( \left( 1 + \frac{1}{I} \right)^{k_t} - 1 \right) a_{r_t} I + q_1 I \sum_{s=0}^{r_t} \phi_{s,t} a_s + I \sum_{s=0}^{r_t} \alpha_{s,t} b_s \label{eq:d_rec_inter} \\
    m_t &\leq IP \sum_{s=0}^{r_t} \left( \left( \left( 1 + \frac{1}{IP} \right)^{j(s,t)} - 1 \right) \left( 1 + \frac{1}{IP} \right)^{\ell(s,t)} + q_2 \psi_{s,t} \right) b_s + IP \sum_{s=0}^{r_t} \beta_{s,t} a_s. \label{eq:m_rec_inter}
\end{align}
\Eqref{eq:d_rec_inter} and \Eqref{eq:m_rec_inter} hold for the base case $t=0$, since
$d_0 = m_0 = 0$. Now suppose that \Eqref{eq:d_rec_inter} and \Eqref{eq:m_rec_inter} hold
for some $t \leq IP$, and we will show that they hold for $t+1$. We consider two cases:
$k_{t+1} \neq 0$ and $k_{t+1} = 0$.

In the first case, $r_{t+1} = r_t$, i.e., step $t+1$ is in the same round as step $t$,
and $k_{t+1} = k_t+1$. Using \Eqref{eq:dual_recursion} together with the inductive
hypothesis:
\begin{align*}
    d_{t+1} &\leq \left( 1 + \frac{1}{I} \right) \left( \left( \left( 1 + \frac{1}{I} \right)^{k_t} - 1 \right) a_{r_t} I + q_1 I \sum_{s=0}^{r_t} \phi_{s,t} a_s + I \sum_{s=0}^{r_t} \alpha_{s,t} b_s \right) \\
    &\quad + q_1 \left( IP \sum_{s=0}^{r_t} \left( \left( \left( 1 + \frac{1}{IP} \right)^{j(s,t)} - 1 \right) \left( 1 + \frac{1}{IP} \right)^{\ell(s,t)} + q_2 \psi_{s,t} \right) b_s + IP \sum_{s=0}^{r_t} \beta_{s,t} a_s \right) + a_{r_t} \\
    &\leq \left( \left( 1 + \frac{1}{I} \right)^{k_t+1} - \left( 1 + \frac{1}{I} \right) + \frac{1}{I} \right) a_{r_t} I + q_1 I \sum_{s=0}^{r_t} \left( \left( 1 + \frac{1}{I} \right) \phi_{s,t} + P \beta_{s,t} \right) a_s \\
    &\quad + I \sum_{s=0}^{r_t} \left( \left( 1 + \frac{1}{I} \right) \alpha_{s,t} + q_1 P \left( \left( 1 + \frac{1}{IP} \right)^{j(s,t)} - 1 \right) \left( 1 + \frac{1}{IP} \right)^{\ell(s,t)} + q_1 q_2 P \psi_{s,t} \right) b_s \\
    &\Eqmark{i}{\leq} \left( \left( 1 + \frac{1}{I} \right)^{k_t+1} - 1 \right) a_{r_t} I + q_1 I \sum_{s=0}^{r_t} \phi_{s,t+1} a_s + I \sum_{s=0}^{r_t} \alpha_{s,t+1} b_s \\
    &\Eqmark{ii}{\leq} \left( \left( 1 + \frac{1}{I} \right)^{k_{t+1}} - 1 \right) a_{r_{t+1}} I + q_1 I \sum_{s=0}^{r_{t+1}} \phi_{s,t+1} a_s + I \sum_{s=0}^{r_{t+1}} \alpha_{s,t+1} b_s,
\end{align*}
where $(i)$ uses the fact that $\phi_{s,t+1} = \left( 1 + \frac{1}{I} \right) \phi_{s,t}
+ P \beta_{s,t}$ and the definitions of $\alpha_{s,t+1}$, $\psi_{s,t}$, and $(ii)$ uses
$k_{t+1} = k_t+1$ and $r_{t+1} = r_t$. Similarly,
\begin{align*}
    m_{t+1} &\leq \left( 1 + \frac{1}{IP} \right) \left( IP \sum_{s=0}^{r_t} \left( \left( \left( 1 + \frac{1}{IP} \right)^{j(s,t)} - 1 \right) \left( 1 + \frac{1}{IP} \right)^{\ell(s,t)} + q_2 \psi_{s,t} \right) b_s + IP \sum_{s=0}^{r_t} \beta_{s,t} a_s \right) \\
    &\quad + q_2 \left( \left( \left( 1 + \frac{1}{I} \right)^{k_t} - 1 \right) a_{r_t} I + q_1 I \sum_{s=0}^{r_t} \phi_{s,t} a_s + I \sum_{s=0}^{r_t} \alpha_{s,t} b_s \right) + b_{r_t} \\
    &\leq IP \sum_{s=0}^{r_t} \bigg( \left( \left( 1 + \frac{1}{IP} \right)^{j(s,t)} - 1 \right) \left( 1 + \frac{1}{IP} \right)^{\ell(s,t)+1} \\
    &\quad + \indicator{s=r_t} \frac{1}{IP} + q_2 \left( \left( 1 + \frac{1}{IP} \right) \psi_{s,t} + \frac{1}{P} \alpha_{s,t} \right) \bigg) b_s \\
    &\quad + IP \sum_{s=0}^{r_t} \left( \left( 1 + \frac{1}{IP} \right) \beta_{s,t} + \indicator{s=r_t} \frac{q_2}{P} \left( \left( 1 + \frac{1}{I} \right)^{k_t} - 1 \right) + \frac{q_1 q_2}{P} \phi_{s,t} \right) a_s \\
    &\Eqmark{i}{\leq} IP \sum_{s=0}^{r_t} \left( \left( \left( 1 + \frac{1}{IP} \right)^{j(s,t+1)} - 1 \right) \left( 1 + \frac{1}{IP} \right)^{\ell(s,t+1)} + q_2 \left( \left( 1 + \frac{1}{IP} \right) \psi_{s,t} + \frac{1}{P} \alpha_{s,t} \right) \right) b_s \\
    &\quad + IP \sum_{s=0}^{r_t} \left( \left( 1 + \frac{1}{IP} \right) \beta_{s,t} + \indicator{s=r_t} \frac{q_2}{P} \left( \left( 1 + \frac{1}{I} \right)^{k_t} - 1 \right) + \frac{q_1 q_2}{P} \phi_{s,t} \right) a_s \\
    &\Eqmark{ii}{\leq} IP \sum_{s=0}^{r_t} \left( \left( \left( 1 + \frac{1}{IP} \right)^{j(s,t+1)} - 1 \right) \left( 1 + \frac{1}{IP} \right)^{\ell(s,t+1)} + q_2 \psi_{s,t+1} \right) b_s + IP \sum_{s=0}^{r_t} \beta_{s,t+1} a_s \\
    &\Eqmark{iii}{\leq} IP \sum_{s=0}^{r_{t+1}} \left( \left( \left( 1 + \frac{1}{IP} \right)^{j(s,t+1)} - 1 \right) \left( 1 + \frac{1}{IP} \right)^{\ell(s,t+1)} + q_2 \psi_{s,t+1} \right) b_s + IP \sum_{s=0}^{r_{t+1}} \beta_{s,t+1} a_s,
\end{align*}
where $(i)$ uses the fact that for $s < r_t$:
\begin{align*}
    &\left( \left( 1 + \frac{1}{IP} \right)^{j(s,t)} - 1 \right) \left( 1 + \frac{1}{IP} \right)^{\ell(s,t)+1} + \indicator{s=r_t} \frac{1}{IP} \\
    &\quad = \left( \left( 1 + \frac{1}{IP} \right)^I - 1 \right) \left( 1 + \frac{1}{IP} \right)^{t+1-(s+1)I} \\
    &\quad = \left( \left( 1 + \frac{1}{IP} \right)^{j(s,t+1)} - 1 \right) \left( 1 + \frac{1}{IP} \right)^{\ell(s,t+1)}
\end{align*}
and for $s = r_t$:
\begin{align*}
    &\left( \left( 1 + \frac{1}{IP} \right)^{j(s,t)} - 1 \right) \left( 1 + \frac{1}{IP} \right)^{\ell(s,t)+1} + \indicator{s=r_t} \frac{1}{IP} \\
    &\quad = \left( \left( 1 + \frac{1}{IP} \right)^{t-sI} - 1 \right) \left( 1 + \frac{1}{IP} \right) + \frac{1}{IP} \\
    &\quad = \left( \left( 1 + \frac{1}{IP} \right)^{t+1-sI} - 1 \right) \\
    &\quad = \left( \left( 1 + \frac{1}{IP} \right)^{j(s,t+1)} - 1 \right) \left( 1 + \frac{1}{IP} \right)^{\ell(s,t+1)},
\end{align*}
$(ii)$ uses the fact that $\psi_{s,t+1} = \left( 1 + \frac{1}{IP} \right) \psi_{s,t} +
\frac{1}{P} \alpha_{s,t}$ and the definitions of $\beta_{s,t+1}$, $\phi_{s,t}$, and
$(iii)$ uses $k_{t+1} = k_t+1$ and $r_{t+1} = r_t$. This completes the inductive step
for the first case ($k_{t+1} \neq 0$).

In the second case (i.e., $k_{t+1} = 0$), \Eqref{eq:d_rec_inter} must hold for $t+1$,
since $d_{t+1} = 0$. So it only remains to show \Eqref{eq:m_rec_inter} holds for $t+1$.
Note that $r_t = r_{t+1}-1$. As in the first case, we can use \Eqref{eq:dual_recursion}
together with the inductive hypothesis to obtain:
\begin{align*}
    m_{t+1} &\leq IP \sum_{s=0}^{r_t} \left( \left( \left( 1 + \frac{1}{IP} \right)^{j(s,t+1)} - 1 \right) \left( 1 + \frac{1}{IP} \right)^{\ell(s,t+1)} + q_2 \psi_{s,t+1} \right) b_s + IP \sum_{s=0}^{r_t} \phi_{s,t+1} a_s \\
    &\leq IP \sum_{s=0}^{r_{t+1}} \left( \left( \left( 1 + \frac{1}{IP} \right)^{j(s,t+1)} - 1 \right) \left( 1 + \frac{1}{IP} \right)^{\ell(s,t+1)} + q_2 \psi_{s,t+1} \right) b_s + IP \sum_{s=0}^{r_{t+1}} \phi_{s,t+1} a_s,
\end{align*}
where the second line uses the fact that the $r_{t+1}$-st element of both sums is $0$,
since $j(r_{t+1}, t+1) = 0$, $\psi_{r_{t+1},t+1} = 0$, and $\phi_{r_{t+1}, t+1} = 0$.
This completes the inductive step for both cases, and proves \Eqref{eq:d_rec_inter} and
\Eqref{eq:m_rec_inter}.

We can now bound $\alpha_{s,t}$ and $\beta_{s,t}$ separately by induction on $t$. First,
for any $s \leq P-1$ and $t$ with $sI \leq t \leq IP$, we claim that
\begin{equation} \label{eq:alpha_rec_ub}
    \alpha_{s,t} \leq 36 q_1 I P,
\end{equation}
which we will show by induction on $t$. Let $s \leq P-1$ be given. For the base case,
\Eqref{eq:alpha_rec_ub} holds when $t = sI$ since $\alpha_{s,sI} = 0$. Now suppose that
it holds for all $t' \leq t$. Then
\begin{align}
    &\left( \left( 1 + \frac{1}{IP} \right)^{m_1} - 1 \right) \left( 1 + \frac{1}{IP} \right)^{m_2} \nonumber \\
    &\quad \leq \left( \left( 1 + \frac{1}{IP} \right)^{t-sI} - 1 \right) \left( 1 + \frac{1}{IP} \right)^{0} \nonumber \\
    &\quad \leq \left( 1 + \frac{1}{IP} \right)^{t-sI} - 1, \label{eq:alpha_ub_inter_1}
\end{align}
and
\begin{align}
    q_1 q_2 \sum_{i=0}^{t-sI-2} \left( 1 + \frac{1}{IP} \right)^i \alpha_{s,t-1-i} &\leq 36 q_1^2 q_2 IP \sum_{i=0}^{t-sI-2} \left( 1 + \frac{1}{IP} \right)^i \nonumber \\
    &= 36 q_1^2 q_2 I^2 P^2 \left( \left( 1 + \frac{1}{IP} \right)^{t-sI-1} - 1 \right) \nonumber \\
    &\leq q_1 P \left( \left( 1 + \frac{1}{IP} \right)^{t-sI} - 1 \right), \label{eq:alpha_ub_inter_2}
\end{align}
where the last line uses the definition of $q_1$ and $q_2$ together with the condition $\eta \leq \frac{1}{60 LIP}$:
\begin{equation*}
    36 q_1 q_2 I^2 P = 36 \cdot 270 \eta^4 L^4 I^4 P^2 \leq \frac{36 \cdot 270}{60^4} L^4 I^4 P^2 \frac{1}{L^4 I^4 P^4} \leq \frac{1}{P^2} \leq 1.
\end{equation*}
Plugging \Eqref{eq:alpha_ub_inter_1} and \Eqref{eq:alpha_ub_inter_2} into the definition
of $\alpha_{s,t+1}$ yields
\begin{align*}
    \alpha_{s,t+1} &\leq \alpha_{s,t} \left( 1 + \frac{1}{I} \right) + 2 q_1 P \left( \left( 1 + \frac{1}{IP} \right)^{t-sI} - 1 \right) \\
    &\Eqmark{i}{\leq} 2 q_1 P \sum_{i=0}^{k_t} \left( 1 + \frac{1}{I} \right)^i \left( \left( 1 + \frac{1}{IP} \right)^{t-sI-i} - 1 \right) \\
    &\leq 2 q_1 P \left( 1 + \frac{1}{I} \right)^I \sum_{i=0}^{k_t} \left( \left( 1 + \frac{1}{IP} \right)^{t-sI-i} - 1 \right) \\
    &\leq 6 q_1 P \sum_{i=0}^{k_t} \left( 1 + \frac{1}{IP} \right)^{t-sI-i} \leq 6 q_1 P \left( 1 + \frac{1}{IP} \right)^{(r_t-s)I} \sum_{i=0}^{k_t} \left( 1 + \frac{1}{IP} \right)^i \\
    &\leq 18 q_1 P \sum_{i=0}^{I-1} \left( 1 + \frac{1}{IP} \right)^i \Eqmark{ii}{\leq} 18 q_1 \sum_{i=0}^{IP-1} \left( 1 + \frac{1}{IP} \right)^i \\
    &\leq 18 q_1 IP \left( \left( 1 + \frac{1}{IP} \right)^{IP} - 1 \right) \leq 36 q_1 IP,
\end{align*}
where $(i)$ unrolls the recurrence on the first line until $\alpha_{s,r_tI} = 0$, $(ii)$
uses $P \sum_{i=0}^{I-1} \left( 1 + \frac{1}{IP} \right)^i \leq \sum_{i=0}^{IP-1} \left(
1 + \frac{1}{IP} \right)^i$ since $\left( 1 + \frac{1}{IP} \right)^i$ is increasing with
$i$, and we repeatedly used $\left( 1 + \frac{1}{x} \right)^x < e \leq 3$. This
completes the induction and proves \Eqref{eq:alpha_rec_ub}.

We will similarly prove the statement
\begin{equation} \label{eq:beta_rec_ub}
    \beta_{s,t} \leq 12 \frac{q_2 I}{P}.
\end{equation}
for all $s \leq P-1$ and $t$ with $sI \leq t \leq IP$ by induction on $t$. Let $s \leq
P-1$ be given. For the base case, \Eqref{eq:alpha_rec_ub} holds when $t=sI$ since
$\beta_{s,sI} = 0$. Now suppose that it holds for all $t' \leq t$. Then
\begin{align}
    q_1 q_2 \sum_{i=0}^{k_t-2} \left( 1 + \frac{1}{I} \right)^i \beta_{s,t-1-i} &\leq 12 \frac{q_1 q_2^2 I}{P} \sum_{i=0}^{k_t-2} \left( 1 + \frac{1}{I} \right)^i \nonumber \\
    &= 12 \frac{q_1 q_2^2 I^2}{P} \left( \left( 1 + \frac{1}{I} \right)^{k_t-1} - 1 \right) \nonumber \\
    &\leq 12 \frac{q_1 q_2^2 I^2}{P} \left( \left( 1 + \frac{1}{I} \right)^{k_t} - 1 \right) \label{eq:beta_ub_inter_1} \\
    &\leq \frac{q_2}{P} \left( \left( 1 + \frac{1}{I} \right)^{k_t} - 1 \right), \label{eq:beta_ub_inter_2}
\end{align}
where the last line uses the definition of $q_1$ and $q_2$ together with the condition
$\eta \leq \frac{1}{60 LIP}$:
\begin{equation*}
    12 q_1 q_2 I^2 \leq 12 \cdot 270 \eta^4 L^4 I^4 P \leq \frac{12 \cdot 270}{60^4} L^4 I^4 P \frac{1}{L^4 I^4 P^4} \leq \frac{1}{P^3} \leq 1.
\end{equation*}
We now consider two cases: $t < (s+1) I$ and $t \geq (s+1)I$. In the first case, we have
$s = r_t$, and plugging \Eqref{eq:beta_ub_inter_2} into the definition of
$\beta_{s,t+1}$ yields
\begin{align*}
    \beta_{s,t+1} &\leq \beta_{s,t} \left( 1 + \frac{1}{IP} \right) + 2 \frac{q_2}{P} \left( \left( 1 + \frac{1}{I} \right)^{k_t} - 1 \right) \\
    &\Eqmark{i}{\leq} 2 \frac{q_2}{P} \sum_{i=0}^{k_t} \left( 1 + \frac{1}{IP} \right)^i \left( \left( 1 + \frac{1}{I} \right)^{k_t-i} - 1 \right) \\
    &\leq 2 \frac{q_2}{P} \left( 1 + \frac{1}{IP} \right)^{k_t} \sum_{i=0}^{k_t} \left( \left( 1 + \frac{1}{I} \right)^{k_t-i} - 1 \right) \\
    &\leq 2 \frac{q_2}{P} \left( 1 + \frac{1}{IP} \right)^{I} \sum_{i=0}^{k_t} \left( 1 + \frac{1}{I} \right)^i \\
    &\leq 6 \frac{q_2I}{P} \left( \left( 1 + \frac{1}{I} \right)^{k_t+1} - 1 \right) \leq 6 \frac{q_2I}{P} \left( \left( 1 + \frac{1}{I} \right)^I - 1 \right) \\
    &\leq 12 \frac{q_2 I}{P},
\end{align*}
where $(i)$ unrolls the recurrence on the first line until $\beta_{s,sI} = 0$ and we
repeatedly used $\left( 1 + \frac{1}{x} \right)^x < e \leq 3$. In the second case, we
have $s > r_t$, so plugging \Eqref{eq:beta_ub_inter_1} into the definition of
$\beta_{s,t+1}$ yields
\begin{align*}
    \beta_{s,t+1} &\leq \beta_{s,t} \left( 1 + \frac{1}{IP} \right) + 12 \frac{q_1 q_2^2 I^2}{P} \left( \left( 1 + \frac{1}{I} \right)^{k_t} - 1 \right) \\
    &\leq \beta_{s,t} \left( 1 + \frac{1}{IP} \right) + 12 \frac{q_1 q_2^2 I^2}{P} \left( \left( 1 + \frac{1}{I} \right)^{I} - 1 \right) \\
    &\leq \beta_{s,t} \left( 1 + \frac{1}{IP} \right) + 24 \frac{q_1 q_2^2 I^2}{P} \\
    &\Eqmark{i}{\leq} 24 \frac{q_1 q_2^2 I^2}{P} \sum_{i=0}^{t-sI} \left( 1 + \frac{1}{IP} \right)^i \\
    &= 24 q_1 q_2^2 I^3 \left( \left( 1 + \frac{1}{IP} \right)^{t-sI+1} - 1 \right) \leq 24 q_1 q_2^2 I^3 \left( \left( 1 + \frac{1}{IP} \right)^{IP} - 1 \right) \\
    &\leq 48 q_1 q_2^2 I^3 \leq 12 \frac{q_2 I}{P},
\end{align*}
where $(i)$ unrolls the recurrence on the previous line until $\beta_{s,sI} = 0$ and the
last inequality uses the definition of $q_1$ and $q_2$ together with the condition $\eta
\leq \frac{1}{60 LIP}$:
\begin{equation*}
    48 q_1 q_2 I^2 = 48 \cdot 270 \eta^4 L^4 I^4 P \leq \frac{48 \cdot 270}{60^4} L^4 I^4 P \frac{1}{L^4 I^4 P^4} \leq \frac{12}{P^3} \leq \frac{12}{P}.
\end{equation*}
This completes the induction in both cases and proves \Eqref{eq:beta_rec_ub}.

We can then use \Eqref{eq:alpha_rec_ub} and \Eqref{eq:beta_rec_ub} to yield bounds for
the remaining terms of \Eqref{eq:d_rec_inter} and \Eqref{eq:m_rec_inter} as follows:
\begin{align*}
    q_2 \psi_{s,t} &= \frac{q_2}{P} \sum_{i=0}^{t-sI-1} \left( 1 + \frac{1}{IP} \right)^i \alpha_{s,t-1-i} \leq 36 q_1 q_2 I \sum_{i=0}^{t-sI-1} \left( 1 + \frac{1}{IP} \right)^i \\
    &\leq 36 q_1 q_2 I \sum_{i=0}^{IP-1} \left( 1 + \frac{1}{IP} \right)^i \leq 36 q_1 q_2 I^2 P \left( \left( 1 + \frac{1}{IP} \right)^{IP} - 1 \right) \\
    &\leq 72 q_1 q_2 I^2 P \leq 72 \cdot 270 \eta^4 L^4 I^4 P^2 \leq \frac{72 \cdot 270}{60^4} L^4 I^4 P^2 \frac{1}{L^4 I^4 P^4} \leq \frac{1}{P^2},
\end{align*}
and
\begin{align*}
    q_1 \phi_{s,t} &= q_1 P \sum_{i=0}^{k_t - 1} \left( 1 + \frac{1}{I} \right)^i \beta_{s,t-i} \leq 12 q_1 q_2 I \sum_{i=0}^{k_t - 1} \left( 1 + \frac{1}{I} \right)^i \\
    &= 12 q_1 q_2 I^2 \left( \left( 1 + \frac{1}{I} \right)^{k_t} - 1 \right) \leq 12 q_1 q_2 I^2 \left( \left( 1 + \frac{1}{I} \right)^I - 1 \right) \leq 24 q_1 q_2 I^2.
\end{align*}
Finally, we can plug these into \Eqref{eq:d_rec_inter} to yield
\begin{align*}
    d_t &\leq \left( \left( 1 + \frac{1}{I} \right)^{k_t} - 1 \right) a_{r_t} I + 24 q_1 q_2 I^3 \sum_{s=0}^{r_t} a_s + 36 q_1 I^2 P \sum_{s=0}^{r_t} b_s \\
    &\leq 2 a_{r_t} I + 24 q_1 q_2 I^3 \sum_{s=0}^{r_t} a_s + 36 q_1 I^2 P \sum_{s=0}^{r_t} b_s,
\end{align*}
and into \Eqref{eq:m_rec_inter}
\begin{align*}
    m_t &\leq IP \sum_{s=0}^{r_t} \left( \left( \left( 1 + \frac{1}{IP} \right)^{m_1} - 1 \right) \left( 1 + \frac{1}{IP} \right)^{m_2} + \frac{1}{P^2} \right) b_s + 12 q_2 I^2 \sum_{s=0}^{r_t} a_s \\
    &\Eqmark{i}{\leq} IP \sum_{s=0}^{r_t} \left( \frac{6}{P} + \frac{1}{P^2} \right) b_s + 12 q_2 I^2 \sum_{s=0}^{r_t} a_s \\
    &\leq 7 I \sum_{s=0}^{r_t} b_s + 12 q_2 I^2 \sum_{s=0}^{r_t} a_s,
\end{align*}
where $(i)$ uses
\begin{align*}
    \left( 1 + \frac{1}{IP} \right)^{m_1} - 1 &\leq \left( 1 + \frac{1}{IP} \right)^I - 1 = IP \sum_{i=0}^{I-1} \left( 1 + \frac{1}{IP} \right)^i \\
    &\quad \leq I \sum_{i=0}^{IP-1} \left( 1 + \frac{1}{IP} \right)^i = \frac{1}{P} \left( \left( 1 + \frac{1}{IP} \right)^{IP} - 1 \right) \leq \frac{2}{P},
\end{align*}
and
\begin{equation*}
    \left( 1 + \frac{1}{IP} \right)^{m_2} \leq \left( 1 + \frac{1}{IP} \right)^{IP} \leq 3.
\end{equation*}

Or, in the original notation:
\begin{align}
    D_{r,k} &\leq 2I \left( 18 \eta^2 L^2 I \sum_{i=1}^N q_{r_t}^i S_{r_0}^i + 36 \eta^2 \sigma^2 \right) + 6480 \eta^4 L^4 I^5 P \sum_{s=r_0}^r \left( 18 \eta^2 L^2 I \sum_{i=1}^N q_s^i S_{r_0}^i + 36 \eta^2 \sigma^2 \right) \nonumber \\
    &\quad + 648 \eta^2 L^2 I^3 P \sum_{s=r_0}^r \left( 15 \eta^2 L^2 IP \sum_{i=1}^N \left( q_s^i + \frac{1}{N} \right) S_{r_0}^i + 15 \eta^2 IP \mathbb{E} \left[ \left\| \nabla f(\bvx_{r_0}) \right\|^2 \middle| \gQ \right] + 9 \eta^2 \rho^2 \sigma^2 \right) \nonumber \\
    &\leq 36 \eta^2 L^2 I^2 \sum_{i=1}^N q_r^i S_{r_0}^i + 18 \cdot 6480 \eta^6 L^6 I^6 P \sum_{s=r_0}^r \sum_{i=1}^N q_s^i S_{r_0}^i \nonumber \\
    &\quad + 9720 \eta^4 L^4 I^4 P^2 \sum_{s=r_0}^r \sum_{i=1}^N \left( q_s^i + \frac{1}{N} \right) S_{r_0}^i + 9720 (r-r_0) \eta^4 L^2 I^4 P^3 \mathbb{E} \left[ \left\| \nabla f(\bvx_{r_0}) \right\|^2 \middle| \gQ \right] \nonumber \\
    &\quad + \left( 72 \eta^2 I + 36 \cdot 6480 (r-r_0) \eta^6 L^4 I^5 P + 5832 (r-r_0) \eta^4 L^2 I^3 P \rho^2 \right) \sigma^2 \nonumber \\
    &\Eqmark{i}{\leq} 36 \eta^2 L^2 I^2 \sum_{i=1}^N q_r^i S_{r_0}^i + 18 \cdot 6480 \eta^6 L^6 I^6 P \sum_{s=r_0}^{r_0+P-1} \sum_{i=1}^N q_s^i S_{r_0}^i \nonumber \\
    &\quad + 9720 \eta^4 L^4 I^4 P^2 \sum_{s=r_0}^{r_0+P-1} \sum_{i=1}^N \left( q_s^i + \frac{1}{N} \right) S_{r_0}^i \nonumber \\
    &\quad + \left( 72 \eta^2 I + 36 \cdot 6480 \eta^6 L^4 I^5 P^2 + 5832 \eta^4 L^2 I^3 P^2 \rho^2 \right) \sigma^2 + 9720 \eta^4 L^2 I^4 P^3 \mathbb{E} \left[ \left\| \nabla f(\bvx_{r_0}) \right\|^2 \middle| \gQ \right] \nonumber \\
    &= 36 \eta^2 L^2 I^2 \sum_{i=1}^N q_r^i S_{r_0}^i + 9720 \eta^4 L^4 I^4 P^3 \frac{1}{N} \sum_{i=1}^N S_{r_0}^i \nonumber \\
    &\quad + \left( 18 \cdot 6480 \eta^6 L^6 I^6 P^2 + 9720 \eta^4 L^4 I^4 P^3 \right) \sum_{i=1}^N \bar{q}_{r_0}^i S_{r_0}^i \nonumber \\
    &\quad + \left( 72 \eta^2 I + 36 \cdot 6480 \eta^6 L^4 I^5 P^2 + 5832 \eta^4 L^2 I^3 P^2 \rho^2 \right) \sigma^2 + 9720 \eta^4 L^2 I^4 P^3 \mathbb{E} \left[ \left\| \nabla f(\bvx_{r_0}) \right\|^2 \middle| \gQ \right] \nonumber \\
    &\Eqmark{ii}{\leq} 36 \eta^2 L^2 I^2 \sum_{i=1}^N q_r^i S_{r_0}^i + 9720 \eta^4 L^4 I^4 P^3 \frac{1}{N} \sum_{i=1}^N S_{r_0}^i + 9753 \eta^4 L^4 I^4 P^3 \sum_{i=1}^N \bar{q}_{r_0}^i S_{r_0}^i \nonumber \\
    &\quad + \left( 73 \eta^2 I + 5832 \eta^4 L^2 I^3 P^2 \rho^2 \right) \sigma^2 + 9720 \eta^4 L^2 I^4 P^3 \mathbb{E} \left[ \left\| \nabla f(\bvx_{r_0}) \right\|^2 \middle| \gQ \right], \label{eq:d_final_ub}
\end{align}
where $(i)$ uses the fact that the interval $\{r_0, \ldots, r\}$ is contained in the
interval $\{r_0, \ldots, r_0 + P - 1\}$ and $(ii)$ uses the condition $\eta \leq
\frac{1}{60 LIP}$ to simplify non-dominating terms. Also
\begin{align}
    M_{r,k} &\leq 7I \sum_{r=r_0}^r \left( 15 \eta^2 L^2 IP\sum_{i=1}^N \left( q_s^i + \frac{1}{N} \right) S_{r_0}^i + 15 \eta^2 IP \mathbb{E} \left[ \left\| \nabla f(\bvx_{r_0}) \right\|^2 \middle| \gQ \right] + 9 \eta^2 \rho^2 \sigma^2 \right) \nonumber \\
    &\quad + 180 \eta^2 L^2 I^3 P \sum_{r=r_0}^r \left( 18 \eta^2 L^2 I \sum_{i=1}^N q_s^i S_{r_0}^i + 36 \eta^2 \sigma^2 \right) \nonumber \\
    &\leq 105 \eta^2 L^2 I^2 P \sum_{r=r_0}^r \sum_{i=1}^N \left( q_s^i + \frac{1}{N} \right) S_{r_0}^i + 3240 \eta^4 L^4 I^4 P \sum_{r=r_0}^r \sum_{i=1}^N q_s^i S_{r_0}^i \nonumber \\
    &\quad + \left( 63 (r-r_0) \eta^2 I \rho^2 + 6480 (r-r_0) \eta^4 L^2 I^3 P \right) \sigma^2 + 105 (r-r_0) \eta^2 I^2 P \mathbb{E} \left[ \left\| \nabla f(\bvx_{r_0}) \right\|^2 \middle| \gQ \right] \nonumber \\
    &\Eqmark{i}{\leq} 105 \eta^2 L^2 I^2 P \sum_{r=r_0}^{r_0+P-1} \sum_{i=1}^N \left( q_s^i + \frac{1}{N} \right) S_{r_0}^i + 3240 \eta^4 L^4 I^4 P \sum_{r=r_0}^{r_0+P-1} \sum_{i=1}^N q_s^i S_{r_0}^i \nonumber \\
    &\quad + \left( 63 \eta^2 IP \rho^2 + 6480 \eta^4 L^2 I^3 P^2 \right) \sigma^2 + 105 \eta^2 I^2 P^2 \mathbb{E} \left[ \left\| \nabla f(\bvx_{r_0}) \right\|^2 \middle| \gQ \right] \nonumber \\
    &\leq 105 \eta^2 L^2 I^2 P^2 \frac{1}{N} \sum_{i=1}^N S_{r_0}^i + \left( 105 \eta^2 L^2 I^2 P^2 + 3240 \eta^4 L^4 I^4 P^2 \right) \sum_{i=1}^N \bar{q}_{r_0}^i S_{r_0}^i \nonumber \\
    &\quad + \left( 63 \eta^2 IP \rho^2 + 6480 \eta^4 L^2 I^3 P^2 \right) \sigma^2 + 105 \eta^2 I^2 P^2 \mathbb{E} \left[ \left\| \nabla f(\bvx_{r_0}) \right\|^2 \middle| \gQ \right] \nonumber \\
    &\Eqmark{ii}{\leq} 105 \eta^2 L^2 I^2 P^2 \frac{1}{N} \sum_{i=1}^N S_{r_0}^i + 106 \eta^2 L^2 I^2 P^2 \sum_{i=1}^N \bar{q}_{r_0}^i S_{r_0}^i \nonumber \\
    &\quad + \left( 63 \eta^2 IP \rho^2 + 6480 \eta^4 L^2 I^3 P^2 \right) \sigma^2 + 105 \eta^2 I^2 P^2 \mathbb{E} \left[ \left\| \nabla f(\bvx_{r_0}) \right\|^2 \middle| \gQ \right], \label{eq:m_final_ub}
\end{align}
where $(i)$ and $(ii)$ use the same operations as in \Eqref{eq:d_final_ub}. Summing
\Eqref{eq:m_final_ub} and \Eqref{eq:d_final_ub},
\begin{align*}
    D_{r,k} + M_{r,k} &\leq 36 \eta^2 L^2 I^2 \sum_{i=1}^N q_r^i S_{r_0}^i + \left( 9720 \eta^4 L^4 I^4 P^3 + 105 \eta^2 L^2 I^2 P^2 \right) \frac{1}{N} \sum_{i=1}^N S_{r_0}^i \\
    &\quad + \left( 9753 \eta^4 L^4 I^4 P^3 + 106 \eta^2 L^2 I^2 P^2 \right) \sum_{i=1}^N \bar{q}_{r_0}^i S_{r_0}^i \\
    &\quad + \left( 73 \eta^2 I + 5832 \eta^4 L^2 I^3 P^2 \rho^2 + 63 \eta^2 IP \rho^2 + 6480 \eta^4 L^2 I^3 P^2 \right) \sigma^2 \\
    &\quad + \left( 9720 \eta^4 L^2 I^4 P^3 + 105 \eta^2 I^2 P^2 \right) \mathbb{E} \left[ \left\| \nabla f(\bvx_{r_0}) \right\|^2 \middle| \gQ \right] \\
    &\leq 36 \eta^2 L^2 I^2 \sum_{i=1}^N q_r^i S_{r_0}^i + 109 \eta^2 L^2 I^2 P^2 \sum_{i=1}^N \left( \bar{q}_{r_0}^i + \frac{1}{N} \right) S_{r_0}^i \\
    &\quad + \left( 75 \eta^2 I + 65 \eta^2 IP \rho^2 \right) \sigma^2 + 108 \eta^2 I^2 P^2 \mathbb{E} \left[ \left\| \nabla f(\bvx_{r_0}) \right\|^2 \middle| \gQ \right],
\end{align*}
where the last inequality uses $\eta \leq \frac{1}{60 LIP}$. This proves
\Eqref{eq:dm_ub_single}. \Eqref{eq:dm_ub_sum} follows by summing over $r \in \{r_0,
\ldots, r_0 + P - 1\}$ and $k \in \{0, \ldots, I - 1\}$, taking total expectation, and
applying the condition $\mathbb{E}_{r_0}[\bar{q}_{r_0}^i] = \frac{1}{N}$.
\end{proof}

\begin{lemma} \label{lem:quadratic_err}
If $\eta \leq \frac{1}{60 LIP}$, $\gamma \eta \leq \frac{1}{60 LIP}$,
$\mathbb{E}_{r_0}[\bar{q}_{r_0}^i] = \frac{1}{N}$, and $\mathbb{E} \left[ \sum_{i=1}^N
\left( v_{r_0}^i \right)^2 \Lambda_{r_0}^i \right] \leq \rho^2$, then
\begin{align}
    \mathbb{E} \left[ \left\| \bvx_{r_0+P} - \bvx_{r_0} \right\|^2 \right] &\leq 6 \gamma^2 \eta^2 I^2 P^2 \mathbb{E} \left[ \|\nabla f(\bvx_{r_0})\|^2 \right] + \gamma \eta IP \left( 5 \gamma \eta \rho^2 + 7 \eta^2 LI \right) \sigma^2 \nonumber \\
    &\quad + 11 \gamma^2 \eta^2 L^2 I^2 P^2 \frac{1}{N} \sum_{i=1}^N \mathbb{E} \left[ S_{r_0}^i \right]. \label{eq:exp_sq_update}
\end{align}
\end{lemma}

\begin{proof}
By the algorithm definition,
\begin{equation*}
    \bvx_{r_0+P} - \bvx_{r_0} = -\gamma \eta \sum_{r=r_0}^{r_0+P-1} \sum_{i=1}^N \sum_{k=0}^{I-1} q_r^i (\nabla F_i(\vx_{r,k}^i; \xi_{r,k}^i) - \mG_{r_0}^i + \mG_{r_0}) = -\gamma \eta \sum_{r=r_0}^{r_0+P-1} \sum_{i=1}^N \sum_{k=0}^{I-1} q_r^i \vg_{r,k}^i.
\end{equation*}

To obtain the variance of the update $\bvx_{r_0+P} - \bvx_{r_0}$,
\begin{align}
    &\mathbb{E} \left[ \left\| \sum_{r=r_0}^{r_0+P} \sum_{i=1}^N \sum_{k=0}^{I-1} q_r^i \left( \vg_{r,k}^i - \bar{\vg}_{r,k}^i \right) \right\|^2 \right] \nonumber \\
    &\quad = \mathbb{E} \left[ \left\| \sum_{r=r_0}^{r_0+P} \sum_{i=1}^N \sum_{k=0}^{I-1} q_r^i \left( \left( \nabla F_i(\vx_{r,k}^i; \xi_{r,k}^i) - \nabla f_i(\vx_{r,k}^i) \right) - \left( \mG_{r_0}^i - \bar{\mG}_{r_0}^i \right) + \left( \mG_{r_0} - \bar{\mG}_{r_0} \right) \right) \right\|^2 \right] \nonumber \\
    &\quad \leq 2 \underbrace{\mathbb{E} \left[ \left\| \sum_{r=r_0}^{r_0+P} \sum_{i=1}^N \sum_{k=0}^{I-1} q_r^i \left( \nabla F_i(\vx_{r,k}^i; \xi_{r,k}^i) - \nabla f_i(\vx_{r,k}^i) \right) \right\|^2 \right]}_{A_1} \nonumber \\
    &\quad\quad + 2 \underbrace{\mathbb{E} \left[ \left\| \sum_{r=r_0}^{r_0+P} \sum_{i=1}^N \sum_{k=0}^{I-1} q_r^i \left( - \mG_{r_0}^i + \bar{\mG}_{r_0}^i + \mG_{r_0} - \bar{\mG}_{r_0} \right) \right\|^2 \right]}_{A_2}. \label{eq:round_update_var_inter}
\end{align}
We can bound the two terms $A_1$ and $A_2$ separately as follows. For $A_1$,
\begin{align}
    A_1 &= 2 \mathbb{E} \left[ \left\| \sum_{r=r_0}^{r_0+P} \sum_{i=1}^N \sum_{k=0}^{I-1} q_r^i \left( \nabla F_i(\vx_{r,k}^i; \xi_{r,k}^i) - \nabla f_i(\vx_{r,k}^i) \right) \right\|^2 \right] \nonumber \\
    &\quad = 2 \mathbb{E} \left[ \mathbb{E} \left[ \left\| \sum_{r=r_0}^{r_0+P} \sum_{i=1}^N \sum_{k=0}^{I-1} q_r^i \left( \nabla F_i(\vx_{r,k}^i; \xi_{r,k}^i) - \nabla f_i(\vx_{r,k}^i) \right) \right\|^2 \middle| \gQ \right] \right] \nonumber \\
    &\quad \Eqmark{i}{=} 2 \mathbb{E} \left[ \sum_{r=r_0}^{r_0+P} \sum_{i=1}^N \sum_{k=0}^{I-1} \mathbb{E} \left[ \left\| q_r^i \left( \nabla F_i(\vx_{r,k}^i; \xi_{r,k}^i) - \nabla f_i(\vx_{r,k}^i) \right) \right\|^2 \middle| \gQ \right] \right] \nonumber \\
    &\quad = 2 \sum_{r=r_0}^{r_0+P} \sum_{i=1}^N \sum_{k=0}^{I-1} \mathbb{E} \left[ \left( q_r^i \right)^2 \mathbb{E} \left[ \left\| \nabla F_i(\vx_{r,k}^i; \xi_{r,k}^i) - \nabla f_i(\vx_{r,k}^i) \right\|^2 \middle| \gQ \right] \right] \nonumber \\
    &\quad \leq 2 \sigma^2 \sum_{r=r_0}^{r_0+P} \sum_{i=1}^N \sum_{k=0}^{I-1} \mathbb{E} \left[ \left( q_r^i \right)^2 \right] \nonumber \\
    &\quad \leq 2 IP \rho^2 \sigma^2, \label{eq:round_update_A1}
\end{align}
where $(i)$ uses the fact that for each $i$, $\left\{ q_r^i \left( \nabla
F_i(\vx_{r,k}^i) - \nabla f_i(\vx_{r,k}^i) \right) \right\}_{r,k}$ is a
martingale difference sequence with respect to $\gG$ (when conditioned on $\gQ$)
and that stochastic gradient noise is independent across clients. For $A_2$,
\begin{align*}
    &\sum_{r=r_0}^{r_0+P} \sum_{i=1}^N \sum_{k=0}^{I-1} q_r^i \left( - \mG_{r_0}^i - \bar{\mG}_{r_0}^i + \mG_{r_0} - \bar{\mG}_{r_0} \right) \\
    &\quad = - \sum_{r=r_0}^{r_0+P} \sum_{i=1}^N \sum_{k=0}^{I-1} q_r^i \left( \mG_{r_0}^i - \bar{\mG}_{r_0}^i \right) + \sum_{r=r_0}^{r_0+P} \sum_{i=1}^N \sum_{k=0}^{I-1} q_r^i \left( \mG_{r_0} - \bar{\mG}_{r_0} \right) \\
    &\quad = - I \sum_{r=r_0}^{r_0+P} \sum_{i=1}^N q_r^i \left( \mG_{r_0}^i - \bar{\mG}_{r_0}^i \right) + IP \left( \mG_{r_0} - \bar{\mG}_{r_0} \right) \\
    &\quad = - IP \sum_{i=1}^N \left( \frac{1}{P} \sum_{r=r_0}^{r_0+P} q_r^i \right) \left( \mG_{r_0}^i - \bar{\mG}_{r_0}^i \right) + IP \sum_{i=1}^N \frac{1}{N} \left( \mG_{r_0}^i - \bar{\mG}_{r_0}^i \right) \\
    &\quad = - IP \sum_{i=1}^N \left( \bar{q}_{r_0}^i - \frac{1}{N} \right) \left( \mG_{r_0}^i - \bar{\mG}_{r_0}^i \right),
\end{align*}
so
\begin{align}
    &A_2 = 2 I^2 P^2 \mathbb{E} \left[ \left\| \sum_{i=1}^N \left( \bar{q}_{r_0}^i - \frac{1}{N} \right) \left( \mG_{r_0}^i - \bar{\mG}_{r_0}^i \right) \right\|^2 \right] \nonumber \\
    &\quad = 2 I^2 P^2 \mathbb{E} \left[ \mathbb{E} \left[ \left\| \sum_{i=1}^N \left( \bar{q}_{r_0}^i - \frac{1}{N} \right) \left( \mG_{r_0}^i - \bar{\mG}_{r_0}^i \right) \right\|^2 \middle| \gQ \right] \right] \nonumber \\
    &\quad \Eqmark{i}{=} 2 I^2 P^2 \mathbb{E} \left[ \sum_{i=1}^N \mathbb{E} \left[ \left\| \left( \bar{q}_{r_0}^i - \frac{1}{N} \right) \left( \mG_{r_0}^i - \bar{\mG}_{r_0}^i \right) \right\|^2 \middle| \gQ \right] \right] \nonumber \\
    &\quad = 2 I^2 P^2 \mathbb{E} \left[ \sum_{i=1}^N \left( \bar{q}_{r_0}^i - \frac{1}{N} \right)^2 \mathbb{E} \left[ \left\| \mG_{r_0}^i - \bar{\mG}_{r_0}^i \right\|^2 \middle| \gQ \right] \right] \nonumber \\
    &\quad = 2 I^2 P^2 \mathbb{E} \left[ \sum_{i=1}^N \left( \bar{q}_{r_0}^i - \frac{1}{N} \right)^2 \mathbb{E} \left[ \left\| \frac{1}{P \bar{z}_{r_0-P}^i I} \sum_{s=r_0-P}^{r_0-1} \sum_{k=0}^{I-1} z_s^i \left( \nabla F_i(\vy_{s,k}^i; \zeta_{s,k}^i) - \nabla f_i(\vy_{s,k}^i) \right) \right\|^2 \middle| \gQ \right] \right] \nonumber \\
    &\quad = 2 I^2 P^2 \mathbb{E} \left[ \sum_{i=1}^N \left( \bar{q}_{r_0}^i - \frac{1}{N} \right)^2 \sum_{s=r_0-P}^{r_0-1} \sum_{k=0}^{I-1} \mathbb{E} \left[ \left\| \frac{1}{P \bar{z}_{r_0-P}^i I} z_s^i \left( \nabla F_i(\vy_{s,k}^i; \zeta_{s,k}^i) - \nabla f_i(\vy_{s,k}^i) \right) \right\|^2 \middle| \gQ \right] \right] \nonumber \\
    &\quad = 2 I^2 P^2 \mathbb{E} \left[ \sum_{i=1}^N \left( \bar{q}_{r_0}^i - \frac{1}{N} \right)^2 \sum_{s=r_0-P}^{r_0-1} \sum_{k=0}^{I-1} \frac{1}{P^2 \left( \bar{z}_{r_0-P}^i \right)^2 I^2} \left( z_s^i \right)^2 \mathbb{E} \left[ \left\| \nabla F_i(\vy_{s,k}^i; \zeta_{s,k}^i) - \nabla f_i(\vy_{s,k}^i) \right\|^2 \middle| \gQ \right] \right] \nonumber \\
    &\quad = 2 I^2 P^2 \frac{\sigma^2}{I} \mathbb{E} \left[ \sum_{i=1}^N \left( \bar{q}_{r_0}^i - \frac{1}{N} \right)^2 \sum_{s=r_0-P}^{r_0-1} \frac{\left( z_s^i \right)^2}{P^2 \left( \bar{z}_{r_0-P}^i \right)^2} \right] \nonumber \\
    &\quad = 2 IP \sigma^2 \mathbb{E} \left[ \sum_{i=1}^N \left( \bar{q}_{r_0}^i - \frac{1}{N} \right)^2 \frac{ \frac{1}{P} \sum_{s=r_0-P}^{r_0-1} \left( z_s^i \right)^2 }{\left( \frac{1}{P} \sum_{s=r_0-P}^{r_0-1} z_s^i \right)^2} \right] \leq 2IP \rho^2 \sigma^2, \label{eq:round_update_A2}
\end{align}
where $(i)$ uses the fact that, conditioned on $\gQ$, the variables $\mG_{r_0}^i -
\bar{\mG}_{r_0}^i$ depend only on stochastic gradient noise, which is independent across
clients, and the last inequality uses the condition $\mathbb{E} \left[ \sum_{i=1}^N
\left( v_{r_0}^i \right)^2 \Lambda_{r_0}^i \right] \leq \rho^2$. Plugging
\Eqref{eq:round_update_A1} and \Eqref{eq:round_update_A2} into
\Eqref{eq:round_update_var_inter} yields
\begin{equation}
    \mathbb{E} \left[ \left\| \sum_{r=r_0}^{r_0+P} \sum_{i=1}^N \sum_{k=0}^{I-1} q_r^i \left( \vg_{r,k}^i - \bar{\vg}_{r,k}^i \right) \right\|^2 \right] \leq 4 IP \rho^2 \sigma^2. \label{eq:window_update_var}
\end{equation}

Therefore
\begin{align}
    &\mathbb{E} \left[ \left\| \bvx_{r_0+P} - \bvx_{r_0} \right\|^2 \right] \nonumber \\
    &\quad = \gamma^2 \eta^2 \mathbb{E} \left[ \left\| \sum_{r=r_0}^{r_0+P-1} \sum_{i=1}^N \sum_{k=0}^{I-1} q_r^i \vg_{r,k}^i \right\|^2 \right] \nonumber \\
    &\quad \leq \gamma^2 \eta^2 \mathbb{E} \left[ \left\| \sum_{r=r_0}^{r_0+P-1} \sum_{i=1}^N \sum_{k=0}^{I-1} q_r^i \bar{\vg}_{r,k}^i \right\|^2 \right] + \gamma^2 \eta^2 \mathbb{E} \left[ \left\| \sum_{r=r_0}^{r_0+P-1} \sum_{i=1}^N \sum_{k=0}^{I-1} q_r^i \left( \vg_{r,k}^i - \bar{\vg}_{r,k}^i \right) \right\|^2 \right] \nonumber \\
    &\quad \Eqmark{i}{\leq} \gamma^2 \eta^2 \mathbb{E} \left[ \left\| \sum_{r=r_0}^{r_0+P-1} \sum_{i=1}^N \sum_{k=0}^{I-1} q_r^i \bar{\vg}_{r,k}^i \right\|^2 \right] + 4 \gamma^2 \eta^2 IP \rho^2 \sigma^2 \nonumber \\
    &\quad \leq \gamma^2 \eta^2 IP \sum_{r=r_0}^{r_0+P-1} \sum_{i=1}^N \sum_{k=0}^{I-1} \mathbb{E} \left[ q_r^i \left\| \nabla f_i(\vx_{r,k}^i) - \bar{\mG}_{r_0}^i + \bar{\mG}_{r_0} \right\|^2 \right] + 4 \gamma^2 \eta^2 IP \rho^2 \sigma^2 \nonumber \\
    &\quad \Eqmark{ii}{\leq} 5 \gamma^2 \eta^2 IP \sum_{r=r_0}^{r_0+P-1} \sum_{i=1}^N \sum_{k=0}^{I-1} \mathbb{E} \big[ q_r^i \big( \|\nabla f(\bvx_{r_0})\|^2 + \|\nabla f_i(\vx_{r,k}^i) - \nabla f_i(\bvx_{r,k})\|^2 \nonumber \\
    &\quad\quad + \|\nabla f_i(\bvx_{r,k}) - \nabla f_i(\bvx_{r_0})\|^2 + \|\nabla f_i(\bvx_{r_0}) - \bar{\mG}_{r_0}^i \|^2 + \|\nabla f(\bvx_{r_0}) - \bar{\mG}_{r_0} \|^2 \big) \big] + 4 \gamma^2 \eta^2 IP \rho^2 \sigma^2 \nonumber \\
    &\quad \leq 5 \gamma^2 \eta^2 I^2 P^2 \mathbb{E} \left[ \|\nabla f(\bvx_{r_0})\|^2 \right] \nonumber \\
    &\quad\quad + 5 \gamma^2 \eta^2 L^2 IP \sum_{r=r_0}^{r_0+P-1} \sum_{k=0}^{I-1} \sum_{i=1}^N \left( \mathbb{E} \left[ q_r^i \|\vx_{r,k}^i - \bvx_{r,k}\|^2 \right] + \mathbb{E} \left[ q_r^i \|\bvx_{r,k} - \bvx_{r_0}\|^2 \right] \right) \nonumber \\
    &\quad\quad + 5 \gamma^2 \eta^2 I^2 P^2 \mathbb{E} \left[ \|\nabla f(\bvx_{r_0}) - \bar{\mG}_{r_0} \|^2 + \sum_{i=1}^N \left( \frac{1}{P} \sum_{r=r_0}^{r_0+P-1} q_r^i \right) \|\nabla f_i(\bvx_{r_0}) - \bar{\mG}_{r_0}^i \|^2 \right] \nonumber \\
    &\quad\quad + 4 \gamma^2 \eta^2 IP \rho^2 \sigma^2 \nonumber \\
    &\quad \Eqmark{iii}{\leq} 5 \gamma^2 \eta^2 I^2 P^2 \mathbb{E} \left[ \|\nabla f(\bvx_{r_0})\|^2 \right] + 5 \gamma^2 \eta^2 L^2 IP \mathbb{E} \left[ \sum_{r=r_0}^{r_0+P-1} \sum_{k=0}^{I-1} D_{r,k} + M_{r,k} \right] \nonumber \\
    &\quad\quad + 10 \gamma^2 \eta^2 I^2 P^2 \frac{1}{N} \sum_{i=1}^N \mathbb{E} \left[ \|\nabla f_i(\bvx_{r_0}) - \bar{\mG}_{r_0}^i \|^2 \right] + 4 \gamma^2 \eta^2 IP \rho^2 \sigma^2, \label{eq:quadratic_err_inter}
\end{align}
where $(i)$ uses \Eqref{eq:window_update_var}, $(ii)$ uses the decomposition
\begin{align*}
    \nabla f_i(\vx_{r,k}^i) - \bar{\mG}_{r_0}^i + \bar{\mG}_{r_0} &= \nabla f(\bvx_{r_0}) + (\nabla f_i(\vx_{r,k}^i) - \nabla f_i(\bvx_{r,k})) + (\nabla f_i(\bvx_{r,k}) - \nabla f_i(\bvx_{r_0})) \\
    &\quad + (\nabla f_i(\bvx_{r_0}) - \bar{\mG}_{r_0}^i) + (\nabla f(\bvx_{r_0}) - \bar{\mG}_{r_0}),
\end{align*}
and $(iii)$ uses $\mG_{r_0} = \frac{1}{N} \sum_{i=1}^N \mG_{r_0}^i$ and $\mathbb{E}_{r_0} \left[ \bar{q}_{r_0}^i \right] = \frac{1}{N}$ to obtain
\begin{align*}
    &\mathbb{E} \left[ \|\nabla f(\bvx_{r_0}) - \bar{\mG}_{r_0}\|^2 + \sum_{i=1}^N \left( \frac{1}{P} \sum_{r=r_0}^{r_0+P-1} q_r^i \right) \|\nabla f_i(\bvx_{r_0}) - \bar{\mG}_{r_0}^i\|^2 \right] \\
    &\quad = \mathbb{E} \left[ \mathbb{E}_{r_0} \left[ \|\nabla f(\bvx_{r_0}) - \bar{\mG}_{r_0}\|^2 + \sum_{i=1}^N \left( \frac{1}{P} \sum_{r=r_0}^{r_0+P-1} q_r^i \right) \|\nabla f_i(\bvx_{r_0}) - \bar{\mG}_{r_0}^i\|^2 \right] \right] \\
    &\quad = \mathbb{E} \left[ \|\nabla f(\bvx_{r_0}) - \bar{\mG}_{r_0}\|^2 + \sum_{i=1}^N \mathbb{E}_{r_0} \left[ \frac{1}{P} \sum_{r=r_0}^{r_0+P-1} q_r^i \right] \|\nabla f_i(\bvx_{r_0}) - \bar{\mG}_{r_0}^i\|^2 \right] \\
    &\quad = \mathbb{E} \left[ \|\nabla f(\bvx_{r_0}) - \bar{\mG}_{r_0}\|^2 + \frac{1}{N} \sum_{i=1}^N \|\nabla f_i(\bvx_{r_0}) - \bar{\mG}_{r_0}^i\|^2 \right] \\
    &\quad = \mathbb{E} \left[ \left\| \frac{1}{N} \sum_{i=1}^N \left( \nabla f_i(\bvx_{r_0}) - \bar{\mG}_{r_0}^i \right) \right\|^2 + \frac{1}{N} \sum_{i=1}^N \|\nabla f_i(\bvx_{r_0}) - \bar{\mG}_{r_0}^i\|^2 \right] \\
    &\quad \leq \mathbb{E} \left[ \frac{1}{N} \sum_{i=1}^N \left\| \nabla f_i(\bvx_{r_0}) - \bar{\mG}_{r_0}^i \right\|^2 + \frac{1}{N} \sum_{i=1}^N \|\nabla f_i(\bvx_{r_0}) - \bar{\mG}_{r_0}^i\|^2 \right].
\end{align*}
The remaining term in \Eqref{eq:quadratic_err_inter} can be simplified as:
\begin{align*}
    \mathbb{E} \left[ \|\nabla f_i(\bvx_{r_0}) - \bar{\mG}_{r_0}^i\|^2 \right] &\leq \mathbb{E} \left[ \left\|\nabla f_i(\bvx_{r_0}) - \frac{1}{IP} \sum_{r=r_0-P}^{r_0-1} \sum_{k=0}^{I-1} \frac{z_r^i}{\bar{z}_{r_0-P}^i} \nabla f_i(\vy_{r,k}^i) \right\|^2 \right] \\
    &\leq \mathbb{E} \left[ \left\| \frac{1}{IP} \sum_{r=r_0-P}^{r_0-1} \sum_{k=0}^{I-1} \frac{z_r^i}{\bar{z}_{r_0-P}^i} \left( \nabla f_i(\bvx_{r_0}) - \nabla f_i(\vy_{r,k}^i) \right) \right\|^2 \right] \\
    &\leq \frac{1}{IP} \sum_{r=r_0-P}^{r_0-1} \sum_{k=0}^{I-1} \mathbb{E} \left[ \frac{z_r^i}{\bar{z}_{r_0-P}^i} \left\| \nabla f_i(\bvx_{r_0}) - \nabla f_i(\vy_{r,k}^i) \right\|^2 \right] \\
    &\leq \frac{L^2}{IP} \sum_{r=r_0-P}^{r_0-1} \sum_{k=0}^{I-1} \mathbb{E} \left[ \frac{z_r^i}{\bar{z}_{r_0-P}^i} \left\| \bvx_{r_0} - \vy_{r,k}^i \right\|^2 \right] \\
    &= L^2 \mathbb{E} \left[ \frac{1}{IP} \sum_{r=r_0-P}^{r_0-1} \sum_{k=0}^{I-1} \frac{z_r^i}{\bar{z}_{r_0-P}^i} \mathbb{E} \left[ \left\| \bvx_{r_0} - \vy_{r,k}^i \right\|^2 \middle| \gQ \right] \right] \\
    &= L^2 \mathbb{E} \left[ S_{r_0}^i \right].
\end{align*}
Plugging back to \Eqref{eq:quadratic_err_inter} yields
\begin{align*}
    &\mathbb{E} \left[ \left\| \bvx_{r_0+P} - \bvx_{r_0} \right\|^2 \right] \\
    &\quad \leq 5 \gamma^2 \eta^2 I^2 P^2 \mathbb{E} \left[ \|\nabla f(\bvx_{r_0})\|^2 \right] + 4 \gamma^2 \eta^2 IP \rho^2 \sigma^2 \\
    &\quad\quad + 5 \gamma^2 \eta^2 L^2 IP \mathbb{E} \left[ \sum_{r=r_0}^{r_0+P-1} \sum_{k=0}^{I-1} (D_{r,k} + M_{r,k}) \right] + 10 \gamma^2 \eta^2 L^2 I^2 P^2 \frac{1}{N} \sum_{i=1}^N \mathbb{E} \left[ S_{r_0}^i \right] \\
    &\quad \Eqmark{i}{\leq} 5 \gamma^2 \eta^2 I^2 P^2 \mathbb{E} \left[ \|\nabla f(\bvx_{r_0})\|^2 \right] + 4 \gamma^2 \eta^2 IP \rho^2 \sigma^2 \\
    &\quad\quad + 5 \gamma^2 \eta^2 L^2 IP \bigg( 108 \eta^2 I^3 P^3 \mathbb{E} \left[ \left\| \nabla f(\bvx_{r_0}) \right\|^2 \right] + \left( 75 \eta^2 I^2 P + 65 \eta^2 I^2 P^2 \rho^2 \right) \sigma^2 \\
    &\quad\quad + 254 \eta^2 L^2 I^3 P^3 \frac{1}{N} \sum_{i=1}^N\mathbb{E} \left[ S_{r_0}^i \right]\bigg) + 10 \gamma^2 \eta^2 L^2 I^2 P^2 \frac{1}{N} \sum_{i=1}^N \mathbb{E} \left[ S_{r_0}^i \right] \\
    &\quad \leq 5 \gamma^2 \eta^2 I^2 P^2 \left( 1 + 108 \eta^2 L^2 I^2 P^2 \right) \mathbb{E} \left[ \|\nabla f(\bvx_{r_0})\|^2 \right] \\
    &\quad\quad + \gamma \eta IP \left( 4 \gamma \eta \rho^2 + 375 \gamma \eta^3 L^2 I^2 P + 325 \gamma \eta^3 L^2 I^2 P^2 \rho^2 \right) \sigma^2 \\
    &\quad\quad + 10 \gamma^2 \eta^2 L^2 I^2 P^2 \left( 1 + 127 \eta^2 L^2 I^2 P^2 \right) \frac{1}{N} \sum_{i=1}^N \mathbb{E} \left[ S_{r_0}^i \right] \\
    &\quad \Eqmark{ii}{\leq} 6 \gamma^2 \eta^2 I^2 P^2 \mathbb{E} \left[ \|\nabla f(\bvx_{r_0})\|^2 \right] + \gamma \eta IP \left( 5 \gamma \eta \rho^2 + 7 \eta^2 LI \right) \sigma^2 + 11 \gamma^2 \eta^2 L^2 I^2 P^2 \frac{1}{N} \sum_{i=1}^N \mathbb{E} \left[ S_{r_0}^i \right].
\end{align*}
where $(i)$ uses \Eqref{eq:dm_ub_sum} from Lemma \ref{lem:dm_ub} and $(ii)$ uses $\eta
\leq \frac{1}{60 LIP}$ and $\gamma \eta \leq \frac{1}{60 LIP}$. This proves
\Eqref{eq:exp_sq_update}.
\end{proof}

\begin{lemma} \label{lem:staleness_rec_ub}
Suppose that $P_{\gQ_{r_0}}(\bar{q}_{r_0}^i > 0) \geq p_{\text{sample}}$, $\mathbb{E}
\left[ w_{r_0}^i \middle| \gQ_{:r_0} \right] \leq \frac{P^2}{N}$,
$\mathbb{E}_{r_0}[\bar{q}_{r_0}^i] = \frac{1}{N}$, and $\mathbb{E} \left[ \sum_{i=1}^N
\left( v_{r_0}^i \right)^2 \Lambda_{r_0}^i \right] \leq \rho^2$. If $\eta \leq
\frac{\sqrt{p_{\text{sample}}}}{60 LIP}$ and $\gamma \eta \leq
\frac{p_{\text{sample}}}{60 LIP}$, then
\begin{align*}
    \frac{1}{N} \sum_{i=1}^N \mathbb{E} \left[ S_{r_0+P}^i \right] &\leq \left( 324 \eta^2 I^2 P^2 + \frac{20}{p_{\text{sample}}} \gamma^2 \eta^2 I^2 P^2 \right) \mathbb{E} \left[ \left\| \nabla f(\bvx_{r_0}) \right\|^2 \right] \\
    &\quad + \left( 226 \eta^2 I + 195 \eta^2 IP \rho^2 + \frac{17}{p_{\text{sample}}} \gamma^2 \eta^2 IP \rho^2 \right) \sigma^2 \\
    &\quad + \left( 1 - \frac{1}{2} p_{\text{sample}} \right) \frac{1}{N} \sum_{j=1}^N \mathbb{E} \left[ S_{r_0}^j \right].
\end{align*}
\end{lemma}

\begin{proof}
Let $1 \leq i \leq N$. We can consider the value $S_{r_0+P}^i$ under two cases:
$\bar{q}_{r_0}^i > 0$ and the complement. Let $A_{r_0}^i = \{\bar{q}_{r_0}^i > 0\}$. Denote
\begin{align*}
    B_1^i &= \indicator{A_{r_0}^i} \frac{1}{I P \bar{q}_{r_0}^i} \sum_{s=r_0}^{r_0+P-1} \sum_{k=0}^{I-1} q_s^i \mathbb{E} \left[ \left\| \vx_{s,k}^i - \bvx_{r_0+P} \right\|^2 \middle| \gQ \right] \\
    B_2^i &= \indicator{\bar{A}_{r_0}^i} \frac{1}{I P \bar{z}_{r_0}^i} \sum_{s=r_0}^{r_0+P-1} \sum_{k=0}^{I-1} z_s^i \mathbb{E} \left[ \left\| \vy_{s,k}^i - \bvx_{r_0+P} \right\|^2 \middle| \gQ \right].
\end{align*}
Then $S_{r_0+P}^i = B_1^i + B_2^i$, and we can consider the two cases separately.

Notice that
\begin{equation*}
    \| \vx_{s,k}^i - \bvx_{r_0+P} \|^2 \leq 3 \|\vx_{s,k}^i - \bvx_{s,k}\|^2 + 3 \|\bvx_{s,k} - \bvx_{r_0}\|^2 + 3 \|\bvx_{r_0+P} - \bvx_{r_0}\|^2.
\end{equation*}
Therefore
\begin{align*}
    B_1^i &= \indicator{A_{r_0}^i} \frac{1}{I P \bar{q}_{r_0}^i} \sum_{s=r_0}^{r_0+P-1} \sum_{k=0}^{I-1} q_s^i \mathbb{E} \left[ \left\| \vx_{s,k}^i - \bvx_{r_0+P} \right\|^2 \middle| \gQ \right] \\
    &\leq \indicator{A_{r_0}^i} \frac{3}{I P \bar{q}_{r_0}^i} \sum_{s=r_0}^{r_0+P-1} \sum_{k=0}^{I-1} q_s^i \bigg( \mathbb{E} \left[ \|\vx_{s,k}^i - \bvx_{s,k}\|^2 \middle| \gQ \right] + \mathbb{E} \left[ \|\bvx_{s,k} - \bvx_{r_0}\|^2 \middle| \gQ \right] \\
    &\quad + \mathbb{E} \left[ \|\bvx_{r_0+P} - \bvx_{r_0}\|^2 \middle| \gQ \right] \bigg) \\
    &\leq \indicator{A_{r_0}^i} \frac{3}{I P \bar{q}_{r_0}^i} \sum_{s=r_0}^{r_0+P-1} \sum_{k=0}^{I-1} q_s^i \left( D_{s,k} + M_{s,k} \right) + 3 \indicator{A_{r_0}^i} \mathbb{E} \left[ \|\bvx_{r_0+P} - \bvx_{r_0}\|^2 \middle| \gQ \right].
\end{align*}
Using \Eqref{eq:dm_ub_single} from Lemma \ref{lem:dm_ub},
\begin{align*}
    &\indicator{A_{r_0}^i} \frac{3}{I P \bar{q}_{r_0}^i} \sum_{s=r_0}^{r_0+P-1} \sum_{k=0}^{I-1} q_s^i \left( D_{s,k} + M_{s,k} \right) \\
    &\quad \leq \indicator{A_{r_0}^i} \frac{3}{I P \bar{q}_{r_0}^i} \sum_{s=r_0}^{r_0+P-1} \sum_{k=0}^{I-1} q_s^i \bigg( 108 \eta^2 I^2 P^2 \mathbb{E} \left[ \left\| \nabla f(\bvx_{r_0}) \right\|^2 \middle| \gQ \right] + \left( 75 \eta^2 I + 65 \eta^2 IP \rho^2 \right) \sigma^2 \\
    &\quad\quad + 109 \eta^2 L^2 I^2 P^2 \sum_{j=1}^N \left( \bar{q}_{r_0}^j + \frac{1}{N} \right) S_{r_0}^j + 36 \eta^2 L^2 I^2 \sum_{j=1}^N q_s^j S_{r_0}^j \bigg) \\
    &\quad \leq 324 \eta^2 I^2 P^2 \mathbb{E} \left[ \left\| \nabla f(\bvx_{r_0}) \right\|^2 \middle| \gQ \right] + \left( 225 \eta^2 I + 195 \eta^2 IP \rho^2 \right) \sigma^2 \\
    &\quad\quad + 327 \eta^2 L^2 I^2 P^2 \sum_{j=1}^N \left( \bar{q}_{r_0}^j + \frac{1}{N} \right) S_{r_0}^j + \indicator{A_{r_0}^i} \frac{3}{P \bar{q}_{r_0}^i} \sum_{s=r_0}^{r_0+P-1} q_s^i \left( 36 \eta^2 L^2 I^2 \sum_{j=1}^N q_s^j S_{r_0}^j \right) \\
    &\quad = 324 \eta^2 I^2 P^2 \mathbb{E} \left[ \left\| \nabla f(\bvx_{r_0}) \right\|^2 \middle| \gQ \right] + \left( 225 \eta^2 I + 195 \eta^2 IP \rho^2 \right) \sigma^2 \\
    &\quad\quad + 327 \eta^2 L^2 I^2 P^2 \sum_{j=1}^N \left( \bar{q}_{r_0}^j + \frac{1}{N} \right) S_{r_0}^j + 108 \eta^2 L^2 I^2 \sum_{j=1}^N \left( \frac{\indicator{A_{r_0}^i}}{P \bar{q}_{r_0}^i} \sum_{s=r_0}^{r_0+P-1} q_s^i q_s^j \right) S_{r_0}^j.
\end{align*}
Denote $w_{r_0}^{i,j} = \frac{\indicator{A_{r_0}^i}}{P \bar{q}_{r_0}^i}
\sum_{s=r_0}^{r_0+P-1} q_s^i q_s^j$, so that $w_{r_0}^i = \frac{1}{N} \sum_{j=1}^N
w_{r_0}^{i,j}$. Then
\begin{align}
    B_1^i &\leq 324 \eta^2 I^2 P^2 \mathbb{E} \left[ \left\| \nabla f(\bvx_{r_0}) \right\|^2 \middle| \gQ \right] + \left( 225 \eta^2 I + 195 \eta^2 IP \rho^2 \right) \sigma^2 \nonumber \\
    &\quad\quad + 327 \eta^2 L^2 I^2 P^2 \sum_{j=1}^N \left( \bar{q}_{r_0}^j + \frac{1}{N} \right) S_{r_0}^j + 108 \eta^2 L^2 I^2 \sum_{j=1}^N \left( \frac{\indicator{A_{r_0}^i}}{P \bar{q}_{r_0}^i} \sum_{s=r_0}^{r_0+P-1} q_s^i q_s^j \right) S_{r_0}^j \nonumber \\
    &\quad\quad + 3 \indicator{A_{r_0}^i} \mathbb{E} \left[ \|\bvx_{r_0+P} - \bvx_{r_0}\|^2 \middle| \gQ \right] \nonumber \\
    \mathbb{E} \left[ B_1^i \middle| \gQ_{:r_0} \right] &\Eqmark{i}{\leq} 324 \eta^2 I^2 P^2 \mathbb{E} \left[ \left\| \nabla f(\bvx_{r_0}) \right\|^2 \middle| \gQ_{:r_0} \right] + \left( 225 \eta^2 I + 195 \eta^2 IP \rho^2 \right) \sigma^2 \nonumber \\
    &\quad\quad + 654 \eta^2 L^2 I^2 P^2 \frac{1}{N} \sum_{j=1}^N S_{r_0}^j + 108 \eta^2 L^2 I^2 \sum_{j=1}^N \mathbb{E} \left[ \frac{\indicator{A_{r_0}^i}}{P \bar{q}_{r_0}^i} \sum_{s=r_0}^{r_0+P-1} q_s^i q_s^j \middle| \gQ_{:r_0} \right] S_{r_0}^j \nonumber \\
    &\quad\quad + 3 \mathbb{E} \left[ \indicator{A_{r_0}^i} \|\bvx_{r_0+P} - \bvx_{r_0}\|^2 \middle| \gQ_{:r_0} \right] \nonumber \\
    \mathbb{E} \left[ B_1^i \right] &\Eqmark{ii}{\leq} 324 \eta^2 I^2 P^2 \mathbb{E} \left[ \left\| \nabla f(\bvx_{r_0}) \right\|^2 \right] + \left( 225 \eta^2 I + 195 \eta^2 IP \rho^2 \right) \sigma^2 \nonumber \\
    &\quad\quad + 654 \eta^2 L^2 I^2 P^2 \frac{1}{N} \sum_{j=1}^N \mathbb{E} \left[ S_{r_0}^j \right] + 108 \eta^2 L^2 I^2 \sum_{j=1}^N \mathbb{E} \left[ \mathbb{E} \left[ w_{r_0}^{i,j} \middle| \gQ_{:r_0} \right] S_{r_0}^j \right] \nonumber \\
    &\quad\quad + 3 \mathbb{E} \left[ \indicator{A_{r_0}^i} \|\bvx_{r_0+P} - \bvx_{r_0}\|^2 \right], \label{eq:staleness_b1}
\end{align}
where $(i)$ uses $\mathbb{E}_{r_0}[\bar{q}_{r_0}^i] = \frac{1}{N}$ and the tower
property $\mathbb{E} \left[ \mathbb{E} \left[ \cdot \middle| \gQ \right] \middle|
\gQ_{:r_0} \right] = \mathbb{E} \left[ \cdot \middle| \gQ_{:r_0} \right]$, and $(ii)$
uses the tower property $\mathbb{E} \left[ \mathbb{E} \left[ \cdot \middle| \gQ_{:r_0}
\right] \right] = \mathbb{E} \left[ \cdot \right]$.

Now consider $B_2^i$. Under $\bar{A}_{r_0}^i$, $z_r^i = z_{r-P}^i$ and $\vy_{r,k}^i =
\vy_{r-P,k}^i$ for all $r \in \{r_0, \ldots, r_0+P-1 \}$. Therefore
\begin{align*}
    B_2^i &= \indicator{\bar{A}_{r_0}^i} \frac{1}{I P \bar{z}_{r_0}^i} \sum_{s=r_0}^{r_0+P-1} \sum_{k=0}^{I-1} z_s^i \mathbb{E} \left[ \left\| \vy_{s,k}^i - \bvx_{r_0+P} \right\|^2 \middle| \gQ \right] \\
    &= \indicator{\bar{A}_{r_0}^i} \frac{1}{I P \bar{z}_{r_0-P}^i} \sum_{s=r_0-P}^{r_0-1} \sum_{k=0}^{I-1} z_s^i \mathbb{E} \left[ \left\| \vy_{s,k}^i - \bvx_{r_0+P} \right\|^2 \middle| \gQ \right] \\
    &\Eqmark{i}{\leq} \indicator{\bar{A}_{r_0}^i} \frac{1}{I P \bar{z}_{r_0-P}^i} \sum_{s=r_0-P}^{r_0-1} \sum_{k=0}^{I-1} z_s^i \bigg( (1 + \lambda) \mathbb{E} \left[ \left\| \vy_{s,k}^i - \bvx_{r_0} \right\|^2 \middle| \gQ \right] \\
    &\quad + \left( 1 + \frac{1}{\lambda} \right) \mathbb{E} \left[ \| \bvx_{r_0+P} - \bvx_{r_0} \|^2 \middle| \gQ \right] \bigg) \\
    &= \indicator{\bar{A}_{r_0}^i} (1 + \lambda) \frac{1}{I P \bar{z}_{r_0-P}^i} \sum_{s=r_0-P}^{r_0-1} \sum_{k=0}^{I-1} z_s^i \mathbb{E} \left[ \left\| \vy_{s,k}^i - \bvx_{r_0} \right\|^2 \middle| \gQ \right] \\
    &\quad + \indicator{\bar{A}_{r_0}^i} \left( 1 + \frac{1}{\lambda} \right) \mathbb{E} \left[ \left\| \bvx_{r_0+P} - \bvx_{r_0} \right\|^2 \middle| \gQ \right] \\
    &= \indicator{\bar{A}_{r_0}^i} (1 + \lambda) S_{r_0}^i + \indicator{\bar{A}_{r_0}^i} \left( 1 + \frac{1}{\lambda} \right) \mathbb{E} \left[ \left\| \bvx_{r_0+P} - \bvx_{r_0} \right\|^2 \middle| \gQ \right]
\end{align*}
where $(i)$ uses Young's inequality with an arbitrary $\lambda > 0$. Taking conditional
expectation $\mathbb{E}[\cdot | \gQ_{:r_0}]$ followed by total expectation yields
\begin{align}
    \mathbb{E}[B_2^i | \gQ_{:r_0}] &\leq \left( 1 - p_{\text{sample}} \right) (1 + \lambda) S_{r_0}^i + \left( 1 + \frac{1}{\lambda} \right) \mathbb{E} \left[ \indicator{\bar{A}_{r_0}^i} \left\| \bvx_{r_0+P} - \bvx_{r_0} \right\|^2 \middle| \gQ_{:r_0} \right] \nonumber \\
    \mathbb{E}[B_2^i] &\leq \left( 1 - p_{\text{sample}} \right) (1 + \lambda) \mathbb{E} \left[ S_{r_0}^i \right] + \left( 1 + \frac{1}{\lambda} \right) \mathbb{E} \left[ \indicator{\bar{A}_{r_0}^i} \left\| \bvx_{r_0+P} - \bvx_{r_0} \right\|^2 \right]. \label{eq:staleness_b2}
\end{align}

Adding \Eqref{eq:staleness_b1} and \Eqref{eq:staleness_b2} yields
\begin{align*}
    \mathbb{E} \left[ S_{r_0+P}^i \right] &= \mathbb{E}[B_1^i] + \mathbb{E}[B_2^i] \\
    &\leq 324 \eta^2 I^2 P^2 \mathbb{E} \left[ \left\| \nabla f(\bvx_{r_0}) \right\|^2 \right] + \left( 225 \eta^2 I + 195 \eta^2 IP \rho^2 \right) \sigma^2 \\
    &\quad + 654 \eta^2 L^2 I^2 P^2 \frac{1}{N} \sum_{j=1}^N \mathbb{E} \left[ S_{r_0}^j \right] + 108 \eta^2 L^2 I^2 \sum_{j=1}^N \mathbb{E} \left[ \mathbb{E} \left[ w_{r_0}^{i,j} \middle| \gQ_{:r_0} \right] S_{r_0}^j \right] \\
    &\quad + \left( 1 - p_{\text{sample}} \right) (1 + \lambda) \mathbb{E} \left[ S_{r_0}^i \right] + \left( 3 + \frac{1}{\lambda} \right) \mathbb{E} \left[ \|\bvx_{r_0+P} - \bvx_{r_0}\|^2 \right].
\end{align*}
Averaging over $i$,
\begin{align*}
    &\frac{1}{N} \sum_{i=1}^N \mathbb{E} \left[ S_{r_0+P}^i \right] \\
    &\quad\leq 324 \eta^2 I^2 P^2 \mathbb{E} \left[ \left\| \nabla f(\bvx_{r_0}) \right\|^2 \right] + \left( 225 \eta^2 I + 195 \eta^2 IP \rho^2 \right) \sigma^2 + 654 \eta^2 L^2 I^2 P^2 \frac{1}{N} \sum_{j=1}^N \mathbb{E} \left[ S_{r_0}^j \right] \\
    &\quad\quad + 108 \eta^2 L^2 I^2 \sum_{j=1}^N \mathbb{E} \left[ \mathbb{E} \left[ \frac{1}{N} \sum_{i=1}^N w_{r_0}^{i,j} \middle| \gQ_{:r_0} \right] S_{r_0}^j \right] \\
    &\quad\quad + \left( 1 - p_{\text{sample}} \right) (1 + \lambda) \frac{1}{N} \sum_{i=1}^N \mathbb{E} \left[ S_{r_0}^i \right] + \left( 3 + \frac{1}{\lambda} \right) \mathbb{E} \left[ \|\bvx_{r_0+P} - \bvx_{r_0}\|^2 \right] \\
    &\quad\leq 324 \eta^2 I^2 P^2 \mathbb{E} \left[ \left\| \nabla f(\bvx_{r_0}) \right\|^2 \right] + \left( 225 \eta^2 I + 195 \eta^2 IP \rho^2 \right) \sigma^2 \\
    &\quad\quad + \left( \left( 1 - p_{\text{sample}} \right) (1 + \lambda) + 654 \eta^2 L^2 I^2 P^2 \right) \frac{1}{N} \sum_{j=1}^N \mathbb{E} \left[ S_{r_0}^j \right] \\
    &\quad\quad + 108 \eta^2 L^2 I^2 \sum_{j=1}^N \mathbb{E} \left[ \mathbb{E} \left[ w_{r_0}^j \middle| \gQ_{:r_0} \right] S_{r_0}^j \right] + \left( 3 + \frac{1}{\lambda} \right) \mathbb{E} \left[ \|\bvx_{r_0+P} - \bvx_{r_0}\|^2 \right] \\
    &\quad\leq 324 \eta^2 I^2 P^2 \mathbb{E} \left[ \left\| \nabla f(\bvx_{r_0}) \right\|^2 \right] + \left( 225 \eta^2 I + 195 \eta^2 IP \rho^2 \right) \sigma^2 \\
    &\quad\quad + \left( \left( 1 - p_{\text{sample}} \right) (1 + \lambda) + 762 \eta^2 L^2 I^2 P^2 \right) \frac{1}{N} \sum_{j=1}^N \mathbb{E} \left[ S_{r_0}^j \right] + \left( 3 + \frac{1}{\lambda} \right) \mathbb{E} \left[ \|\bvx_{r_0+P} - \bvx_{r_0}\|^2 \right],
\end{align*}
where the last inequality uses the condition $\mathbb{E} \left[ w_{r_0}^i \middle|
\gQ_{:r_0} \right] \leq \frac{P^2}{N}$. We can then apply Lemma \ref{lem:quadratic_err}
and choose $\lambda = \frac{3 p_{\text{sample}}}{10 (1 - p_{\text{sample}})}$ to obtain
\begin{align*}
    &\frac{1}{N} \sum_{i=1}^N \mathbb{E} \left[ S_{r_0+P}^i \right] \\
    &\quad\leq \left( 324 \eta^2 I^2 P^2 + \left( 3 + \frac{1}{\lambda} \right) 6 \gamma^2 \eta^2 I^2 P^2 \right) \mathbb{E} \left[ \left\| \nabla f(\bvx_{r_0}) \right\|^2 \right] \\
    &\quad\quad + \left( 225 \eta^2 I + 195 \eta^2 IP \rho^2 + \left( 3 + \frac{1}{\lambda} \right) 5 \gamma^2 \eta^2 IP \rho^2 + \left( 3 + \frac{1}{\lambda} \right) 7 \gamma \eta^3 L I^2 P \right) \sigma^2 \\
    &\quad\quad + \left( \left( 1 - p_{\text{sample}} \right) (1 + \lambda) + 762 \eta^2 L^2 I^2 P^2 + \left( 3 + \frac{1}{\lambda} \right) 11 \gamma^2 \eta^2 L^2 I^2 P^2 \right) \frac{1}{N} \sum_{j=1}^N \mathbb{E} \left[ S_{r_0}^j \right] \\
    &\quad\leq \left( 324 \eta^2 I^2 P^2 + \frac{20}{p_{\text{sample}}} \gamma^2 \eta^2 I^2 P^2 \right) \mathbb{E} \left[ \left\| \nabla f(\bvx_{r_0}) \right\|^2 \right] \\
    &\quad\quad + \left( 225 \eta^2 I + 195 \eta^2 IP \rho^2 + \frac{17}{p_{\text{sample}}} \gamma^2 \eta^2 IP \rho^2 + \frac{23}{p_{\text{sample}}} \gamma \eta^3 L I^2 P \right) \sigma^2 \\
    &\quad\quad + \left( 1 - \frac{7}{10} p_{\text{sample}} + 762 \eta^2 L^2 I^2 P^2 + \frac{110}{3 p_{\text{sample}}} \gamma^2 \eta^2 L^2 I^2 P^2 \right) \frac{1}{N} \sum_{j=1}^N \mathbb{E} \left[ S_{r_0}^j \right] \\
    &\quad\leq \left( 324 \eta^2 I^2 P^2 + \frac{20}{p_{\text{sample}}} \gamma^2 \eta^2 I^2 P^2 \right) \mathbb{E} \left[ \left\| \nabla f(\bvx_{r_0}) \right\|^2 \right] \\
    &\quad\quad + \left( 226 \eta^2 I + 195 \eta^2 IP \rho^2 + \frac{17}{p_{\text{sample}}} \gamma^2 \eta^2 IP \rho^2 \right) \sigma^2 + \left( 1 - \frac{1}{2} p_{\text{sample}} \right) \frac{1}{N} \sum_{j=1}^N \mathbb{E} \left[ S_{r_0}^j \right],
\end{align*}
where the last inequality uses the conditions $\eta \leq
\frac{\sqrt{p_{\text{sample}}}}{60 LIP}$ and $\gamma \eta \leq
\frac{p_{\text{sample}}}{60 LIP}$, and we used that $3 + \frac{1}{\lambda} \leq
\frac{10}{3 p_{\text{sample}}}$.
\end{proof}

\begin{theorem}[Theorem \ref{thm:main_convergence} restated] \label{thm:convergence}
Suppose Assumptions \ref{ass:obj} and \ref{ass:sampling} hold, and $\mathbb{E}[w_{r_0}^i
| \gQ_{:r_0}] \leq \frac{P^2}{N}$, and $\mathbb{E} \left[ \sum_{i=1}^N \left( v_{r_0}^i
\right)^2 \Lambda_{r_0}^i \right] \leq \rho^2$. If
\begin{align*}
    \gamma \eta \leq \frac{p_{\text{sample}}}{60 LIP}, \quad \eta \leq \frac{\sqrt{p_{\text{sample}}}}{60 LIP},
\end{align*}
then Algorithm 1 satisfies
\begin{align*}
    \frac{P}{R} \sum_{r_0 \in \{0, P, \ldots, R-P\}} \mathbb{E}[\|\nabla f(\bvx_{r_0})\|^2] \leq \frac{5 \Delta}{\gamma \eta I R} + \left( 20 \gamma \eta L \rho^2 + 5785 \eta^2 L^2 IP \right) \sigma^2.
\end{align*}
\end{theorem}

\begin{proof}
Recall that
\begin{equation} \label{eq:window_update}
    \bvx_{r_0+P} - \bvx_{r_0} = -\gamma \eta \sum_{r=r_0}^{r_0+P-1} \sum_{i=1}^N q_r^i \sum_{k=0}^{I-1} \nabla F_i(\vx_{r,k}^i; \xi_{r,k}^i) + \gamma \eta I P \sum_{i=1}^N \left( \bar{q}_{r_0}^i - \frac{1}{N} \right) \mG_{r_0}^i
\end{equation}
Using the quadratic upper bound for smooth functions and taking total expectation,
\begin{align}
    \mathbb{E}[f(\bvx_{r_0+P}) - f(\bvx_{r_0})] &\leq - \gamma \eta \mathbb{E} \left[ \left\langle \nabla f(\bvx_{r_0}), \sum_{r=r_0}^{r_0+P-1} \sum_{i=1}^N q_r^i \sum_{k=0}^{I-1} \nabla F_i(\vx_{r,k}^i; \xi_{r,k}^i) \right\rangle \right] \nonumber \\
    &\quad + \gamma \eta IP \mathbb{E} \left[ \left\langle \nabla f(\bvx_{r_0}), \sum_{i=1}^N \left( \bar{q}_{r_0}^i - \frac{1}{N} \right) \mG_{r_0}^i \right\rangle \right] \nonumber \\
    &\quad + \frac{L}{2} \mathbb{E} \left[ \|\bvx_{r_0+P} - \bvx_{r_0}\|^2 \right] \nonumber \\
    &\Eqmark{i}{\leq} - \gamma \eta \mathbb{E} \left[ \left\langle \nabla f(\bvx_{r_0}), \mathbb{E}_{r_0} \left[ \sum_{r=r_0}^{r_0+P-1} \sum_{i=1}^N q_r^i \sum_{k=0}^{I-1} \nabla F_i(\vx_{r,k}^i; \xi_{r,k}^i) \right] \right\rangle \right] \nonumber \\
    &\quad + \gamma \eta IP \mathbb{E} \left[ \left\langle \nabla f(\bvx_{r_0}), \sum_{i=1}^N \left( \mathbb{E}_{r_0} \left[ \bar{q}_{r_0}^i \right] - \frac{1}{N} \right) \mG_{r_0}^i \right\rangle \right] \nonumber \\
    &\quad + \frac{L}{2} \mathbb{E} \left[ \|\bvx_{r_0+P} - \bvx_{r_0}\|^2 \right] \nonumber \\
    &\Eqmark{ii}{\leq} - \gamma \eta \mathbb{E} \left[ \left\langle \nabla f(\bvx_{r_0}), \mathbb{E}_{r_0} \left[ \sum_{r=r_0}^{r_0+P-1} \sum_{i=1}^N q_r^i \sum_{k=0}^{I-1} \nabla f_i(\vx_{r,k}^i) \right] \right\rangle \right] \nonumber \\
    &\quad + \frac{L}{2} \mathbb{E} \left[ \|\bvx_{r_0+P} - \bvx_{r_0}\|^2 \right], \label{eq:descent_inter_1}
\end{align}
where $(i)$ uses the law of total expectation $\mathbb{E}[\cdot] = \mathbb{E}[
\mathbb{E}_{r_0}[\cdot] ]$, and $(ii)$ uses the condition
$\mathbb{E}_{r_0}[\bar{q}_{r_0}^i] = \frac{1}{N}$ together with
\begin{align*}
    \mathbb{E}_{r_0} \left[ \sum_{r=r_0}^{r_0+P-1} \sum_{i=1}^N q_r^i \sum_{k=0}^{I-1} \nabla F_i(\vx_{r,k}^i; \xi_{r,k}^i) \right] &= \sum_{r=r_0}^{r_0+P-1} \sum_{i=1}^N \sum_{k=0}^{I-1} \mathbb{E}_{r_0} \left[ q_r^i \nabla F_i(\vx_{r,k}^i; \xi_{r,k}^i) \right] \\
    &= \sum_{r=r_0}^{r_0+P-1} \sum_{i=1}^N \sum_{k=0}^{I-1} \mathbb{E}_{r_0} \left[ \mathbb{E} \left[ q_r^i \nabla F_i(\vx_{r,k}^i; \xi_{r,k}^i) \middle| \vx_{r,k}^i \right] \right] \\
    &= \sum_{r=r_0}^{r_0+P-1} \sum_{i=1}^N \sum_{k=0}^{I-1} \mathbb{E}_{r_0} \left[ q_r^i \mathbb{E} \left[ \nabla F_i(\vx_{r,k}^i; \xi_{r,k}^i) \middle| \vx_{r,k}^i \right] \right] \\
    &= \sum_{r=r_0}^{r_0+P-1} \sum_{i=1}^N \sum_{k=0}^{I-1} \mathbb{E}_{r_0} \left[ q_r^i \nabla f_i(\vx_{r,k}^i) \right] \\
    &= \mathbb{E}_{r_0} \left[ \sum_{r=r_0}^{r_0+P-1} \sum_{i=1}^N q_r^i \sum_{k=0}^{I-1} \nabla f_i(\vx_{r,k}^i) \right].
\end{align*}

Focusing on the inner product term of \Eqref{eq:descent_inter_1}:
\begin{align}
    &- \gamma \eta \mathbb{E} \left[ \left\langle \nabla f(\bvx_{r_0}), \mathbb{E}_{r_0} \left[ \sum_{r=r_0}^{r_0+P-1} \sum_{i=1}^N q_r^i \sum_{k=0}^{I-1} \nabla f_i(\vx_{r,k}^i) \right] \right\rangle \right] \nonumber \\
    &\quad = -\frac{\gamma \eta}{IP} \mathbb{E} \left[ \left\langle IP \nabla f(\bvx_{r_0}), \mathbb{E}_{r_0} \left[ \sum_{r=r_0}^{r_0+P-1} \sum_{i=1}^N q_r^i \sum_{k=0}^{I-1} \nabla f_i(\vx_{r,k}^i) \right] \right\rangle \right] \nonumber \\
    &\quad \leq -\frac{\gamma \eta IP}{2} \mathbb{E} \left[ \left\| \nabla f(\bvx_{r_0}) \right\|^2 \right] + \frac{\gamma \eta}{2IP} \mathbb{E} \left[ \left\| \mathbb{E}_{r_0} \left[ \sum_{r=r_0}^{r_0+P-1} \sum_{i=1}^N q_r^i \sum_{k=0}^{I-1} \nabla f_i(\vx_{r,k}^i) \right] - IP \nabla f(\bvx_{r_0}) \right\|^2 \right], \label{eq:descent_linear_inter}
\end{align}
where we used $-\langle a, b \rangle = \frac{1}{2} \|b-a\|^2 - \frac{1}{2} \|a\|^2 -
\frac{1}{2} \|b\|^2 \leq \frac{1}{2} \|b-a\|^2 - \frac{1}{2} \|a\|^2$. Also
\begin{align*}
    &\mathbb{E} \left[ \left\| \mathbb{E}_{r_0} \left[ \sum_{r=r_0}^{r_0+P-1} \sum_{i=1}^N q_r^i \sum_{k=0}^{I-1} \nabla f_i(\vx_{r,k}^i) \right] - IP \nabla f(\bvx_{r_0}) \right\|^2 \right] \\
    &\quad = \mathbb{E} \left[ \left\| \mathbb{E}_{r_0} \left[ \sum_{r=r_0}^{r_0+P-1} \sum_{i=1}^N q_r^i \sum_{k=0}^{I-1} \left( \nabla f_i(\vx_{r,k}^i) - \nabla f(\bvx_{r_0}) \right) \right] \right\|^2 \right] \\
    &\quad \Eqmark{i}{\leq} 3 \mathbb{E} \left[ \left\| \mathbb{E}_{r_0} \left[ \sum_{r=r_0}^{r_0+P-1} \sum_{i=1}^N q_r^i \sum_{k=0}^{I-1} \left( \nabla f_i(\vx_{r,k}^i) - \nabla f_i(\bvx_{r,k}) \right) \right] \right\|^2 \right] \\
    &\quad\quad + 3 \mathbb{E} \left[ \left\| \mathbb{E}_{r_0} \left[ \sum_{r=r_0}^{r_0+P-1} \sum_{i=1}^N q_r^i \sum_{k=0}^{I-1} (\nabla f_i(\bvx_{r,k}) - \nabla f_i(\bvx_{r_0})) \right] \right\|^2 \right] \\
    &\quad\quad + 3 \mathbb{E} \left[ \left\| \mathbb{E}_{r_0} \left[ \sum_{r=r_0}^{r_0+P-1} \sum_{i=1}^N q_r^i \sum_{k=0}^{I-1} (\nabla f_i(\bvx_{r_0}) - \nabla f(\bvx_{r_0})) \right] \right\|^2 \right] \\
    &\quad \Eqmark{ii}{\leq} 3 \mathbb{E} \left[ \left\| \sum_{r=r_0}^{r_0+P-1} \sum_{i=1}^N q_r^i \sum_{k=0}^{I-1} \left( \nabla f_i(\vx_{r,k}^i) - \nabla f_i(\bvx_{r,k}) \right) \right\|^2 \right] \\
    &\quad\quad + 3 \mathbb{E} \left[ \left\| \sum_{r=r_0}^{r_0+P-1} \sum_{i=1}^N q_r^i \sum_{k=0}^{I-1} (\nabla f_i(\bvx_{r,k}) - \nabla f_i(\bvx_{r_0})) \right\|^2 \right] \\
    &\quad\quad + 3 \mathbb{E} \left[ \left\| I \sum_{i=1}^N \mathbb{E}_{r_0} \left[ \sum_{r=r_0}^{r_0+P-1} q_r^i \right] (\nabla f_i(\bvx_{r_0}) - \nabla f(\bvx_{r_0})) \right\|^2 \right] \\
    &\quad \Eqmark{iii}{\leq} 3IP \sum_{r=r_0}^{r_0+P-1} \sum_{i=1}^N \sum_{k=0}^{I-1} \mathbb{E} \left[ q_r^i \left\|  \nabla f_i(\vx_{r,k}^i) - \nabla f_i(\bvx_{r,k}) \right\|^2 \right] \\
    &\quad\quad + 3IP \sum_{r=r_0}^{r_0+P-1} \sum_{i=1}^N \sum_{k=0}^{I-1} \mathbb{E} \left[ q_r^i \left\| \nabla f_i(\bvx_{r,k}) - \nabla f_i(\bvx_{r_0}) \right\|^2 \right] \\
    &\quad\quad + 3I^2 P^2 \mathbb{E} \left[ \left\| \sum_{i=1}^N \mathbb{E}_{r_0} \left[ \frac{1}{P} \sum_{r=r_0}^{r_0+P-1} q_r^i \right] (\nabla f_i(\bvx_{r_0}) - \nabla f(\bvx_{r_0})) \right\|^2 \right] \\
    &\quad \Eqmark{iv}{\leq} 3L^2 IP \sum_{r=r_0}^{r_0+P-1} \sum_{i=1}^N \sum_{k=0}^{I-1} \mathbb{E} \left[ q_r^i \left\| \vx_{r,k}^i - \bvx_{r,k} \right\|^2 \right] + 3L^2 IP \sum_{r=r_0}^{r_0+P-1} \sum_{i=1}^N \sum_{k=0}^{I-1} \mathbb{E} \left[ q_r^i \left\| \bvx_{r,k} - \bvx_{r_0} \right\|^2 \right] \\
    &\quad = 3L^2 IP \sum_{r=r_0}^{r_0+P-1} \sum_{k=0}^{I-1} \mathbb{E} \left[ \sum_{i=1}^N q_r^i \mathbb{E} \left[ \left\| \vx_{r,k}^i - \bvx_{r,k} \right\|^2 \middle| \gQ \right] \right] \\
    &\quad\quad + 3L^2 IP \sum_{r=r_0}^{r_0+P-1} \sum_{k=0}^{I-1} \mathbb{E} \left[ \sum_{i=1}^N q_r^i \mathbb{E} \left[ \left\| \bvx_{r,k} - \bvx_{r_0} \right\|^2 \middle| \gQ \right] \right] \\
    &\quad \leq 3L^2 IP \mathbb{E} \left[ \sum_{r=r_0}^{r_0+P-1} \sum_{k=0}^{I-1} (D_{r,k} + M_{r,k}) \right],
\end{align*}
where $(i)$ uses the decomposition
\begin{equation*}
    \nabla f_i(\vx_{r,k}^i) - \nabla f(\bvx_{r_0}) = (\nabla f_i(\vx_{r,k}^i) - \nabla f_i(\bvx_{r,k})) + (\nabla f_i(\bvx_{r,k}) - \nabla f_i(\bvx_{r_0})) + (\nabla f_i(\bvx_{r_0}) - \nabla f(\bvx_{r_0})),
\end{equation*}
$(ii)$ and $(iii)$ use Jensen's inequality, and $(iv)$ uses smoothness of each $f_i$
along with the condition $\mathbb{E}_{r_0} \left[\bar{q}_{r_0}^i\right] = \frac{1}{N}$.
Plugging into \Eqref{eq:descent_linear_inter} yields
\begin{align*}
    &- \gamma \eta \mathbb{E} \left[ \left\langle \nabla f(\bvx_{r_0}), \mathbb{E}_{r_0} \left[ \sum_{r=r_0}^{r_0+P-1} \sum_{i=1}^N q_r^i \sum_{k=0}^{I-1} \nabla f_i(\vx_{r,k}^i) \right] \right\rangle \right] \\
    &\quad \leq -\frac{\gamma \eta IP}{2} \mathbb{E} \left[ \left\| \nabla f(\bvx_{r_0}) \right\|^2 \right] + \frac{3}{2} \gamma \eta L^2 \mathbb{E} \left[ \sum_{r=r_0}^{r_0+P-1} \sum_{k=0}^{I-1} (D_{r,k} + M_{r,k}) \right] \\
    &\quad \Eqmark{i}{\leq} -\frac{\gamma \eta IP}{2} \mathbb{E} \left[ \left\| \nabla f(\bvx_{r_0}) \right\|^2 \right] + \frac{3}{2} \gamma \eta L^2 \bigg( 108 \eta^2 I^3 P^3 \mathbb{E} \left[ \left\| \nabla f(\bvx_{r_0}) \right\|^2 \right] \\
    &\quad\quad + \left( 75 \eta^2 I^2 P + 65 \eta^2 I^2 P^2 \rho^2 \right) \sigma^2 + 254 \eta^2 L^2 I^3 P^3 \frac{1}{N} \sum_{i=1}^N \mathbb{E} \left[ S_{r_0}^i \right] \bigg) \\
    &\quad \leq \gamma \eta IP \left( -\frac{1}{2} + 162 \eta^2 L^2 I^2 P^2 \right) \mathbb{E} \left[ \left\| \nabla f(\bvx_{r_0}) \right\|^2 \right] + \gamma \eta IP \left( 113 \eta^2 L^2 I + 98 \eta^2 L^2 IP \rho^2 \right) \sigma^2 \\
    &\quad\quad + 381 \gamma \eta^3 L^4 I^3 P^3 \frac{1}{N} \sum_{i=1}^N \mathbb{E} \left[ S_{r_0}^i \right],
\end{align*}
where $(i)$ uses \Eqref{eq:dm_ub_sum} from Lemma \ref{lem:dm_ub}.

Plugging into \Eqref{eq:descent_inter_1} and applying Lemma \ref{lem:quadratic_err}
yields
\begin{align}
    &\mathbb{E}[f(\bvx_{r_0+P}) - f(\bvx_{r_0})] \nonumber \\
    &\quad \leq \gamma \eta IP \left( -\frac{1}{2} + 162 \eta^2 L^2 I^2 P^2 \right) \mathbb{E} \left[ \left\| \nabla f(\bvx_{r_0}) \right\|^2 \right] + \gamma \eta IP \left( 113 \eta^2 L^2 I + 98 \eta^2 L^2 I P \rho^2 \right) \sigma^2 \nonumber \\
    &\quad\quad + 381 \gamma \eta^3 L^4 I^3 P^3 \frac{1}{N} \sum_{i=1}^N \mathbb{E} \left[ S_{r_0}^i \right] + \frac{1}{2} L \mathbb{E} \left[ \|\bvx_{r_0+P} - \bvx_{r_0}\|^2 \right] \nonumber \\
    &\quad \leq \gamma \eta IP \left( -\frac{1}{2} + 162 \eta^2 L^2 I^2 P^2 \right) \mathbb{E} \left[ \left\| \nabla f(\bvx_{r_0}) \right\|^2 \right] + \gamma \eta IP \left( 113 \eta^2 L^2 I + 98 \eta^2 L^2 I P \rho^2 \right) \sigma^2 \nonumber \\
    &\quad\quad + 381 \gamma \eta^3 L^4 I^3 P^3 \frac{1}{N} \sum_{i=1}^N \mathbb{E} \left[ S_{r_0}^i \right] + \frac{1}{2} L \bigg( 6 \gamma^2 \eta^2 I^2 P^2 \mathbb{E} \left[ \|\nabla f(\bvx_{r_0})\|^2 \right] \\
    &\quad\quad + \gamma \eta IP \left( 5 \gamma \eta \rho^2 + 7 \eta^2 L I \right) \sigma^2 + 11 \gamma^2 \eta^2 L^2 I^2 P^2 \frac{1}{N} \sum_{i=1}^N \mathbb{E} \left[ S_{r_0}^i \right] \bigg) \nonumber \\
    &\quad \leq \gamma \eta IP \left( -\frac{1}{2} + 162 \eta^2 L^2 I^2 P^2 + 3 \gamma \eta LIP \right) \mathbb{E} \left[ \left\| \nabla f(\bvx_{r_0}) \right\|^2 \right] \nonumber \\
    &\quad\quad + \gamma \eta IP \left( \frac{5}{2} \gamma \eta L \rho^2 + 117 \eta^2 L^2 I + 98 \eta^2 L^2 I P \rho^2 \right) \sigma^2 \nonumber \\
    &\quad\quad + \left( 381 \gamma \eta^3 L^4 I^3 P^3 + 6 \gamma^2 \eta^2 L^3 I^2 P^2 \right) \frac{1}{N} \sum_{i=1}^N \mathbb{E} \left[ S_{r_0}^i \right] \nonumber \\
    &\quad \Eqmark{i}{\leq} \gamma \eta IP \left( -\frac{1}{2} + 162 \eta^2 L^2 I^2 P^2 + 3 \gamma \eta LIP \right) \mathbb{E} \left[ \left\| \nabla f(\bvx_{r_0}) \right\|^2 \right] \nonumber \\
    &\quad\quad + \gamma \eta IP \left( \frac{5}{2} \gamma \eta L \rho^2 + 117 \eta^2 L^2 I + 98 \eta^2 L^2 I P \rho^2 \right) \sigma^2 + p_{\text{sample}} \gamma \eta L^2 IP \frac{1}{N} \sum_{i=1}^N \mathbb{E} \left[ S_{r_0}^i \right], \label{eq:descent_inter_3}
\end{align}
where $(i)$ uses $\eta \leq \frac{\sqrt{p_{\text{sample}}}}{60 LIP}$ and $\gamma \eta
\leq \frac{p_{\text{sample}}}{60 LIP}$.

Using Lemma \ref{lem:staleness_rec_ub},
\begin{align*}
    &2 \gamma \eta L^2 IP \frac{1}{N} \sum_{i=1}^N \mathbb{E} \left[ S_{r_0+P}^i \right] \\
    &\quad \leq 2 \gamma \eta L^2 IP \bigg( \left( 324 \eta^2 I^2 P^2 + \frac{20}{p_{\text{sample}}} \gamma^2 \eta^2 I^2 P^2 \right) \mathbb{E} \left[ \left\| \nabla f(\bvx_{r_0}) \right\|^2 \right] \\
    &\quad + \left( 226 \eta^2 I + 195 \eta^2 IP \rho^2 + \frac{17}{p_{\text{sample}}} \gamma^2 \eta^2 IP \rho^2 \right) \sigma^2 + \left( 1 - \frac{1}{2} p_{\text{sample}} \right) \frac{1}{N} \sum_{j=1}^N \mathbb{E} \left[ S_{r_0}^j \right] \bigg) \\
    &\quad \leq \gamma \eta IP \left( 648 \eta^2 L^2 I^2 P^2 + \frac{40}{p_{\text{sample}}} \gamma^2 \eta^2 L^2 I^2 P^2 \right) \mathbb{E} \left[ \left\| \nabla f(\bvx_{r_0}) \right\|^2 \right] \\
    &\quad + \gamma \eta IP \left( 552 \eta^2 L^2 I + 390 \eta^2 L^2 IP \rho^2 + \frac{34}{p_{\text{sample}}} \gamma^2 \eta^2 L^2 IP \rho^2 \right) \sigma^2 \\
    &\quad + 2 \gamma \eta L^2 IP \frac{1}{N} \sum_{j=1}^N \mathbb{E} \left[ S_{r_0}^j \right] - p_{\text{sample}} \gamma \eta L^2 IP \frac{1}{N} \sum_{j=1}^N \mathbb{E} \left[ S_{r_0}^j \right] \\
    &\quad \leq \gamma \eta IP \left( 648 \eta^2 L^2 I^2 P^2 + \frac{40}{p_{\text{sample}}} \gamma^2 \eta^2 L^2 I^2 P^2 \right) \mathbb{E} \left[ \left\| \nabla f(\bvx_{r_0}) \right\|^2 \right] \\
    &\quad + \gamma \eta IP \left( 552 \eta^2 L^2 I + 390 \eta^2 L^2 IP \rho^2 + \gamma \eta L \rho^2 \right) \sigma^2 \\
    &\quad + 2 \gamma \eta L^2 IP \frac{1}{N} \sum_{j=1}^N \mathbb{E} \left[ S_{r_0}^j \right] - p_{\text{sample}} \gamma \eta L^2 IP \frac{1}{N} \sum_{j=1}^N \mathbb{E} \left[ S_{r_0}^j \right],
\end{align*}
where the last inequality uses $\gamma \eta \leq \frac{p_{\text{sample}}}{60 LIP}$.
Adding this to \Eqref{eq:descent_inter_3} and rearranging,
\begin{align*}
    &\mathbb{E}[f(\bvx_{r_0+P})] + 2 \gamma \eta L^2 IP \frac{1}{N} \sum_{i=1}^N \mathbb{E} \left[ S_{r_0+P}^i \right] \\
    &\quad \leq \left( \mathbb{E}[f(\bvx_{r_0})] + 2 \gamma \eta L^2 IP \frac{1}{N} \sum_{i=1}^N \mathbb{E} \left[ S_{r_0}^i \right] \right) \\
    &\quad\quad \gamma \eta IP \left( -\frac{1}{2} + 810 \eta^2 L^2 I^2 P^2 + 3 \gamma \eta LIP + \frac{40}{p_{\text{sample}}} \eta^2 L^2 I^2 P^2 \right) \mathbb{E} \left[ \left\| \nabla f(\bvx_{r_0}) \right\|^2 \right] \\
    &\quad\quad + \gamma \eta IP \left( \frac{7}{2} \gamma \eta L \rho^2 + 669 \eta^2 L^2 I + 488 \eta^2 L^2 I P \rho^2 \right) \sigma^2 \\
    &\quad \leq \left( \mathbb{E}[f(\bvx_{r_0})] + 2 \gamma \eta L^2 IP \frac{1}{N} \sum_{i=1}^N \mathbb{E} \left[ S_{r_0}^i \right] \right) \\
    &\quad\quad -\frac{1}{5} \gamma \eta IP \mathbb{E} \left[ \left\| \nabla f(\bvx_{r_0}) \right\|^2 \right] + \gamma \eta IP \left( 4 \gamma \eta L \rho^2 + 669 \eta^2 L^2 I + 488 \eta^2 L^2 I P \rho^2 \right) \sigma^2,
\end{align*}
where we used $\gamma \eta \leq \frac{1}{60 LIP}$ and $\eta \leq \frac{1}{60 LIP}$.
Finally, we can average over $r_0 \in \{0, P, 2P, \ldots, R - P\}$ and rearrange to
obtain
\begin{align*}
    \frac{P}{R} \sum_{r_0 \in \{0, P, \ldots, R-P\}} \mathbb{E}[\|\nabla f(\bvx_{r_0})\|^2] &\leq 5 \frac{\mathbb{E}[f(\bvx_0) - f(\bvx_R)]}{\gamma \eta I R} + 10 \frac{L^2 P}{R} \frac{1}{N} \sum_{i=1}^N \mathbb{E} \left[ S_0^i - S_R^i \right] \\
    &\quad + \left( 20 \gamma \eta L \rho^2 + 3345 \eta^2 L^2 I + 2440 \eta^2 L^2 IP \rho^2 \right) \sigma^2 \\
    &\Eqmark{i}{\leq} \frac{5 \Delta}{\gamma \eta I R} + \left( 20 \gamma \eta L \rho^2 + 3345 \eta^2 L^2 I + 2440 \eta^2 L^2 IP \rho^2 \right) \sigma^2 \\
    &\Eqmark{ii}{\leq} \frac{5 \Delta}{\gamma \eta I R} + \left( 20 \gamma \eta L \rho^2 + 5785 \eta^2 L^2 IP \right) \sigma^2,
\end{align*}
where $(i)$ uses $\frac{1}{N} \sum_{i=1}^N \mathbb{E}[S_R^i] \geq 0$ and
$\frac{1}{N} \sum_{i=1}^N \mathbb{E}[S_0^i] = 0$ by initialization, and $(ii)$ uses $P \geq 1$ and $\rho \leq 1$.
\end{proof}

\begin{corollary}[Corollary \ref{cor:main_complexity} restated] \label{cor:complexity}
Let $\epsilon > 0$ and $I \geq 1$. Suppose Assumptions \ref{ass:obj} and
\ref{ass:sampling} hold, and that $\mathbb{E}[w_{r_0}^i | \gQ_{:r_0}] \leq
\frac{P^2}{N}$, and $\mathbb{E} \left[ \sum_{i=1}^N \left( v_{r_0}^i \right)^2
\Lambda_{r_0}^i \right] \leq \rho^2$. If
\begin{align*}
    R &\geq \frac{\Delta L \rho^2 \sigma^2}{I \epsilon^4} + \frac{300 \Delta LP}{p_{\text{sample}} \epsilon^2} \\
    \eta &\leq \frac{1}{264} \min \left\{ \frac{\Delta^{1/4} \rho^{1/2}}{L^{3/4} I^{3/4} R^{1/4} P^{1/2} \sigma^{1/2}}, \frac{\rho \sqrt{p_{\text{sample}}}}{LIP} \right\} \\
    \gamma &= \frac{1}{\eta} \min \left\{ \frac{\sqrt{\Delta}}{60 \sqrt{LIR} \rho \sigma}, \frac{p_{\text{sample}}}{60 LIP} \right\},
\end{align*}
then Algorithm 1 satisfies
\begin{equation*}
    \frac{P}{R} \sum_{r_0 \in \{0, P, \ldots, R-P\}} \mathbb{E}[\|\nabla f(\bvx_{r_0})\|^2] \leq 302 \epsilon^2.
\end{equation*}
\end{corollary}

\begin{proof}
First, $\eta \leq \frac{\rho \sqrt{p_{\text{sample}}}}{264 LIP} \leq
\frac{\sqrt{p_{\text{sample}}}}{60 LIP}$ and $\gamma \eta \leq
\frac{p_{\text{sample}}}{60 LIP}$ together with Assumptions \ref{ass:obj} and \ref{ass:sampling} imply that the conditions of Theorem \ref{thm:convergence} are satisfied.
Therefore
\begin{align*}
    \frac{P}{R} \sum_{r_0 \in \{0, P, \ldots, R-P\}} \mathbb{E}[\|\nabla f(\bvx_{r_0})\|^2] &\leq \frac{5 \Delta}{\gamma \eta I R} + \left( 20 \gamma \eta L \rho^2 + 5785 \eta^2 L^2 IP \right) \sigma^2.
\end{align*}
From our choice of $\eta$,
\begin{align*}
    5785 \eta^2 L^2 IP &\leq \frac{5785}{264^2} L^2 IP \min \left\{ \frac{\Delta^{1/2} \rho}{L^{3/2} I^{3/2} R^{1/2} P \sigma}, \frac{\rho^2 p_{\text{sample}}}{L^2 I^2 P^2} \right\} \\
    &\leq 5 L \rho^2 \min \left\{ \frac{\Delta^{1/2}}{60 L^{1/2} I^{1/2} R^{1/2} \rho \sigma}, \frac{p_{\text{sample}}}{60 LIP} \right\} \\
    &= 5 \gamma \eta L \rho^2.
\end{align*}
Therefore
\begin{align*}
    \frac{P}{R} \sum_{r_0 \in \{0, P, \ldots, R-P\}} \mathbb{E}[\|\nabla f(\bvx_{r_0})\|^2] &\leq \frac{5 \Delta}{\gamma \eta I R} + 25 \gamma \eta L \rho^2 \sigma^2 \\
    &\leq \frac{5 \Delta}{IR} \left( \frac{60 \sqrt{LIR} \rho \sigma}{\sqrt{\Delta}} + \frac{60 LIP}{p_{\text{sample}}} \right) + 25 L \rho^2 \sigma^2 \frac{\sqrt{\Delta}}{60 \sqrt{LIR} \rho \sigma} \\
    &\leq \frac{301 \sqrt{\Delta L} \rho \sigma}{\sqrt{IR}} + \frac{300 \Delta LP}{p_{\text{sample}} R},
\end{align*}
where we used our choice of $\gamma \eta$. Finally, from our choice of $R$,
\begin{align*}
    \frac{P}{R} \sum_{r_0 \in \{0, P, \ldots, R-P\}} \mathbb{E}[\|\nabla f(\bvx_{r_0})\|^2] &\leq \frac{301 \sqrt{\Delta L} \rho \sigma}{\sqrt{I}} \frac{\sqrt{I} \epsilon^2}{\sqrt{\Delta L} \rho \sigma} + \frac{300 \Delta LP}{p_{\text{sample}}} \frac{p_{\text{sample}} \epsilon^2}{300 \Delta LP} \\
    &\leq 302 \epsilon^2.
\end{align*}
\end{proof}

\section{Proofs for Specific Participation Patterns} \label{app:pattern_proofs}
Here we derive the convergence rates for the specific participation patterns discussed
in Section \ref{sec:pattern_results}. In order to apply Corollary
\ref{cor:main_complexity}, the conditions in Assumption \ref{ass:sampling},
\Eqref{eq:w_cond}, and \Eqref{eq:v_cond} must be satisfied.  Since we already showed
that Assumption \ref{ass:sampling} is satisfied by regularized participation and cyclic
participation (in Section \ref{sec:participation_patterns}), it only remains to show
that \Eqref{eq:w_cond} and \Eqref{eq:v_cond} is satisfied, and plug the appropriate
values of $\rho^2$, $P$, and $p_{\text{sample}}$ into Corollary
\ref{cor:main_complexity}.

\begin{corollary}[Regularized Participation]
Under regularized participation, for any $\epsilon > 0$, there exist choices of $\eta, \gamma, I$, and $R$ such that $\mathbb{E}[\|\nabla f(\hat{\vx})\|^2] \leq O(\epsilon^2)$ and
\begin{equation*}
    R = \mathcal{O} \left( \frac{LP}{\epsilon^2} \right), \quad RI = \mathcal{O} \left( \frac{\Delta L \rho^2 \sigma^2}{\epsilon^4} \right).
\end{equation*}
\end{corollary}

\begin{proof}
We first show that \Eqref{eq:w_cond} and \Eqref{eq:v_cond} are satisfied under
regularized participation. Let $r_0 \in \{0, \ldots, R-P\}$. Using the fact that
$\bar{q}_{r_0}^i = \frac{1}{N}$ almost surely:
\begin{align*}
    \mathbb{E} \left[ w_{r_0}^i \middle| \gQ_{:r_0} \right] &= \mathbb{E} \left[ \frac{1}{N} \sum_{j=1}^N \frac{\indicator{\bar{q}_{r_0}^j > 0}}{P \bar{q}_{r_0}^j} \sum_{s=r_0}^{r_0+P-1} q_s^i q_s^j \middle| \gQ_{:r_0} \right] = \frac{1}{P} \mathbb{E} \left[ \sum_{j=1}^N \sum_{s=r_0}^{r_0+P-1} q_s^i q_s^j \middle| \gQ_{:r_0} \right] \\
    &= \frac{1}{P} \mathbb{E} \left[ \sum_{s=r_0}^{r_0+P-1} q_s^i \sum_{j=1}^N q_s^j \middle| \gQ_{:r_0} \right] = \frac{1}{P} \mathbb{E} \left[ \sum_{s=r_0}^{r_0+P-1} q_s^i \middle| \gQ_{:r_0} \right] \\
    &= \mathbb{E} \left[ q_{r_0}^i \middle| \gQ_{:r_0} \right] = \frac{1}{N} \leq \frac{P^2}{N},
\end{align*}
where we used the fact that $\sum_{i=1}^N q_r^i = 1$. This shows that \Eqref{eq:w_cond} is satisfied.
Also, $\bar{q}_{r_0}^i = \frac{1}{N}$ means that $v_{r_0}^i = 0$, so that $\mathbb{E} \left[ \left( v_{r_0}^i \right)^2 \Lambda_{r_0}^i \right] = 0 \leq \rho^2$, and \Eqref{eq:v_cond} is satisfied.
Therefore we can apply Corollary \ref{cor:main_complexity}. The complexity result follows by plugging in $p_{\text{sample}} = 1$ and choosing $I = \mathcal{O} \left( \frac{\Delta \rho^2 \sigma^2}{P \epsilon^2} \right)$.
\end{proof}

\begin{corollary}[Cyclic Participation]
Under cyclic participation with $\bar{K}$ groups and $S$ participating clients in each round, for any $\epsilon > 0$, there exist choices of $\eta, \gamma, P, I$, and $R$ such that $\mathbb{E}[\|\nabla f(\hat{\vx})\|^2] \leq O(\epsilon^2)$ and
\begin{equation*}
    R = \mathcal{O} \left( \frac{L \bar{K}}{\epsilon^2} \left( \frac{N}{S} \right) \right), \quad RI = \mathcal{O} \left( \frac{\Delta L \sigma^2}{S \epsilon^4} \right).
\end{equation*}
\end{corollary}

\begin{proof}
Again, we show that \Eqref{eq:w_cond} and \Eqref{eq:v_cond} are satisfied under cyclic participation. Let $r_0 \in \{0, \ldots, R-P\}$.
Denoting $g(i) = \floor{\frac{\bar{K}}{N} (i-1)}$ and $r(i) = r_0 + g(i)$, it holds that,
over the rounds $r \in \{r_0, \ldots, r_0 + P - 1\}$, client $i$ belongs to group $g(i)$ and is available only
during round $r(i)$. Also, let $\Pi(i)$ denote the set of indices of clients in group $g(i)$.
To simplify $w_{r_0}^i$, recall that for client $i$ in group $(r \mod \bar{K})$:
\begin{equation*}
    q_r^i = \begin{cases}
        \frac{1}{S} & \text{with probability } \frac{S}{N} \\
        0 & \text{with probability } 1 - \frac{S}{N}
    \end{cases},
\end{equation*}
and $q_r^i = 0$ for all clients $i$ not in group $(r \mod \bar{K})$.
So
\begin{align*}
    \mathbb{E} \left[ w_{r_0}^i \middle| \gQ_{:r_0} \right] &= \mathbb{E} \left[ \frac{1}{N} \sum_{j=1}^N \frac{\indicator{\bar{q}_{r_0}^j > 0}}{P \bar{q}_{r_0}^j} \sum_{s=r_0}^{r_0+P-1} q_s^i q_s^j \middle| \gQ_{:r_0} \right] \\
    &= \frac{1}{NP} \mathbb{E} \left[ \sum_{s=r_0}^{r_0+P-1} q_s^i \sum_{j=1}^N \frac{\indicator{\bar{q}_{r_0}^j > 0}}{\bar{q}_{r_0}^j} q_s^j \middle| \gQ_{:r_0} \right] \\
    &\Eqmark{i}{=} \frac{1}{NP} \mathbb{E} \left[ q_{r(i)}^i \sum_{j=1}^N \frac{\indicator{\bar{q}_{r_0}^j > 0}}{\bar{q}_{r_0}^j} q_{r(i)}^j \middle| \gQ_{:r_0} \right] \\
    &\Eqmark{ii}{=} \frac{1}{NP} \mathbb{E} \left[ q_{r(i)}^i \sum_{j \in \Pi(i)} \frac{\indicator{\bar{q}_{r_0}^j > 0}}{\bar{q}_{r_0}^j} q_{r(i)}^j \middle| \gQ_{:r_0} \right] \\
    &\Eqmark{iii}{=} \frac{1}{NP} \mathbb{E} \left[ q_{r(i)}^i \sum_{j \in \Pi(i)} \frac{\indicator{\bar{q}_{r_0}^j > 0}}{1/SP} \frac{1}{S} \middle| \gQ_{:r_0} \right] \\
    &= \frac{1}{N} \mathbb{E} \left[ q_{r(i)}^i \sum_{j \in \Pi(i)} \indicator{\bar{q}_{r_0}^j > 0} \middle| \gQ_{:r_0} \right] \\
    &\Eqmark{iv}{=} \frac{S}{N} \mathbb{E} \left[ q_{r(i)}^i \middle| \gQ_{:r_0} \right] \\
    &= \frac{S}{N} \left( \frac{S}{N} \frac{1}{S} + \left( 1 - \frac{S}{N} \right) \cdot 0 \right) \\
    &= \frac{S}{N^2} \leq \frac{P^2}{N},
\end{align*}
where $(i)$ uses $q_r^i = 0$ for all $r \neq r(i)$, $(ii)$ uses $q_{r(i)}^j = 0$ for all
$j \notin \Pi(i)$, $(iii)$ uses the fact that $\bar{q}_{r_0}^j > 0$ implies
$q_{r(i)}^j = \frac{1}{S}$ and $\bar{q}_{r_0}^j = \frac{1}{SP}$, and $(iv)$ uses the
fact that $S$ clients are sampled in each round. This shows that \Eqref{eq:w_cond} is satisfied.

To see that \Eqref{eq:v_cond} is satisfied:
\begin{align*}
    \mathbb{E} \left[ \sum_{i=1}^N \left( v_{r_0}^i \right)^2 \Lambda_{r_0}^i \right] &\Eqmark{i}{=} \mathbb{E} \left[ \sum_{i=1}^N \mathbb{E} \left[ \left( \bar{q}_{r_0}^i - \frac{1}{N} \right)^2 \middle| \gQ_{:r_0} \right] \Lambda_{r_0}^i \right] \\
    &\Eqmark{ii}{=} \frac{1}{SNP} \left( 1 - \frac{S \bar{K}}{N} \right) \sum_{i=1}^N \mathbb{E} \left[ \Lambda_{r_0}^i \right] \\
    &= \frac{1}{SNP} \left( 1 - \frac{S \bar{K}}{N} \right) \sum_{i=1}^N \left( \frac{1}{P} \sum_{s=r_0-P}^{r_0-1} \left( z_s^i \right)^2 \bigg/ \left( \frac{1}{P} \sum_{s=r_0-P}^{r_0-1} z_s^i \right)^2 \right) \\
    &\Eqmark{iii}{=} \frac{1}{SP} \left( 1 - \frac{S \bar{K}}{N} \right) \frac{1}{P} \frac{1}{S^2} \bigg/ \left( \frac{1}{P} \frac{1}{S} \right)^2 \\
    &= \frac{1}{S} \left( 1 - \frac{S \bar{K}}{N} \right) \leq \frac{1}{S} = \rho^2,
\end{align*}
where $(i)$ uses the tower property, $(ii)$ uses the variance of $\bar{q}_{r_0}^i$ as computed in Appendix \ref{app:baseline_complexity}, and $(iii)$ uses the fact that $\bar{z}_{r_0}^i > 0$ by construction of $\bar{z}_{r_0}^i$, and in this case $z_r^i = \frac{1}{S}$ for exactly one $r \in \{r_0 - P, \ldots, r_0 - 1\}$ and $z_r^i = 0$ for all other $r$. Therefore \Eqref{eq:v_cond} is satisfied.

This shows we can apply Corollary \ref{cor:main_complexity}. The complexity result follows by plugging in
$p_{\text{sample}} = \frac{S}{N}$, $P = \bar{K}$, $\rho = \frac{1}{\sqrt{S}}$, and choosing $I = \mathcal{O} \left( \frac{\Delta \sigma^2}{\bar{K} N \epsilon^2} \right)$.
\end{proof}

\section{Complexity of Amplified FedAvg} \label{app:baseline_complexity}
In this section we derive the computation and communication complexity required for Amplified FedAvg \cite{wang2022unified} to find an $\epsilon$-stationary point under various client participation patterns, which we listed in Table \ref{tab:complexity} of the main paper. Table \ref{tab:complexity} compares the complexity under i.i.d. participation, regularized participation, and cyclic participation.
Since i.i.d. participation is a special case of cyclic participation with $\bar{K}=1$ groups, here we only consider regularized and cyclic participation, and the result for i.i.d. participation follows.

Many works in federated learning characterize data heterogeneity by assuming that there exists a constant $\kappa$ such that
\begin{equation*}
    \|\nabla f_i(\vx) - \nabla f(\vx)\| \leq \kappa,
\end{equation*}
for all $\vx$.
The previous analysis of Amplified FedAvg \cite{wang2022unified} instead assumes an upper bound $\tilde{\delta}(P)$ on a weighted heterogeneity that depends on client sampling. Specifically, they assume that there exists $\tilde{\delta}(P)$ such that
\begin{equation*}
    \left\| \frac{1}{P} \sum_{r=r_0}^{r_0+P-1} \sum_{i=1}^N q_r^i (\nabla f_i(\vx) - \nabla f(\vx)) \right\|^2 \leq \tilde{\delta}^2(P),
\end{equation*}
for all $\vx$ and $r_0$. We restate their Corollary 3.2 for convenience:
\begin{corollary}[Corollary 3.2 \cite{wang2022unified} informally restated]
There exist parameter choices such that Amplified FedAvg satisfies
\begin{equation} \label{eq:amp_fedavg_conv}
    \min_r \mathbb{E} \left[ \|\nabla f(\bvx_r)\|^2 \right] \leq \mathcal{O} \left( \frac{\sqrt{\Delta L} \rho \sigma}{\sqrt{RI}} + \frac{\Delta LP + \kappa^2}{R} + \frac{\sigma^2}{RIP} + \mathbb{E} \left[ \tilde{\delta}^2(P) \right] \right).
\end{equation}
\end{corollary}

As pointed out in their Section 4.2, we can interpret $\tilde{\delta}(P)$ in terms of the conventional heterogeneity constant $\kappa$ as:
\begin{equation*}
    \tilde{\delta}^2(P) \leq N \kappa^2 \sum_{i=1}^N \left( \bar{q}_{r_0}^i - \frac{1}{N} \right)^2.
\end{equation*}
Our Assumption \ref{ass:sampling}(b) implies that $\mathbb{E}_{\gQ_{r_0}}[\bar{q}_{r_0}^i] = \frac{1}{N}$. If we choose $v \geq 0$ such that $\text{Var}[\bar{q}_{r_0}^i] \leq v^2$, then we can take expecation of the above to obtain
\begin{equation*}
    \mathbb{E} \left[ \tilde{\delta}^2(P) \right] \leq N^2 \kappa^2 v^2.
\end{equation*}
This ``conversion" of $\tilde{\delta}(P)$ to $\kappa$ and $v$ will allow us to compare their complexity to that of algorithms that use the conventional heterogeneity assumption by computing $v$ for each participation pattern.

\paragraph{Regularized Participation} In this case, $\bar{q}_{r_0}^i = \frac{1}{N}$ almost surely, so that $v=0$ and accordingly $\tilde{\delta}^2(P) = 0$. Therefore
\begin{equation*}
    \min_r \mathbb{E} \left[ \|\nabla f(\bvx_r)\|^2 \right] \leq \mathcal{O} \left( \frac{\sqrt{\Delta L} \rho \sigma}{\sqrt{RI}} + \frac{\Delta LP + \kappa^2}{R} + \frac{\sigma^2}{RIP} \right).
\end{equation*}
Therefore the choices
\begin{equation*}
    R \geq \mathcal{O} \left( \frac{\Delta LP + \kappa^2}{\epsilon^2} \right), \quad RI \geq \mathcal{O} \left( \frac{\Delta L \rho^2 \sigma^2}{\epsilon^4} + \frac{\sigma^2}{P \epsilon^2} \right),
\end{equation*}
imply
\begin{equation*}
    \min_r \mathbb{E} \left[ \|\nabla f(\bvx_r)\|^2 \right] \leq \mathcal{O} \left( \epsilon^2 \right).
\end{equation*}

\paragraph{Cyclic Participation} We can compute $v$ in terms of the parameters of the participation pattern: number of groups $\bar{K}$ and number of participating clients $S$ in each round.
Although we chose $P = \bar{K}$ for Amplified SCAFFOLD, we can satisfy Assumption(b) \ref{ass:sampling} by choosing $P = m\bar{K}$ for any $m \in \mathbb{N}$, and indeed to achieve $\mathbb{E}[\tilde{\delta}^2(P)] \leq \epsilon^2$ we must choose $P = m\bar{K}$ with $m$ depending on $\epsilon$.

Let $A(i)$ denote the set of rounds in $\{r_0, \ldots, r_0+P-1\}$ during which client $i$ is available. Note that $A(i)$ has size $m$, since $P = m\bar{K}$. Then
\begin{equation*}
    \bar{q}_{r_0}^i - \frac{1}{N} = \frac{1}{P} \sum_{r=r_0}^{r_0+P-1} q_r^i - \frac{1}{N} \Eqmark{i}{=} \frac{1}{P} \sum_{r \in A(i)} q_r^i - \frac{1}{N} = \frac{1}{m\bar{K}} \sum_{r \in A(i)} q_r^i - \frac{1}{b \bar{K}} = \frac{1}{\bar{K}} \left( \frac{1}{m} \sum_{r \in A(i)} q_r^i - \frac{1}{b} \right),
\end{equation*}
where $(i)$ uses the fact that $q_r^i = 0$ for all $r \notin A(i)$. Therefore
\begin{align*}
    v^2 &= \frac{1}{\bar{K}^2} \mathbb{E} \left[ \left( \frac{1}{m} \sum_{r \in A(i)} q_r^i - \frac{1}{b} \right)^2 \right] \\
    &\Eqmark{i}{=} \frac{1}{\bar{K}^2} \frac{1}{m^2} \sum_{r \in A(i)} \mathbb{E} \left[ \left( q_r^i - \frac{1}{b} \right)^2 \right] \\
    &\Eqmark{ii}{=} \frac{1}{\bar{K}^2} \frac{1}{m^2} \sum_{r \in A(i)} \frac{1}{S^2} \frac{S}{b} \left( 1 - \frac{S}{b} \right) \\
    &= \frac{1}{S m\bar{K}^2 b} \left( 1 - \frac{S}{b} \right) \\
    &= \frac{1}{SNP} \left( 1 - \frac{S \bar{K}}{N} \right),
\end{align*}
where $(i)$ uses the fact that $\{q_r^i\}_{r \in A(i)}$ are independent and $(ii)$ uses the fact that $q_r^i$ equals $\frac{1}{S}$ times a Bernoulli variable for $r \in A(i)$.

With a bound for $v^2$, we can bound the remaining term in \Eqref{eq:amp_fedavg_conv} as follows:
\begin{equation*}
    \mathbb{E}[\tilde{\delta}^2(P)] \leq N^2 \kappa^2 v^2 \leq \frac{\kappa^2}{P} \frac{N}{S} \left( 1 - \frac{S\bar{K}}{N} \right).
\end{equation*}
Therefore, Amplified FedAvg under cyclic participation can find an $\epsilon$-stationary point with the choices
\begin{align*}
    P &\geq \max \left\{ \bar{K}, \frac{\kappa^2}{\epsilon^2} \frac{N}{S} \left( 1 - \frac{S\bar{K}}{N} \right) \right\} \\
    R &\geq \mathcal{O} \left( \frac{\Delta L \bar{K} + \kappa^2}{\epsilon^2} + \frac{\Delta LN \kappa^2}{S \epsilon^4} \left( 1 - \frac{S\bar{K}}{N} \right) \right) \\
    RI &\geq  \mathcal{O} \left( \frac{\Delta L \sigma^2}{S \epsilon^4} \right).
\end{align*}

\section{Experiment Details} \label{app:experiment_details}
Here we discuss experimental details deferred from the main body: client sampling
parameters, heterogeneity protocol, hyperparameter tuning, definition of the synthetic
objective, and specification of the CNN architecture used for the image classification
experiments.

\subsection{Client Sampling Parameters}
For the synthetic objective, we set the number of groups $\bar{K}=2$ and the
availability time $g=240$. We set the communication interval $I=10$ and train for
$R=5000$ rounds.

For Fashion-MNIST, we set the communication interval $I=30$ and train a logistic
regression model for $R=2000$ rounds, with availability time $g=4$.

For CIFAR-10, we set the communication interval $I=5$ and train a two-layer CNN for
$R=12000$ rounds, with availability time $g=10$.

\subsection{Heterogeneity Protocol}
The following protocol is commonly used in the literature \cite{karimireddy2020scaffold,
wang2022unified} to convert a dataset into a collection of heterogeneous local datasets
according to a data similarity parameter $s$, where $s=0\%$ creates maximal data
heterogeneity across clients, and $s=100\%$ means that data is allocated to each client
uniformly at random.

A single, non-federated dataset (e.g. CIFAR-10) is partitioned into two subsets: $s\%$
of the samples are allocated to an i.i.d. pool and randomly shuffled, and the remaining
$(100-s)\%$ of the sampled are allocated to a non-i.i.d. pool and are sorted by label.
Samples are allocated to each client so that $s\%$ of each local dataset comes from the
i.i.d. pool, and the remaining $(100-s)\%$ comes from the non-i.i.d. pool, so that with
a small $s$, the majority of each local dataset consists of a small number of labels.

\subsection{Hyperparameter Tuning}
For all three experiments in the main body (synthetic objective, Fashion-MNIST, and
CIFAR-10), each of the four baselines are individually tuned with grid search. For
algorithms that use amplified updates (Amplified FedAvg and Amplified SCAFFOLD), we tune
the amplification rate $\gamma$ by searching over a fixed range of values. For the other
algorithms (FedAvg and SCAFFOLD), we indicate the lack of amplified updates by setting
$\gamma = 1$. We tune $\eta$ by allowing $\gamma \eta$ over a fixed range of values. For
FedProx's $\mu$ parameter, we also search over a fixed range of values. The search range
and final values for each parameter are written in Table \ref{tab:hyperparams}, along
with the final values adopted for each algorithm.

\begin{table*}[t]
\caption{Hyperparameter search ranges and final values.}
\label{tab:hyperparams}
\begin{center}
\resizebox{\textwidth}{!}{
\begin{tabular}{@{}lccl@{}}
\toprule
& $\gamma$ values & $\gamma \eta$ values & $\mu$ values \\
\midrule
Synthetic & & & \\
\quad\quad FedAvg & \{$\mathbf{1}$\} & $\{10^{-6}, \mathbf{10^{-5}}, 10^{-4}, 10^{-3}\}$ & $\{\mathbf{0}\}$ \\
\quad\quad FedProx & \{$\mathbf{1}$\} & $\{10^{-6}, \mathbf{10^{-5}}, 10^{-4}, 10^{-3}\}$ & $\{\mathbf{0.01}, 0.1, 1.0, 10.0\}$ \\
\quad\quad SCAFFOLD & \{$\mathbf{1}$\} & $\{10^{-6}, 10^{-5}, \mathbf{10^{-4}}, 10^{-3}\}$ & $\{\mathbf{0}\}$ \\
\quad\quad Amplified FedAvg & $\{1.25, 1.5, 2, \mathbf{3}\}$ & $\{10^{-6}, \mathbf{10^{-5}}, 10^{-4}, 10^{-3}\}$ & $\{\mathbf{0}\}$ \\
\quad\quad Amplified SCAFFOLD & $\{1.25, \mathbf{1.5}, 2, 3\}$ & $\{10^{-6}, 10^{-5}, \mathbf{10^{-4}}, 10^{-3}\}$ & $\{\mathbf{0}\}$ \\
\midrule
Fashion-MNIST & & & \\
\quad\quad FedAvg & \{$\mathbf{1}$\} & $\{10^{-5}, \mathbf{10^{-4}}, 10^{-3}, 10^{-2}\}$ & $\{\mathbf{0}\}$ \\
\quad\quad FedProx & \{$\mathbf{1}$\} & $\{10^{-5}, \mathbf{10^{-4}}, 10^{-3}, 10^{-2}\}$ & $\{0.01, 0.1, 1.0, \mathbf{10.0}\}$ \\
\quad\quad SCAFFOLD & \{$\mathbf{1}$\} & $\{10^{-5}, 10^{-4}, 10^{-3}, \mathbf{10^{-2}}\}$ & $\{\mathbf{0}\}$ \\
\quad\quad Amplified FedAvg & $\{1.25, 1.5, \mathbf{2}, 3\}$ & $\{10^{-5}, \mathbf{10^{-4}}, 10^{-3}, 10^{-2}\}$ & $\{\mathbf{0}\}$ \\
\quad\quad Amplified SCAFFOLD & $\{1.25, \mathbf{1.5}, 2, 3\}$ & $\{10^{-5}, 10^{-4}, 10^{-3}, \mathbf{10^{-2}}\}$ & $\{\mathbf{0}\}$ \\
\midrule
CIFAR-10 & & & \\
\quad\quad FedAvg & \{$\mathbf{1}$\} & $\{10^{-6}, \mathbf{10^{-5}}, 10^{-4}, 10^{-3}\}$ & $\{\mathbf{0}\}$ \\
\quad\quad FedProx & \{$\mathbf{1}$\} & $\{10^{-6}, \mathbf{10^{-5}}, 10^{-4}, 10^{-3}\}$ & $\{\mathbf{0.01}, 0.1, 1.0, 10.0\}$ \\
\quad\quad SCAFFOLD & \{$\mathbf{1}$\} & $\{10^{-6}, 10^{-5}, \mathbf{10^{-4}}, 10^{-3}\}$ & $\{\mathbf{0}\}$ \\
\quad\quad Amplified FedAvg & $\{1.25, 1.5, 2, \mathbf{3}\}$ & $\{10^{-6}, 10^{-5}, \mathbf{10^{-4}}, 10^{-3}\}$ & $\{\mathbf{0}\}$ \\
\quad\quad Amplified SCAFFOLD & $\{1.25, \mathbf{1.5}, 2, 3\}$ & $\{10^{-6}, 10^{-5}, 10^{-4}, \mathbf{10^{-3}}\}$ & $\{\mathbf{0}\}$ \\
\bottomrule
\end{tabular}}
\end{center}
\end{table*}

\subsection{Synthetic Objective}
For the synthetic experiment, we use a difficult objective from a lower bound analysis
of FedAvg \cite{woodworth2020minibatch}. As defined in the lower bound analysis, the
objective is parameterized by $H, \kappa, \sigma, c, \mu$, and $L$. The objective maps
$\mathbb{R}^4$ to $\mathbb{R}$, and there are only two clients with corresponding local
objectives
\begin{align*}
    f_1(\vx) &= \frac{\mu}{2} \left( x_1 - c \right)^2 + \frac{H}{2} \left( x_2 - \frac{\sqrt{\mu} c}{\sqrt{H}} \right)^2 + \frac{H}{8} \left( x_3^2 + [x_3]_+^2 \right) + \kappa x_4 \\
    f_2(\vx) &= \frac{\mu}{2} \left( x_1 - c \right)^2 + \frac{H}{2} \left( x_2 - \frac{\sqrt{\mu} c}{\sqrt{H}} \right)^2 + \frac{H}{8} \left( x_3^2 + [x_3]_+^2 \right) - \kappa x_4,
\end{align*}
where $[x]_+ := \max \left\{ x, 0 \right\}$. The stochastic gradients for $f_1$ and
$f_2$ are sampled from the distributions $\nabla f_1(\vx) + \xi e_3$ and $\nabla
f_2(\vx) + \xi e_3$, respectively, where $\xi \sim \mathcal{N}(0, \sigma^2)$ and $e_3$
denotes the third standard basis vector in $\mathbb{R}^4$. We set the parameters of the
objective as follows:
\begin{align*}
    &H = 16, \quad \kappa = 16, \quad \sigma = 1 \\
    &c = 1, \quad \mu = 2, \quad L = 2.
\end{align*}

\subsection{CNN Architecture}
We use a simple 2-layer CNN for CIFAR-10. The first layer is a convolutional layer with
64 channels, a $5\times5$ kernel, stride of 2, padding of 2, and a ReLU activation. The
second layer is a fully connected layer with no activation.

\section{Additional Experimental Results} \label{app:extra_experiments} In this section,
we provide two additional experimental results. First, in Section
\ref{app:extra_baselines}, we add four baselines to the experimental settings of the
main paper: FedAdam \cite{reddi2020adaptive}, FedYogi \cite{reddi2020adaptive}, FedAvg-M
\cite{cheng2023momentum}, and Amplified FedAvg with FedProx regularization. Second, in
Section \ref{app:sca}, we evaluate all nine algorithms (five from the main paper and the
four additional baselines) under another non-i.i.d. client participation pattern for the
CIFAR-10 dataset, which we refer to as Stochastic Client Availability (SCA).

\subsection{Additional Baselines} \label{app:extra_baselines}
We evaluate the four additional baselines for the Fashion-MNIST and CIFAR-10 experiments
from Section \ref{sec:experiments}, keeping the same experimental setup. We tuned the
hyperparameters of all baselines according to the hyperparameter ranges suggested in the
original paper of each algorithm, and we allow the same compute budget for tuning each
baseline as we did for tuning the algorithms in the original paper, in terms of the
total number of hyperparameter combinations evaluated. Also, the results are averaged
over five random seeds. The results are shown in Figures \ref{fig:fashion_all} and
\ref{fig:cifar_all}.

\begin{figure}
\begin{center}
\includegraphics[width=0.8\linewidth]{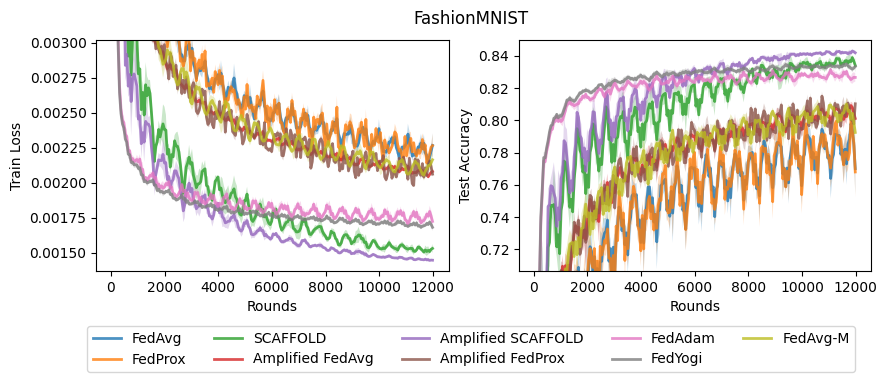}
\caption{FashionMNIST with additional baselines. Amplified SCAFFOLD maintains the best performance.}
\label{fig:fashion_all}
\end{center}
\end{figure}

\begin{figure}
\begin{center}
\includegraphics[width=0.8\linewidth]{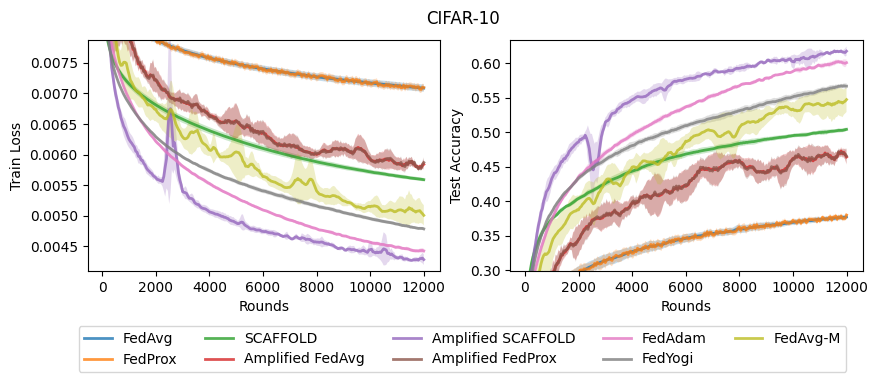}
\caption{CIFAR-10 with additional baselines. FedAdam is competitive, but Amplified SCAFFOLD maintains superiority.}
\label{fig:cifar_all}
\end{center}
\end{figure}

For FashionMNIST, FedAdam and FedYogi reach moderate training loss quickly, but
are soon overtaken by Amplified SCAFFOLD and later by SCAFFOLD. FedAvg-M
exhibits a minor advantage over FedAvg, but performs about the same as Amplified
FedAvg. Amplified FedProx (i.e. Amplified FedAvg with FedProx regularization)
performs nearly identically to Amplified FedAvg.

For CIFAR-10, FedAdam is more competitive, but is still outperformed by
Amplified SCAFFOLD. FedYogi and FedAvg-M are further behind, though both still
outperform SCAFFOLD. Amplified FedProx is again nearly identical to Amplified
FedAvg.

These additional comparisons demonstrate that Amplified SCAFFOLD outperforms
strong empirical baselines (FedAdam, FedYogi) under cyclic client participation,
reinforcing the empirical validation of our algorithm. This performance is
consistent with the fact that Amplified SCAFFOLD has convergence guarantees
under periodic participation, while FedAdam and FedYogi were not designed for
settings beyond i.i.d. client sampling.

\subsection{CIFAR-10 with Stochastic Client Availability} \label{app:sca}
Here, we include an evaluation under another non-i.i.d. participation pattern,
which we refer to as Stochastic Cyclic Availability (SCA). SCA models device
availability which is both periodic and unreliable. Similarly to cyclic participation,
the set of clients is divided into groups, and at each round one group is deemed the
"active" group, while the others are inactive. Unlike cyclic participation, in SCA not
every client in the active group is always available: Instead, when a group becomes
active, the clients in that group become available for sampling with probability $80\%$, while
clients in inactive groups have probability $5\%$ to be available for participation. The
active group changes every $g$ rounds. This stochastic availability models the real-life
situation where a client device can be unavailable at a time of day when it is usually
available, or vice versa. In this way, SCA is more flexible than cyclic participation
and better captures the unreliability of client devices. Lastly, we reused the remaining
settings ($g$, $\bar{K}$, $I$, etc.) and the tuned hyperparameters for each baseline
from the CIFAR-10 experiment under cyclic participation. Again, we average each
algorithm's performance over five random seeds.

Results for CIFAR-10 under SCA participation are shown in Figure \ref{fig:cifar_ra_all}.
Again, Amplified SCAFFOLD outperforms all baselines under SCA participation. The
relative performance of each baseline is similar as under cyclic participation, with
FedAdam staying competitive with Amplified SCAFFOLD, followed by FedYogi and FedAvg-M.
The remainder of the baselines have significantly worse performance, and again Amplified
FedAvg has not benefitted by adding FedProx regularization.

\begin{figure}
\begin{center}
\includegraphics[width=0.8\linewidth]{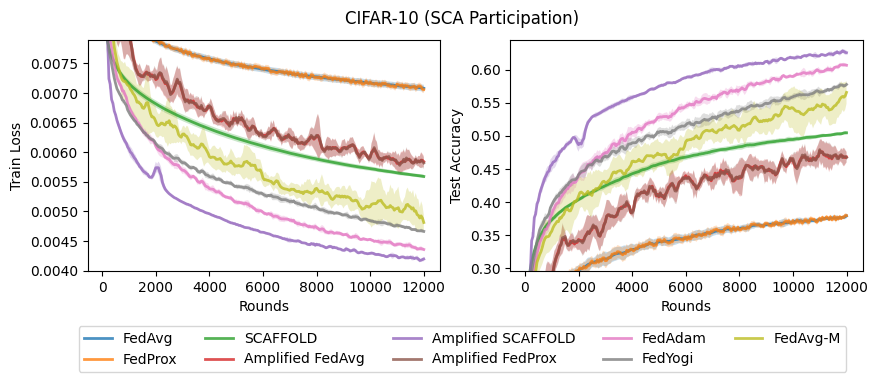}
\caption{CIFAR-10 under SCA (stochastic cyclic availability). Amplified SCAFFOLD converges fastest.}
\label{fig:cifar_ra_all}
\end{center}
\end{figure}

\section{Extended Comparison with Baselines} \label{app:comparison}
Here we include an extended comparison against two relevant prior works.

\paragraph{FedAvg with Cyclic Participation}
\cite{cho2023convergence} analyzes FedAvg under cyclic participation, but the
resulting convergence rate does not benefit from local steps unless regularized
participation is satisfied. They analyze FedAvg for $L$-smooth and $\mu$-PL
objectives. Using their notation, $\bar{K}$ is the number of client groups,
$\kappa$ is the condition number of the objective, $\gamma$ is the intra-group
heterogeneity, $M$ is the total number of clients, $N$ is the number of clients
that participate in each round, and $T$ is the number of communication rounds.
Then the dominating term in their convergence rate for FedAvg under cyclic
participation (Theorem 2) is \[
    \tilde{\gO} \left( \frac{\bar{K} \kappa \gamma^2}{\mu NT} \left( \frac{M/\bar{K}-N}{M/\bar{K}-1} \right) \right).
\] Notice that the number of local steps (denoted $\tau$) does not appear in
this dominating term, so there is no way to reduce the communication complexity
(compared to parallel SGD) by taking local steps. The only exception is when
this term is zero from $N = M/\bar{K}$, which is equivalent to the condition
that every client participates within every cycle of availability. Therefore,
this result cannot show a benefit from local steps unless the client
participation is regularized.

\paragraph{SCAFFOLD} In Section \ref{sec:pattern_results}, we mentioned a
discrepancy between the complexity of SCAFFOLD vs. Amplified SCAFFOLD in terms
of the dependence on $N/S$. In Table \ref{tab:complexity}, the communication
complexity of SCAFFOLD and Amplified SCAFFOLD under i.i.d. participation differs
by their dependence on $\frac{N}{S}$. The complexity of Amplified SCAFFOLD is
$\mathcal{O} \left( \frac{N}{S} \right)$, while that of SCAFFOLD is $\mathcal{O}
\left( \left( \frac{N}{S} \right)^{2/3} \right)$. This difference in the order
of $\frac{N}{S}$ is due to a potential small issue in the analysis of SCAFFOLD,
which we intentionally avoided by accepting a slightly worse dependence on
$\frac{N}{S}$.

This difference stems from an apparent mistake in the original SCAFFOLD
analysis. In the proof of Lemma 16 (PMLR version of SCAFFOLD), the
second-to-last equation of page 32 is obtained with an incorrect step. Namely,
while the current version includes
\begin{align}
    \Xi_r &= \frac{1}{KN} \sum_{i,k} \mathbb{E} \|\mathbf{\alpha}_{i,k-1}^r - \mathbf{x}^r\|^2 \nonumber \\
    &= \left( 1 - \frac{S}{N} \right) \frac{1}{KN} \sum_{i,k} \mathbb{E} \|\mathbf{\alpha}_{i,k-1}^{r-1} - \mathbf{x}^r\|^2 + \frac{S}{N} \frac{1}{KN} \sum_{i,k} \mathbb{E} \|\mathbf{y}_{i,k-1}^r - \mathbf{x}^r\|^2, \label{eq:scaffold_wrong}
\end{align}
where the last line is obtained by conditioning on the event $\alpha_{i,k-1}^r =
\alpha_{i,k-1}^{r-1}$ and the complement $\alpha_{i,k-1}^{r-1} = y_{i,k-1}^r$,
which have probabilities $1-\frac{S}{N}$ and $\frac{S}{N}$, respectively.
However, this condition is not denoted in \Eqref{eq:scaffold_wrong}. A corrected
version should be written
\begin{align*}
    \Xi_r &= \left( 1 - \frac{S}{N} \right) \frac{1}{KN} \sum_{i,k} \mathbb{E} \left[ \|\mathbf{\alpha}_{i,k-1}^{r-1} - \mathbf{x}^r\|^2 \;|\; \alpha_{i,k-1}^r = \alpha_{i,k-1}^{r-1} \right] \\
    &\quad + \frac{S}{N} \frac{1}{KN} \sum_{i,k} \mathbb{E} \left[ \|\mathbf{y}_{i,k-1}^r - \mathbf{x}^r\|^2 \;|\; \alpha_{i,k-1}^{r-1} = y_{i,k-1}^r \right],
\end{align*}
but this conditional expectation is not specified in the proof of SCAFFOLD. For
the remainder of the proof of Lemma 16, these terms are treated as total
expectation, leading to an inconsistency. Lemma 16 concludes by applying Lemma
15 in order to bound the term $\mathbb{E} \left[ \left\| \mathbb{E}_{r-1} \left[
\Delta \mathbf{x}^r \right] \right\|^2 \right]$. However, when we make the
correction to include the conditional expectation, we do not have a bound of
$\Xi_r$ in terms of $\mathbb{E} \left[ \left\| \mathbb{E}_{r-1} \left[ \Delta
\mathbf{x}^r \right] \right\|^2 \right]$, we instead have a bound in terms of
$\mathbb{E} \left[ \left\| \mathbb{E}_{r-1} \left[ \Delta \mathbf{x}^r \;|\;
\alpha_{i,k-1}^{r-1} = y_{i,k-1}^r \right] \right\|^2 \right]$. But this term
with a conditional expectation inside the norm can't be bounded with Lemma 15.

A pessimistic solution is to use Jensen's inequality to bound
\begin{align*}
    \mathbb{E} \left[ \left\| \mathbb{E}_{r-1} \left[ \Delta \mathbf{x}^r \;|\; \alpha_{i,k-1}^{r-1} = y_{i,k-1}^r \right] \right\|^2 \right] &\leq \mathbb{E} \left[ \mathbb{E}_{r-1} \left[ \left\| \Delta \mathbf{x}^r \right\|^2 \;|\; \alpha_{i,k-1}^{r-1} = y_{i,k-1}^r \right] \right] \\
    &= \mathbb{E} \left[ \left\| \Delta \mathbf{x}^r \right\|^2 \right],
\end{align*}
where the second line follows from the tower property. This is the step that we
perform in our analysis of Amplified SCAFFOLD, and this results in the
$\frac{N}{S}$ dependence.

Fixing their analysis to recover the same $\left( \frac{N}{S} \right)^{2/3}$
dependence may be possible, but we have instead focused on achieving the best
known complexity in terms of $\epsilon, \kappa, \sigma$, and other problem
parameters.

%%%%%%%%%%%%%%%%%%%%%%%%%%%%%%%%%%%%%%%%%%%%%%%%%%%%%%%%%%%%

\newpage
\section*{NeurIPS Paper Checklist}

\begin{enumerate}

\item {\bf Claims}
    \item[] Question: Do the main claims made in the abstract and introduction accurately reflect the paper's contributions and scope?
    \item[] Answer: \answerYes{} % Replace by \answerYes{}, \answerNo{}, or \answerNA{}.
    \item[] Justification: Every claim made in the abstract is specifically tied to a theoretical convergence property of our proposed algorithm, which are stated in Section \ref{sec:main_results} and proven in \ref{app:main_proof} and \ref{app:pattern_proofs}.

\item {\bf Limitations}
    \item[] Question: Does the paper discuss the limitations of the work performed by the authors?
    \item[] Answer: \answerYes{} % Replace by \answerYes{}, \answerNo{}, or \answerNA{}.
    \item[] Justification: We discussion limitations in the "limitations" paragraph of Section \ref{sec:conclusion}.

\item {\bf Theory Assumptions and Proofs}
    \item[] Question: For each theoretical result, does the paper provide the full set of assumptions and a complete (and correct) proof?
    \item[] Answer: \answerYes{} % Replace by \answerYes{}, \answerNo{}, or \answerNA{}.
    \item[] Justification: All of the theoretical results are stated with the full set of assumptions in Sections \ref{sec:setup}, \ref{sec:participation_framework}, and \ref{sec:main_results}. The proofs are contained in Section \ref{app:main_proof} and \ref{app:pattern_proofs}.

    \item {\bf Experimental Result Reproducibility}
    \item[] Question: Does the paper fully disclose all the information needed to reproduce the main experimental results of the paper to the extent that it affects the main claims and/or conclusions of the paper (regardless of whether the code and data are provided or not)?
    \item[] Answer: \answerYes{} % Replace by \answerYes{}, \answerNo{}, or \answerNA{}.
    \item[] Justification: Experimental details are fully specified in Section \ref{sec:exp_setup} and Appendix \ref{app:experiment_details}.

\item {\bf Open access to data and code}
    \item[] Question: Does the paper provide open access to the data and code, with sufficient instructions to faithfully reproduce the main experimental results, as described in supplemental material?
    \item[] Answer: \answerYes{} % Replace by \answerYes{}, \answerNo{}, or \answerNA{}.
    \item[] Justification: The code for our experiments is included in the supplemental material with instructions for reproduction.

\item {\bf Experimental Setting/Details}
    \item[] Question: Does the paper specify all the training and test details (e.g., data splits, hyperparameters, how they were chosen, type of optimizer, etc.) necessary to understand the results?
    \item[] Answer: \answerYes{} % Replace by \answerYes{}, \answerNo{}, or \answerNA{}.
    \item[] Justification: The main training details are specified in Section \ref{sec:exp_setup}. Details such as hyperparameters, tuning protocols, architecture choices, etc. are specified in Section \ref{app:experiment_details}.

\item {\bf Experiment Statistical Significance}
    \item[] Question: Does the paper report error bars suitably and correctly defined or other appropriate information about the statistical significance of the experiments?
    \item[] Answer: \answerYes{} % Replace by \answerYes{}, \answerNo{}, or \answerNA{}.
    \item[] Justification: All of our experiments show results averaged over 3-5 random seeds, and error bars are shown for learning curves. Error bar calculation is specified in Section \ref{sec:exp_setup}.

\item {\bf Experiments Compute Resources}
    \item[] Question: For each experiment, does the paper provide sufficient information on the computer resources (type of compute workers, memory, time of execution) needed to reproduce the experiments?
    \item[] Answer: \answerYes{} % Replace by \answerYes{}, \answerNo{}, or \answerNA{}.
    \item[] Justification: Hardware resources are specified in Section \ref{sec:exp_setup}.

\item {\bf Code Of Ethics}
    \item[] Question: Does the research conducted in the paper conform, in every respect, with the NeurIPS Code of Ethics \url{https://neurips.cc/public/EthicsGuidelines}?
    \item[] Answer: \answerYes{} % Replace by \answerYes{}, \answerNo{}, or \answerNA{}.
    \item[] Justification: We have read and conformed to the NeurIPS Code of Ethics.

\item {\bf Broader Impacts}
    \item[] Question: Does the paper discuss both potential positive societal impacts and negative societal impacts of the work performed?
    \item[] Answer: \answerNA{} % Replace by \answerYes{}, \answerNo{}, or \answerNA{}.
    \item[] Justification: Our paper is a theoretical paper on mathematical optimization problems arising in distributed learning. We do not see any direct paths to negative applications beyond those existing for any application of distributed learning or machine learning in general.

\item {\bf Safeguards}
    \item[] Question: Does the paper describe safeguards that have been put in place for responsible release of data or models that have a high risk for misuse (e.g., pretrained language models, image generators, or scraped datasets)?
    \item[] Answer: \answerNA{} % Replace by \answerYes{}, \answerNo{}, or \answerNA{}.
    \item[] Justification: Our paper does not involve the release of any data or models.

\item {\bf Licenses for existing assets}
    \item[] Question: Are the creators or original owners of assets (e.g., code, data, models), used in the paper, properly credited and are the license and terms of use explicitly mentioned and properly respected?
    \item[] Answer: \answerYes{} % Replace by \answerYes{}, \answerNo{}, or \answerNA{}.
    \item[] Justification: Our paper uses existing datasets (Fashion-MNIST, CIFAR-10), and we cite original sources for both datasets in Section \ref{sec:exp_setup}.

\item {\bf New Assets}
    \item[] Question: Are new assets introduced in the paper well documented and is the documentation provided alongside the assets?
    \item[] Answer: \answerNA{} % Replace by \answerYes{}, \answerNo{}, or \answerNA{}.
    \item[] Justification: This paper does not release new assets.

\item {\bf Crowdsourcing and Research with Human Subjects}
    \item[] Question: For crowdsourcing experiments and research with human subjects, does the paper include the full text of instructions given to participants and screenshots, if applicable, as well as details about compensation (if any)?
    \item[] Answer: \answerNA{} % Replace by \answerYes{}, \answerNo{}, or \answerNA{}.
    \item[] Justification: This paper does not involve crowdsourcing or research with human subjects.

\item {\bf Institutional Review Board (IRB) Approvals or Equivalent for Research with Human Subjects}
    \item[] Question: Does the paper describe potential risks incurred by study participants, whether such risks were disclosed to the subjects, and whether Institutional Review Board (IRB) approvals (or an equivalent approval/review based on the requirements of your country or institution) were obtained?
    \item[] Answer: \answerNA{} % Replace by \answerYes{}, \answerNo{}, or \answerNA{}.
    \item[] Justification: This paper does not involve crowdsourcing or research with human subjects.

\end{enumerate}

\end{document}